\def\year{2022}\relax
\documentclass[letterpaper]{article} 
\usepackage{aaai22}  
\usepackage{times}  
\usepackage{helvet}  
\usepackage{courier}  
\usepackage[hyphens]{url}  
\usepackage{graphicx} 
\usepackage{natbib}  
\usepackage{caption} 
\DeclareCaptionStyle{ruled}{labelfont=normalfont,labelsep=colon,strut=off} 
\usepackage{algorithm}
\usepackage{algorithmic}

\usepackage{newfloat}
\usepackage{listings}
\floatstyle{ruled}
\newfloat{listing}{tb}{lst}{}
\floatname{listing}{Listing}
\usepackage{xcolor}
\usepackage{amsthm}
\usepackage{amsmath}
\usepackage{amssymb}
\usepackage{thmtools,thm-restate}

\makeatletter
\newcommand\footnoteref[1]{\protected@xdef\@thefnmark{\ref{#1}}\@footnotemark}
\makeatother

\usepackage{multirow}
\usepackage{makecell}
\usepackage{subcaption}
\usepackage{booktabs}

\title{Learning Action Translator for Meta\\ Reinforcement Learning on Sparse-Reward Tasks}
\author{
    Yijie Guo\textsuperscript{\rm 1}, Qiucheng Wu\textsuperscript{\rm 1},
    Honglak Lee\textsuperscript{\rm 1,2}
}
\affiliations{
    \textsuperscript{\rm 1}University of Michigan,
    \textsuperscript{\rm 2}LG AI Research\\

    guoyijie@umich.edu, wuqiuche@umich.edu,
    honglak@eecs.umich.edu\\
}

\usepackage{bibentry}

\begin{document}

\maketitle

\begin{abstract}
Meta reinforcement learning (meta-RL) aims to learn a policy solving a set of training tasks simultaneously and quickly adapting to new tasks. It requires massive amounts of data drawn from training tasks to infer the common structure shared among tasks. Without heavy reward engineering, the sparse rewards in long-horizon tasks exacerbate the problem of sample efficiency in meta-RL. Another challenge in meta-RL is the discrepancy of difficulty level among tasks, which might cause one easy task dominating learning of the shared policy and thus preclude policy adaptation to new tasks.
This work introduces a novel objective function to learn an action translator among training tasks.
We theoretically verify that the value of the transferred policy with the action translator can be close to the value of the source policy and our objective function (approximately) upper bounds the value difference.
We propose to combine the action translator with context-based meta-RL algorithms for better data collection and more efficient exploration during meta-training.
Our approach empirically improves the sample efficiency and performance of meta-RL algorithms on sparse-reward tasks.
\end{abstract}

\section{Introduction}
Deep reinforcement learning (DRL) methods achieved remarkable success in solving complex tasks \citep{mnih2015dqn,silver2016mastering,schulman2017proximal}.
While conventional DRL methods learn an individual policy for each task, meta reinforcement learning (meta-RL) algorithms \citep{finn2017model,duan2016rl,mishra2017simple} learn the shared structure across a distribution of tasks so that the agent can quickly adapt to unseen related tasks in the test phase.
Unlike most of the existing meta-RL approaches working on tasks with dense rewards, we instead focus on the sparse-reward training tasks, which are more common in real-world scenarios without access to carefully designed reward functions in the environments.
Recent works in meta-RL propose off-policy algorithms \citep{rakelly2019efficient,fakoor2019meta} and model-based algorithms \citep{nagabandi2018deep,nagabandi2018learning} to improve the sample efficiency in meta-training procedures. 
However, it remains challenging to efficiently solve multiple tasks that require reasoning over long horizons with sparse rewards. 
In these tasks, the scarcity of positive rewards exacerbates the issue of sample efficiency, which plagues meta-RL algorithms and 
makes exploration difficult due to a lack of guidance signals.

Intuitively, we hope that solving one task facilitates learning of other related tasks since the training tasks share a common structure.
However, it is often not the case in practice \citep{rusu2015policy,parisotto2015actor}.
Previous works \citep{teh2017distral,yu2020gradient} point out that detrimental gradient interference might cause an imbalance in policy learning on multiple tasks.
Policy distillation \citep{teh2017distral} and gradient projection \citep{yu2020gradient} are developed in meta-RL algorithms to alleviate this issue.
However, this issue might become more severe in the sparse-reward setting because it is hard to explore each task to obtain meaningful gradient signals for policy updates.
Good performance in one task does not automatically help exploration on the other tasks since the agent lacks positive rewards on the other tasks to learn from.  

\begin{figure}
	\small
    \centering
	\includegraphics[width=0.9\linewidth]{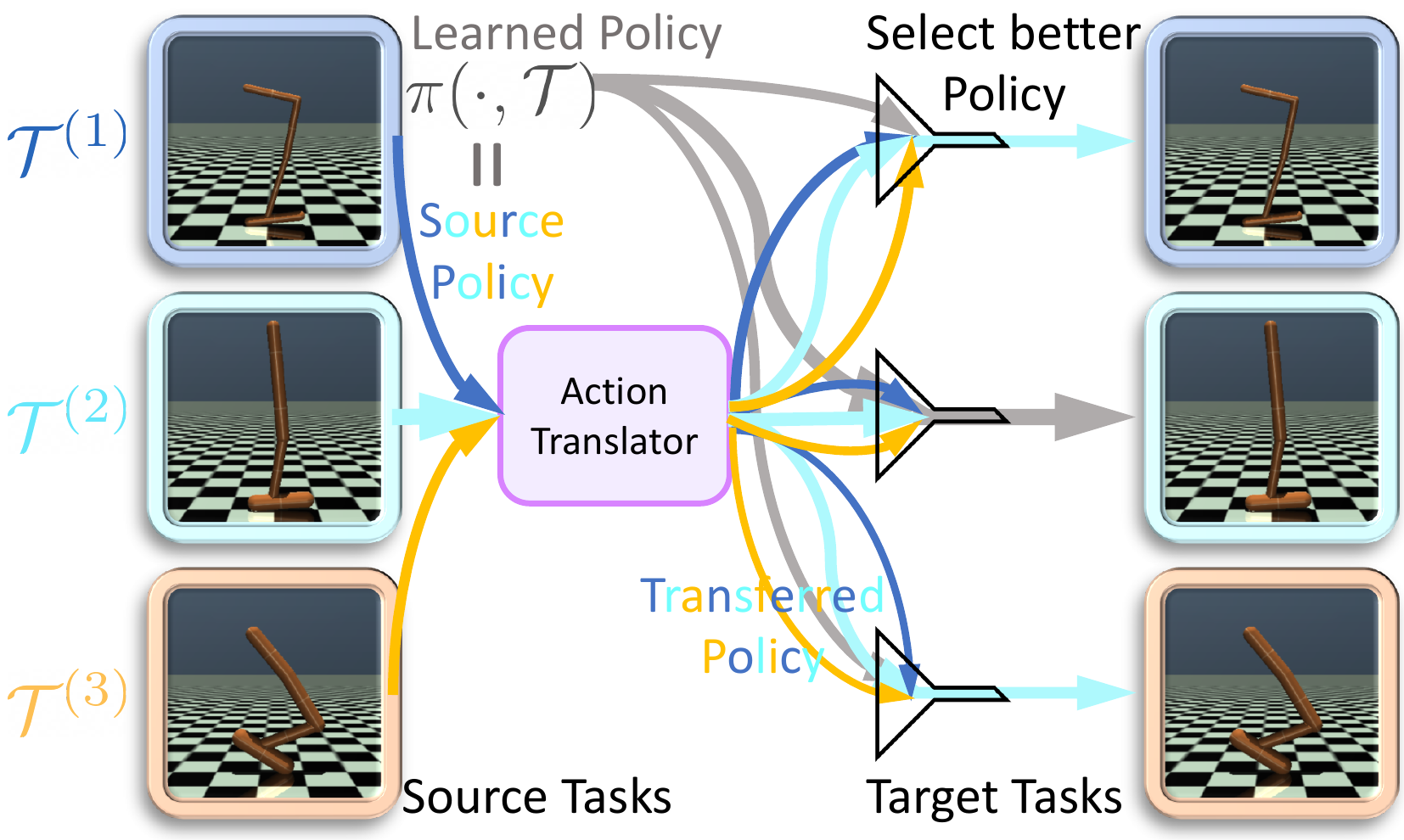} 
	\caption{Illustration of our policy transfer. Size of arrows represents avg. episode reward of learned or transferred policy on target tasks. Different colors indicate different tasks.
	}
	\label{fig:intro} 
\end{figure}

In this work, we aim to fully exploit the highly-rewarding transitions occasionally discovered by the agent in the exploration. 
The good experiences in one task should not only improve the policy on this task but also benefit the policy on other tasks to drive deeper exploration.
Specifically, once the agent learns from the successful trajectories in one training task, we transfer the good policy in this task to other tasks 
to get more positive rewards on other training tasks.
In Fig.~\ref{fig:intro}, if the learned policy $\pi$ performs better on task $\smash{\mathcal{T}^{(2)}}$ than other tasks, then our goal is to transfer the good policy $\smash{\pi(\cdot, \mathcal{T}^{(2)})}$ to other tasks $\smash{\mathcal{T}^{(1)}}$ and $\smash{\mathcal{T}^{(3)}}$. 
%
To enable such transfer, we propose to learn an action translator among multiple training tasks.
The objective function forces the translated action to behave on the target task similarly to the source action on the source task.
We consider the policy transfer for any pair of source and target tasks in the training task distribution (see the colored arrows in Fig.~\ref{fig:intro}).
The agent executes actions following the transferred policy if the transferred policy attains higher rewards than the learned policy on the target task in recent episodes.
This approach enables the agent to leverage relevant data from  multiple training tasks, encourages the learned policy to perform similarly well on multiple training tasks, and thus leads to better performance when applying the well-trained policy to test tasks.

We summarize the contributions: (1) We introduce a novel objective function to transfer any policy from a source Markov Decision Process (MDP) to a target MDP. We prove a theoretical guarantee that the transferred policy can achieve the expected return on the target MDP close to the source policy on the source MDP. The difference in expected returns is (approximately) upper bounded by our loss function with a constant multiplicative factor. (2) We develop an off-policy RL algorithm called \textbf{M}eta-RL with \textbf{C}ontext-conditioned \textbf{A}ction \textbf{T}ranslator (MCAT), applying a policy transfer mechanism in meta-RL to help exploration across multiple sparse-rewards tasks. (3) We empirically demonstrate the effectiveness of MCAT on a variety of simulated control tasks with the MuJoCo physics engine \citep{todorov2012mujoco}, showing that policy transfer improves the performance of context-based meta-RL algorithms.

\section{Related Work}
\textbf{Context-based Meta-RL}
Meta reinforcement learning has been extensively studied in the literature \citep{finn2017model, stadie2018some, sung2017learning, xu2018meta} with many works developing the context-based approaches \citep{rakelly2019efficient,ren2020ocean,liu2020explore}.
\citet{duan2016rl,wang2016learning,fakoor2019meta} employ recurrent neural networks to encode context transitions and formulate the policy conditioning on the context variables.
The objective function of maximizing expected return trains the context encoder and policy jointly. \citet{rakelly2019efficient} leverage a permutation-invariant encoder to aggregate experiences as probabilistic context variables and optimizes it with variational inference.
The posterior sampling is beneficial for exploration on sparse-reward tasks in the adaptation phase, but there is access to dense rewards during training phase.
\citet{li2020generalized} considers a task-family of reward functions.
\citet{lee2020context,seo2020trajectory} trains the context encoder with forward dynamics prediction.
These model-based meta-RL algorithms assume the reward function is accessible for planning. In the sparse-reward setting without ground-truth reward functions, they may struggle to discover non-zero rewards and accurately estimating the reward for model-based planning may be problematic as well.

\textbf{Policy Transfer in RL}
Policy transfer studies the knowledge transfer in target tasks given a set of source tasks and their expert policies.
Policy distillation \citep{rusu2015policy,yin2017knowledge,parisotto2015actor} minimize the divergence of action distributions between the source policy and the learned policy on the target task.
Along this line of works, \citet{teh2017distral} create a centroid policy in multi-task reinforcement learning and distills the knowledge from the task-specific policies to this centroid policy.
Alternatively, inter-task mapping between the source and target tasks \citep{zhu2020transfer} can assist the policy transfer.
Most of these works \citep{gupta2017learning,konidaris2006autonomous,ammar2011reinforcement} assume existence of correspondence over the state space and learn the state mapping between tasks. Recent work \citep{zhang2020learning} learns the state correspondence and action correspondence with dynamic cycle-consistency loss.
Our method differs from this approach, in that we enable action translation among multiple tasks with a simpler objective function.
Importantly, our approach is novel to utilize the policy transfer for any pair of source and target tasks in meta-RL.

\textbf{Bisimulation for States in MDPs}
Recent works on state representation learning \citep{ferns2004metrics, zhang2020invariant, agarwal2021contrastive} investigate the bismilarity metrics for states on multiple MDPs and consider how to learn a representation for states leading to almost identical behaviors under the same action in diverse MDPs. In multi-task reinforcement learning and meta reinforcement learning problems, \citet{zhang2020invariant, zhang2020robust} derives transfer and generalization bounds based on the task and state similarity.
We also bound the value of policy transfer across tasks but our approach is to establish action equivalence instead of state equivalence. 

\section{Method}
In this section, we first describe our approach to learn a context encoder capturing the task features and learn a forward dynamics model predicting next state distribution given the task context (Sec.~\ref{sec:context}).
Then we introduce an objective function to train an action translator so that the translated action on the target task behaves equivalently to the source action on the source task.
The action translator can be conditioned on the task contexts and thus it can transfer a good policy from any arbitrary source task to any other target task in the training set (Sec.~\ref{sec:multi_translator}).
Finally, we propose to combine the action translator with a context-based meta-RL algorithm to transfer the good policy from any one task to the others.
During meta-training, this policy transfer approach helps exploit the good experiences encountered on any one task and benefits the data collection and further policy optimization on other sparse-reward tasks (Sec.~\ref{sec:combine}).
Fig.~\ref{fig:overview} provides an overview of our approach MCAT. 

\begin{figure*}[t]
\centering
\includegraphics[width=\textwidth]{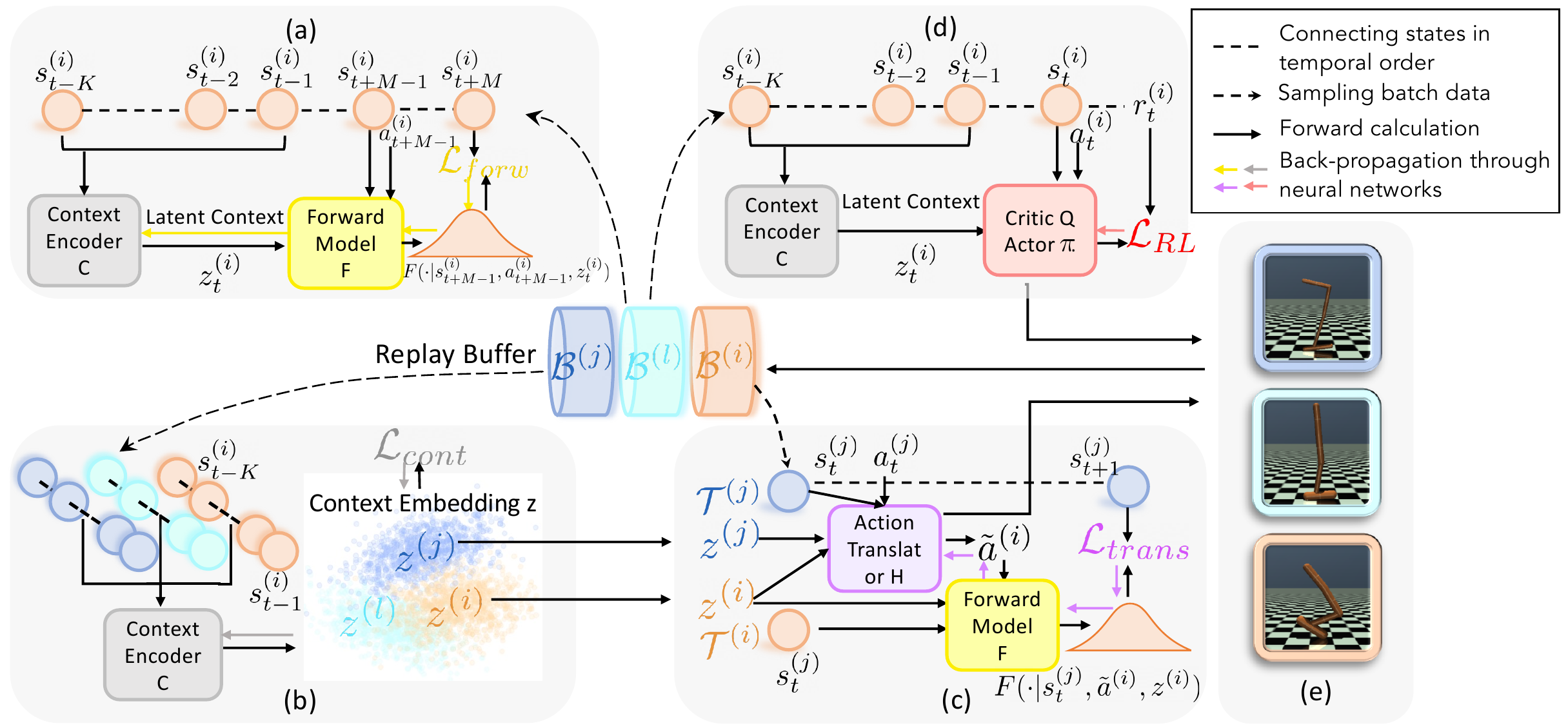}%
 \caption{Overview of MCAT.
 (a) We use forward dynamics prediction loss to train the context encoder $\smash{C}$ and forward model $\smash{F}$.
 (b) We regularize the context encoder $\smash{C}$ with the contrastive loss, so context vectors of transition segments from the same task cluster together.
 (c) With fixed $\smash{C}$ and $\smash{F}$, we learn the action translator $\smash{H}$ for any pair of source task $\smash{\mathcal{T}^{(j)}}$ and target task $\smash{\mathcal{T}^{(i)}}$.
 The action translator aims to generate action $\smash{\tilde{a}^{(i)}}$ on the target task leading to the same next state $\smash{s^{(j)}_{t+1}}$ as the source action $\smash{a^{(j)}_t}$ on the source task.
 (d) With fixed $\smash{C}$, we learn the critic $\smash{Q}$ and actor $\smash{\pi}$ conditioning on the context feature.
 (e) If the agent is interacting with the environment on task $\smash{\mathcal{T}^{(i)}}$, we compare learned policy $\smash{\pi(s,z^{(i)})}$ and transferred policy $\smash{H(s, \pi(s,z^{(j)}), z^{(j)}, z^{(i)})}$, which transfers a good policy $\smash{\pi(s,z^{(j)})}$ on source task $\smash{\mathcal{T}^{(j)}}$ to target task $\smash{\mathcal{T}^{(i)}}$.
 We select actions according to the policy with higher average episode rewards in the recent episodes. Transition data are pushed into the buffer.
 We remark that the components $\smash{C,F,H,Q,\pi}$ are trained alternatively not jointly and this fact facilitates the learning process.}
\label{fig:overview}
\end{figure*}

\subsection{Problem Formulation}
\label{sec:problem}
Following meta-RL formulation in previous work \citep{duan2016rl,mishra2017simple,rakelly2019efficient}, we assume a distribution of tasks $\smash{p(\mathcal{T})}$ and each task is a Markov decision process (MDP) defined as a tuple $\smash{(\mathcal{S}, \mathcal{A}, p, r, \gamma, \rho_0)}$ with state space $\smash{\mathcal{S}}$ , action space $\smash{\mathcal{A}}$,
transition function $\smash{p(s'|s, a)}$,
reward function $\smash{r(s,a,s')}$, discounting factor $\smash{\gamma}$, and initial state distribution $\smash{\rho_0}$.
We can alternatively define the reward function as $\smash{r(s, a)=\sum_{s'\in \mathcal{S}} p(s'|s,a)r(s, a, s')}$.
In context-based meta-RL algorithms, we learn a policy $\smash{\pi(\cdot|s^{(i)}_t, z^{(i)}_t)}$ shared for any task $\smash{\mathcal{T}^{(i)}\sim p(\mathcal{T})}$, where $t$ denotes the timestep in an episode, $i$ denotes the index of a task, the context variable $\smash{z^{(i)}_t \in \mathcal{Z}}$ captures contextual information from history transitions on the task MDP and $\smash{\mathcal{Z}}$ is the space of context vectors.
The shared policy is optimized to maximize its value $\smash{V^{\pi}(\mathcal{T}^{(i)})=\mathbb{E}_{\rho_0^{(i)},\pi,p^{(i)}}[\sum_{t=0}^{\infty}\gamma^t r^{(i)}_t]}$ on each training task $\smash{\mathcal{T}^{(i)}}$.
Following prior works in meta-RL \citep{yu2017preparing,nagabandi2018learning,nagabandi2018deep,zhou2019environment,lee2020context}, we study tasks with the same state space, action space, reward function but varying dynamics functions.
Importantly, we focus on more challenging setting of sparse rewards.
Our goal is to learn a shared policy robust to the dynamic changes and generalizable to unseen tasks. 

\subsection{Learning Context \& Forward Model}
\label{sec:context}

In order to capture the knowledge about any task $\smash{\mathcal{T}^{(i)}}$, we leverage a context encoder $\smash{C:\mathcal{S}^{K}\times \mathcal{A}^{K}\rightarrow \mathcal{Z}}$, where  $\smash{K}$ is the number of past steps used to infer the context.
Related ideas have been explored by \citep{rakelly2019efficient, zhou2019environment,lee2020context}.
In Fig.~\ref{fig:overview}a, given $\smash{K}$ past transitions $\smash{(s^{(i)}_{t-K}, a^{(i)}_{t-K}, \cdots, s^{(i)}_{t-1}, a^{(i)}_{t-1})}$,
context encoder $\smash{C}$ produces the latent context $\smash{z^{(i)}_t=C(s^{(i)}_{t-K}, a^{(i)}_{t-K}, \cdots, s^{(i)}_{t-2}, a^{(i)}_{t-2}, s^{(i)}_{t-1}, a^{(i)}_{t-1})}$.
We train the context encoder $\smash{C}$ and forward dynamics $\smash{F}$ with an objective function to predict the forward dynamics in future transitions $\smash{s^{(i)}_{t+m}}$ ($\smash{1\leq m\leq M}$) within $\smash{M}$ future steps.
The state prediction in multiple future steps drives latent context embeddings $\smash{z^{(i)}_t}$ to be temporally consistent.
The learned context encoder tends to capture dynamics-specific, contextual information (e.g. environment physics parameters).
Formally, we minimize the negative log-likelihood of observing the future states under dynamics prediction.

\setlength{\belowdisplayskip}{0pt} \setlength{\abovedisplayskip}{0pt}
\vspace*{-0.2in}
\begin{equation}
\label{eq:forw}
    \mathcal{L}_{forw}=-\sum_{m=1}^M \log F(s^{(i)}_{t+m}|s^{(i)}_{t+m-1}, a^{(i)}_{t+m-1}, z^{(i)}_t).
\end{equation}

Additionally, given trajectory segments 
from the same task, we require their context embeddings to be similar, whereas the contexts of history transitions from different tasks should be distinct (Fig.~\ref{fig:overview}b).
We propose a contrastive loss \citep{hadsell2006dimensionality} to constrain embeddings within a small distance for positive pairs (i.e. samples from the same task) and push embeddings apart with a distance greater than a margin value $m$ for negative pairs (i.e. samples from different tasks). $\smash{z^{(i)}_{t_1}}$, $\smash{z^{(j)}_{t_2}}$ denote context embeddings of two trajectory samples from $\smash{\mathcal{T}^{(i)}}$, $\smash{\mathcal{T}^{(j)}}$. The contrastive loss function is defined as:

\begin{equation}
\label{eq:cont}
    \mathcal{L}_{cont}=1_{i=j} \|z^{(i)}_{t_1}-z^{(j)}_{t_2}\|^2
    + 1_{i\neq j} \max(0, m-\|z^{(i)}_{t_1}-z^{(j)}_{t_2}\|)
\end{equation}

where $1$ is indicator function. During meta-training, recent transitions on each task $\smash{\mathcal{T}^{(i)}}$ are stored in a buffer $\smash{\mathcal{B}^{(i)}}$ for off-policy learning.
We randomly sample a fairly large batch of trajectory segments from $\smash{\mathcal{B}^{(i)}}$, and average their context embeddings to output task feature $\smash{z^{(i)}}$.
$\smash{z^{(i)}}$ is representative for embeddings on task $\smash{\mathcal{T}^{(i)}}$ and distinctive from features $\smash{z^{(l)}}$ and $\smash{z^{(j)}}$ for other tasks.
We note the learned embedding maintains the similarity across tasks.
$\smash{z^{(i)}}$ is closer to $\smash{z^{(l)}}$ than to $\smash{z^{(j)}}$ if task $\smash{\mathcal{T}^{(i)}}$ is more akin to $\smash{\mathcal{T}^{(l)}}$.
We utilize task features for action translation across multiple tasks.
Appendix~D.5 visualizes context embeddings to study $\mathcal{L}_{cont}$.

\subsection{Learning Action Translator}
\label{sec:multi_translator}

Suppose that transition data $\smash{s^{(j)}_t, a^{(j)}_t, s^{(j)}_{t+1}}$ behave well on task $\smash{\mathcal{T}^{(j)}}$.
We aim to learn an action translator $\smash{H:\mathcal{S}\times \mathcal{A}\times\mathcal{Z}\times\mathcal{Z}\rightarrow \mathcal{A}}$. $\smash{\tilde{a}^{(i)} = H(s^{(j)}_t, a^{(j)}_t, z^{(j)}, z^{(i)})}$ translates the proper action $\smash{a^{(j)}_t}$ from source task $\smash{\mathcal{T}^{(j)}}$ to target task $\smash{\mathcal{T}^{(i)}}$.
In Fig.~\ref{fig:overview}c, if we start from the same state $\smash{s^{(j)}_t}$ on both source and target tasks, the translated action $\smash{\tilde{a}^{(i)}}$ on target task should behave equivalently to the source action $\smash{a^{(j)}_t}$ on the source task.
Thus, the next state $\smash{s_{t+1}^{(i)}\sim p^{(i)}(s^{(j)}_t, \tilde{a}^{(i)})}$ produced from the transferred action $\smash{\tilde{a}^{(i)}}$ on the target task should be close to the real next state $\smash{s^{(j)}_{t+1}}$ gathered on the source task.
The objective function of training the action translator $\smash{H}$ is to maximize the probability of getting next state $\smash{s_{t+1}^{(j)}}$ under the next state distribution  $\smash{s_{t+1}^{(i)}\sim p^{(i)}(s^{(j)}_t, \tilde{a}^{(i)})}$ on the target task.
Because the transition function $\smash{p^{(i)}(s^{(j)}_t, \tilde{a}^{(i)})}$ is unavailable and might be not differentiable, we use the forward dynamics model $\smash{F(\cdot|s_t^{(j)}, \tilde{a}^{(i)}, z^{(i)})}$ to approximate the transition function.
We formulate objective function for action translator $\smash{H}$ as:

\begin{equation}
\label{eq:trans}
    \mathcal{L}_{trans}=-\log F(s_{t+1}^{(j)}|s_t^{(j)}, \tilde{a}^{(i)}, z^{(i)})
\end{equation}

where $\tilde{a}^{(i)} = H(s^{(j)}_t, a^{(j)}_t, z^{(j)}, z^{(i)})$.
We assume to start from the same initial state, the action translator is to find the action on the target task so as to reach the same next state as the source action on the source task.
This intuition to learn the action translator is analogous to learn inverse dynamic model across two tasks.

With a well-trained action translator conditioning on task features $\smash{z^{(j)}}$ and $\smash{z^{(i)}}$, we transfer the good deterministic policy $\smash{\pi(s, z^{(j)})}$ from any source task $\smash{\mathcal{T}^{(j)}}$ to any target task $\smash{\mathcal{T}^{(i)}}$.
When encountering a state $\smash{s^{(i)}}$ on $\smash{\mathcal{T}^{(i)}}$, we query a good action $\smash{a^{(j)}=\pi(s^{(i)}, z^{(j)})}$ which will lead to a satisfactory next state with high return on the source task.
Then $\smash{H}$ translates this good action $\smash{a^{(j)}}$ on the source task to action $\smash{\tilde{a}^{(i)}=H(s^{(i)}, a^{(j)}, z^{(j)}, z^{(i)})}$ on the target task. Executing the translated action $\smash{\tilde{a}^{(i)}}$ moves the agent to a next state on the target task similarly to the good action on the source task. Therefore, transferred policy $\smash{H(s^{(i)}, \pi(s^{(i)}, z^{(j)}), z^{(i)}, z^{(j)})}$ can behave similarly to source policy $\smash{\pi(s, z^{(j)})}$.
Sec.~\ref{sec:fixed} demonstrates the performance of transferred policy in a variety of environments.
Our policy transfer mechanism is related to the action correspondence discussed in \citep{zhang2020learning}.
We extend their policy transfer approach across two domains to multiple domains(tasks) and theoretically validate learning of action translator in Sec.~\ref{sec:theory}. 

\subsection{Combining with Context-based Meta-RL}
\label{sec:combine}

MCAT follows standard off-policy meta-RL algorithms to learn a deterministic policy $\smash{\pi(s_t, z^{(i)}_t)}$ and a value function $\smash{Q(s_t, a_t, z^{(i)}_t)}$, conditioning on the latent task context variable $\smash{z^{(i)}_t}$.
In the meta-training process, using data sampled from $\smash{\mathcal{B}}$, we train the context model $\smash{C}$ and dynamics model $\smash{F}$ with $\mathcal{L}_{forw}$ and $\mathcal{L}_{cont}$ to accurately predict the next state (Fig.~\ref{fig:overview}a \ref{fig:overview}b).
With the fixed context encoder $\smash{C}$ and dynamics model $\smash{F}$, the action translator $\smash{H}$ is optimized to minimize $\mathcal{L}_{trans}$ (Fig.~\ref{fig:overview}c).
Then, with the fixed $\smash{C}$, we train the context-conditioned policy $\smash{\pi}$ and value function $\smash{Q}$ according to $\smash{\mathcal{L}_{RL}}$ (Fig.~\ref{fig:overview}d).
In experiments, we use the objective function $\smash{\mathcal{L}_{RL}}$ from TD3 algorithm \citep{fujimoto2018addressing}.
See pseudo-code of MCAT in Appendix~B.

On sparse-reward tasks where exploration is challenging, the agent might luckily find transitions with high rewards on one task $\smash{\mathcal{T}^{(j)}}$.
Thus, the policy learning on this task might be easier than other tasks.
If the learned policy $\smash{\pi}$ performs better on one task $\smash{\mathcal{T}^{(j)}}$ than another task $\smash{\mathcal{T}^{(i)}}$, we consider the policy transferred from $\smash{\mathcal{T}^{(j)}}$ to $\smash{\mathcal{T}^{(i)}}$.
At a state $s^{(i)}$, we employ the action translator to get a potentially good action $\smash{H(s^{(i)}, \pi(s^{(i)}, z^{(j)}), z^{(j)}, z^{(i)})}$ on target task $\smash{\mathcal{T}^{(i)}}$.
As illustrated in Fig.~\ref{fig:overview}e and Fig.~\ref{fig:intro}, in the recent episodes, if the transferred policy earns higher scores than the learned policy $\smash{\pi(s^{(i)}, z^{(i)})}$ on the target task $\smash{\mathcal{T}^{(i)}}$, we follow the translated actions on $\smash{\mathcal{T}^{(i)}}$ to gather transition data in the current episode.
These data with better returns are pushed into the replay buffer $\smash{\mathcal{B}^{(i)}}$ and produce more positive signals for policy learning in the sparse-reward setting.
These transition samples help improve $\pi$ on $\smash{\mathcal{T}^{(i)}}$ after policy update with off-policy RL algorithms.
As described in Sec.~\ref{sec:multi_translator},
our action translator $\smash{H}$ allows policy transfer across any pair of tasks.
Therefore, with the policy transfer mechanism, the learned policy on each task might benefit from good experiences and policies on any other tasks.
  
\section{Theoretical Analysis}
\label{sec:theory}

In this section, we theoretically support our objective function (Equation~\ref{eq:trans}) to learn the action translator. 
Given $s$ on two MDPs with the same state and action space, 
we define that action $a^{(i)}$ on $\smash{\mathcal{T}^{(i)}}$ is equivalent to action $a^{(j)}$ on $\smash{\mathcal{T}^{(j)}}$ if the actions yielding exactly the same next state distribution and reward, i.e. $\smash{p^{(i)}(\cdot|s, a^{(i)}) = p^{(j)}(\cdot|s, a^{(j)})}$ and  $\smash{r^{(i)}(s, a^{(i)}) = r^{(j)}(s, a^{(j)})}$ .
Ideally, the equivalent action always exists on the target MDP $\smash{\mathcal{T}^{(i)}}$ for any state-action pair on the source MDP $\smash{\mathcal{T}^{(j)}}$ and there exists an action translator function $\smash{H:\mathcal{S}\times \mathcal{A} \rightarrow \mathcal{A}}$ to identify the exact equivalent action.
Starting from state $s$, the translated action $\smash{\tilde{a}=H(s, a)}$ on the task $\smash{\mathcal{T}^{(i)}}$ generates reward and next state distribution the same as action $a$ on the task $\smash{\mathcal{T}^{(j)}}$ (i.e. $\smash{\tilde{a} B_s a}$).
Then any deterministic policy $\smash{\pi^{(j)}}$ on the source task $\smash{\mathcal{T}^{(j)}}$ can be perfectly transferred to the target task $\smash{\mathcal{T}^{(i)}}$ with $\smash{\pi^{(i)}(s) = H(s, \pi^{(j)}(s))}$.
The value of the policy $\smash{\pi^{(j)}}$ on the source task $\smash{\mathcal{T}^{(j)}}$ is equal to the value of transferred policy $\smash{\pi^{(i)}}$ on the target task $\smash{\mathcal{T}^{(i)}}$.

Without the assumption of existence of a perfect correspondence for each action, given any two deterministic policies $\smash{\pi^{(j)}}$ on $\smash{\mathcal{T}^{(j)}}$ and $\smash{\pi^{(i)}}$ on $\smash{\mathcal{T}^{(i)}}$, we prove that the difference in the policy value is upper bounded by a scalar $\frac{d}{1-\gamma}$ depending on L1-distance between reward functions $\smash{|r^{(i)}(s, \pi^{(i)}(s)) - r^{(j)}(s, \pi^{(j)}(s))|}$ and total-variation distance between next state distributions $\smash{D_{TV}(p^{(i)}(\cdot|s,\pi^{(i)}(s)), p^{(j)}(\cdot|s,\pi^{(j)}(s)))}$.
Detailed theorem (Theorem 1) and proof are in Appendix~A.

For a special case where reward function $\smash{r(s,a,s')}$ only depends on the current state $s$ and next state $s'$, the upper bound of policy value difference is only related to the distance in next state distributions.

\begin{restatable}{cor}{boundpolicytransfertv}
\label{cor:bound_policy}
Let $\smash{\mathcal{T}^{(i)}=\{\mathcal{S}, \mathcal{A}, p^{(i)}, r^{(i)}, \gamma, \rho_0\}}$ and  $\smash{\mathcal{T}^{(j)}=\{\mathcal{S}, \mathcal{A}, p^{(j)}, r^{(j)}, \gamma, \rho_0\}}$ be two MDPs sampled from the distribution of tasks $p(\mathcal{T})$. $\pi^{(i)}$, $\pi^{(j)}$ is the deterministic policy on $\smash{\mathcal{T}^{(i)}}$, $\smash{\mathcal{T}^{(j)}}$. Assume the reward function only depends on the state and next state $\smash{r^{(i)}(s, a^{(i)}, s')=r^{(j)}(s, a^{(j)}, s')= r(s, s')}$. Let $d = \sup_{s\in\mathcal{S}} 2M D_{TV}(p^{(j)}(\cdot|s,\pi^{(j)}(s)), p^{(i)}(\cdot|s,\pi^{(i)}(s)))$ and $\smash{M=\sup_{s\in \mathcal{S}, s'\in \mathcal{S}} |r(s, s')+\gamma V^{\pi^{(j)}}(s, \mathcal{T}^{(j)})|}$.
$\forall s \in \mathcal{S}$, we have
\begin{equation}
   \left|V^{\pi^{(i)}}(s, \mathcal{T}^{(i)}) - V^{\pi^{(j)}}(s, \mathcal{T}^{(j)})\right| \leq \frac{d}{1-\gamma} 
\end{equation}
\end{restatable}
According to Proposition~\ref{cor:bound_policy}, if we can optimize the action translator $\smash{H}$ to minimize $d$ for policy $\smash{\pi^{(j)}}$ and $\smash{\pi^{(i)}(s) = H(s, \pi^{(j)}(s))}$, the value of the transferred policy $\smash{\pi^{(i)}}$ on the target task can be close to the value of source policy $\smash{\pi^{(j)}}$.
In many real-world scenarios, especially sparse-reward tasks, the reward heavily depends on the state and next state instead of action. For example, robots running forward receive rewards according to their velocity (i.e. the location difference between the current and next state within one step); robot arms manipulating various objects earn positive rewards only when they are in the target positions.
Thus, our approach focuses on the cases with reward function approximately as $r(s,s')$ under the assumption of Proposition~\ref{cor:bound_policy}.
For any state $\smash{s\in \mathcal{S}}$, we minimize the total-variation distance between two next state distributions $\smash{D_{TV}(p^{(j)}(\cdot|s_t,\pi^{(j)}(s_t)), p^{(i)}(\cdot|s_t,\pi^{(i)}(s_t)))}$ on source and target MDPs.
Besides, we discuss the policy transfer for tasks with a general reward function in Appendix~C.3.


There is no closed-form solution of $\smash{D_{TV}}$ and $\smash{D_{TV}}$ is related with Kullback–Leibler (KL) divergence $\smash{D_{KL}}$ by the inequality $\smash{D_{TV}(p\|q)^2\leq D_{KL}(p\|q)}$ 
Thus, we instead consider minimizing $\smash{D_{KL}}$ between two next state distributions.
$\smash{D_{KL}(p^{(j)}||p^{(i)})}$ is $\smash{-\sum_{s'}p^{(j)}(s') \log p^{(i)}(s')} + \smash{\sum_{s'}p^{(j)}(s') \log p^{(j)}(s')}$.
The second term does not involve $H$ and thus can be viewed as a constant term when optimizing $H$.
We focus on minimizing the first term $\smash{-\sum_{s'}p^{(j)}(s') \log p^{(i)}(s')}$.
$F$ is a forward model approximating $\smash{p^{(i)}(s')}$.
We sample transitions $s, \smash{\pi^{(j)}(s)}, s'$ from the source task. $s'$ follows the distribution $\smash{p^{(j)}(s')}$.
Thus, minimizing the negative log-likelihood of observing the next state $\smash{L_{trans}=-\log F(s'|s,\pi^{(i)}(s))}$ is to approximately minimize $D_{KL}$.
Experiments in Sec.~\ref{sec:fixed}
suggest that this objective function works well for policy transfer across two MDPs. 
Sec.~\ref{sec:multi_translator} explains the motivation behind $\smash{\mathcal{L}_{trans}}$ (Equation~\ref{eq:trans}) to learn an action translator among multiple MDPs instead of only two MDPs.

\section{Experiment}
\label{sec:experiment}
We design and conduct experiments to answer the following questions: (1) Does the transferred policy perform well on the target task (Tab.~\ref{tab:fixed_dataset}, Fig.~\ref{fig:path})?
(2) Can we transfer the good policy for any pair of source and target tasks (Fig.~\ref{fig:heatmap_improvement})?
(3) Does policy transfer improve context-based Meta-RL algorithms (Fig.~\ref{fig:ours_baseline}, Tab.~\ref{tab:ours_baseline},  Tab.~\ref{tab:effect_pt})?
(4) Is the policy transfer more beneficial when the training tasks have sparser rewards (Tab.~\ref{tab:more_sparse})?
Experimental details can be found in Appendix~C.

\begin{figure*}[!t]
\centering
\begin{subfigure}{0.195\textwidth}
    \centering
    \includegraphics[width=\linewidth]{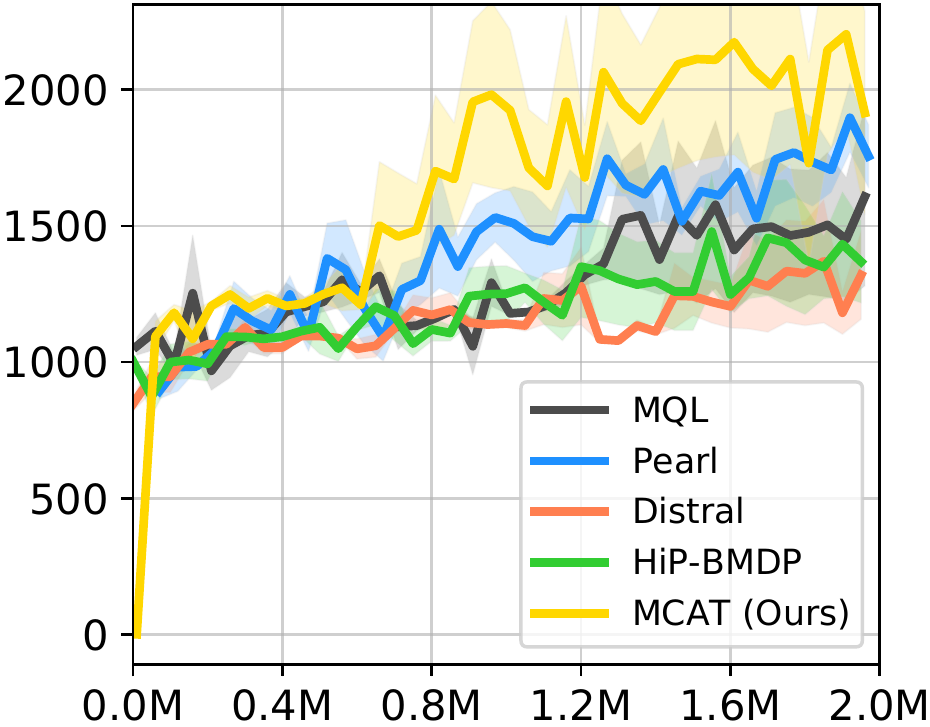}
    \caption{Hopper Size}
\end{subfigure}%
\hfill
\begin{subfigure}{0.195\textwidth}
    \includegraphics[width=\linewidth]{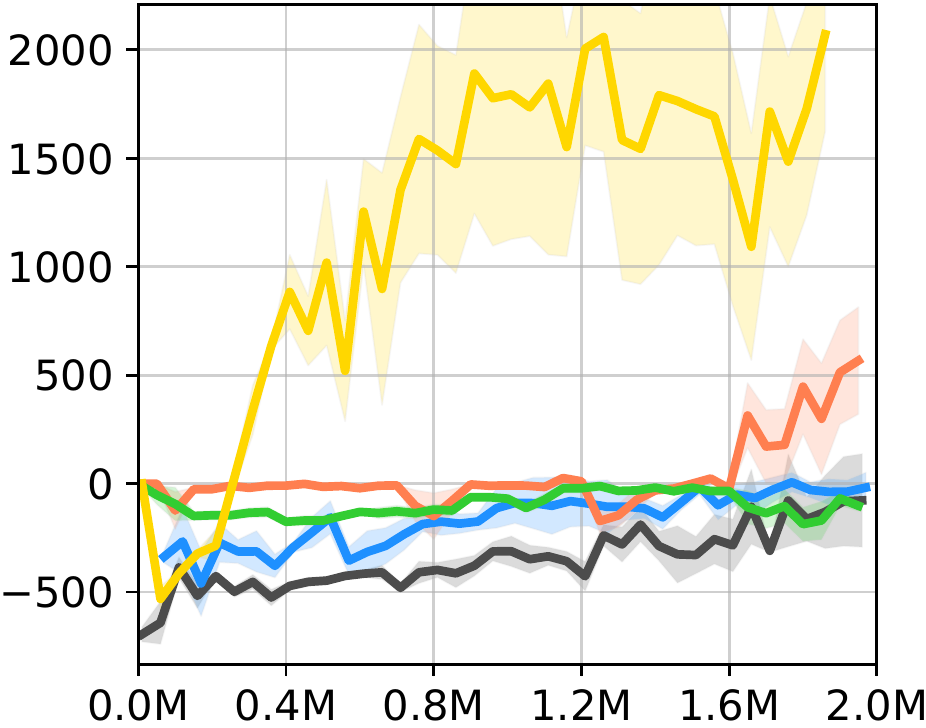}
    \caption{HalfCheetah Armature}
\end{subfigure}%
\hfill
\begin{subfigure}{0.195\textwidth}
    \centering
    \includegraphics[width=\linewidth]{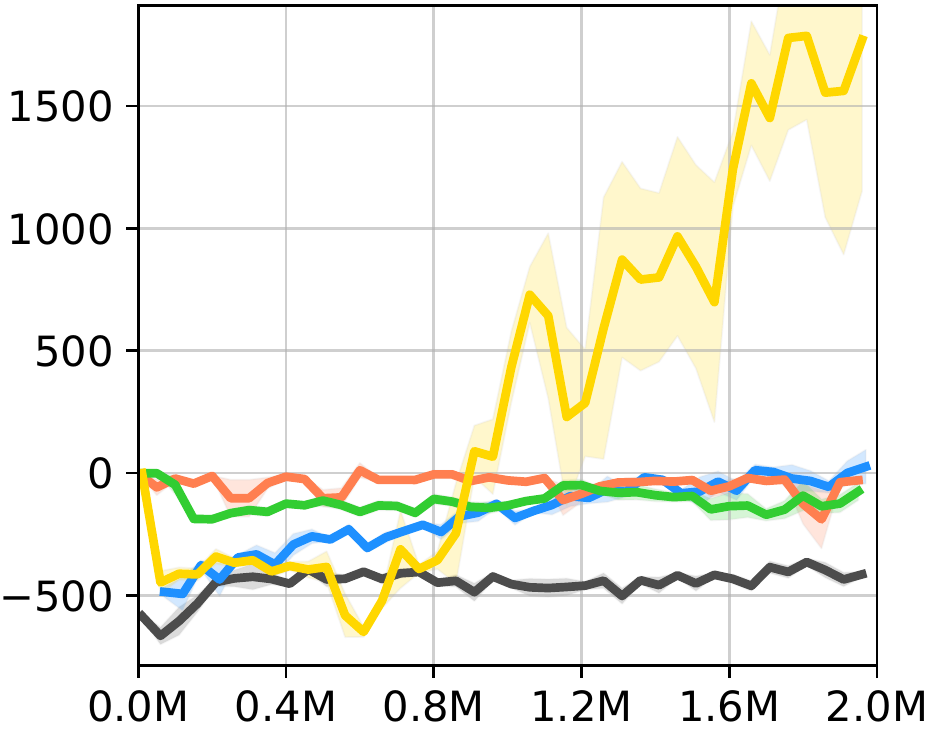}
    \caption{HalfCheetah Mass}
\end{subfigure}
\hfill
\begin{subfigure}{0.195\textwidth}
    \centering
    \includegraphics[width=\linewidth]{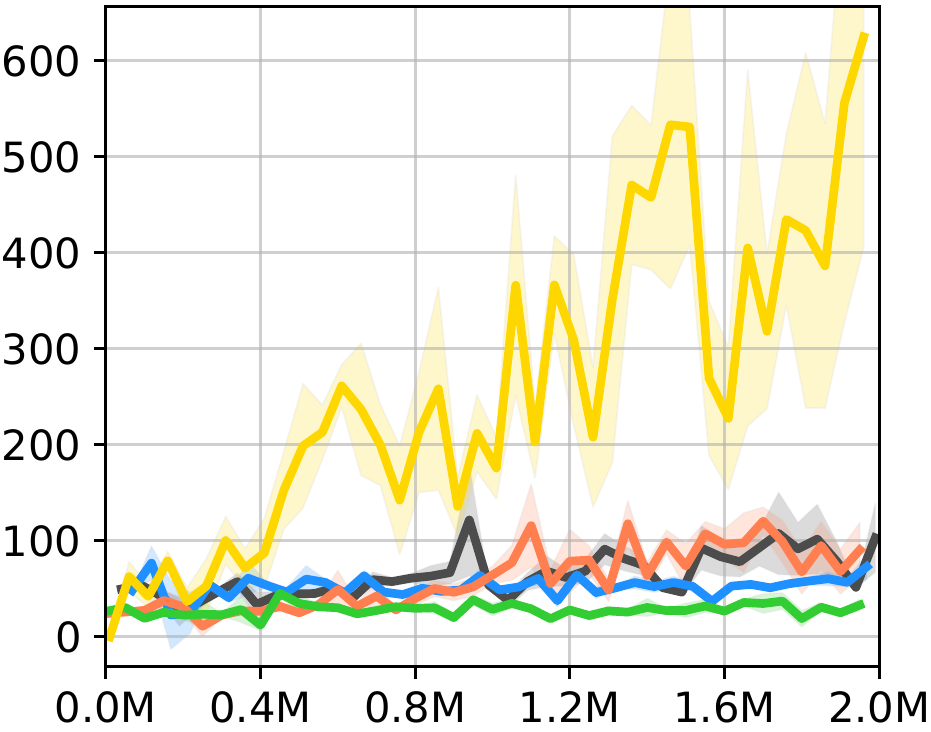}
    \caption{Ant Damping}
\end{subfigure}
\hfill
\begin{subfigure}{0.195\textwidth}
    \centering
    \includegraphics[width=\linewidth]{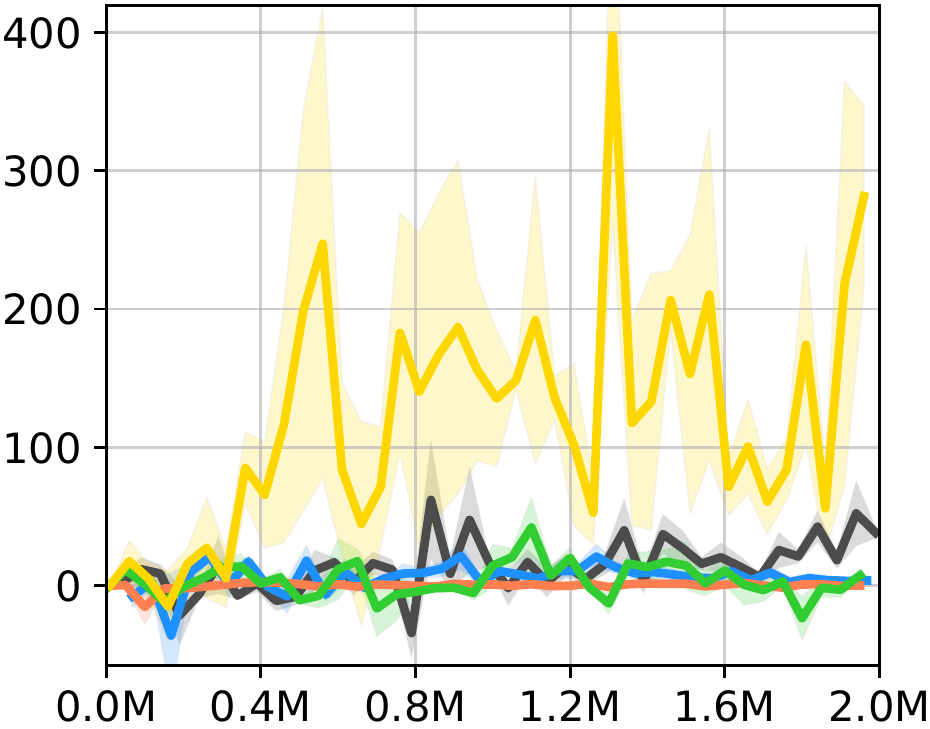}
    \caption{Ant Cripple}
\end{subfigure}

    \vspace*{-3pt}
    \caption{Learning curves of episode rewards on test tasks, averaged over 3 runs. The x-axis is total number of timesteps and the y-axis is average episode reward.
    Shadow areas indicate standard error.}
    \label{fig:ours_baseline}
\end{figure*}

\subsection{Policy Transfer with Fixed Dataset}

\label{sec:fixed}
We test our proposed action translator with fixed datasets of transitions aggregated from pairs of source and target tasks.
On MuJoCo environments HalfCheetah and Ant, we create tasks with varying dynamics as in \cite{zhou2019environment,lee2020context,zhang2020learning}.
We keep default physics parameters in source tasks and modify them to yield noticeable changes in the dynamics for target tasks.
On HalfCheetah, the tasks differ in the armature.
On Ant, we set different legs crippled.
A well-performing policy is pre-trained on the source task with TD3 algorithm \citep{fujimoto2018addressing} and dense rewards.
We then gather training data with mediocre policies on the source and target tasks.
We also include object manipulation tasks on MetaWorld benchmark \citep{yu2020meta}. Operating objects with varied physics properties requires the agent to handle different dynamics.
The knowledge in grasping and pushing a cylinder might be transferrable to tasks of moving a coffee mug or a cube.
The agent gets a reward of 1.0 if the object is in the goal location. Otherwise, the reward is 0.
We use the manually-designed good policy as the source policy and collect transition data by adding noise to the action drawn from the good policy.

\begin{table}[!ht]
\centering
\setlength{\tabcolsep}{2pt}
\begin{tabular}{c|ccc}
\toprule
\makecell{Setting}  & \makecell{Source \\policy} & 
\makecell{Transferred \\ policy \\ \citep{zhang2020learning}} & \makecell{Transferred \\ policy \\(Ours)} \\ 
\midrule
\makecell{HalfCheetah} & \makecell{2355.0} &
\makecell{\textbf{3017.1}\small{($\pm$44.2)}} &
\makecell{2937.2\small{($\pm$9.5)}} \\ 
\makecell{Ant} & 
\makecell{55.8} &
\makecell{97.2\small{($\pm$2.5)}} &
\makecell{\textbf{208.1}\small{($\pm$8.2)}}  \\ 
\midrule
\makecell{Cylinder-Mug}& 0.0 & 
\makecell{308.1\small{($\pm$75.3)}} &
\makecell{\textbf{395.6}\small{($\pm$19.4)}}\\
\makecell{Cylinder-Cube}& 0.0 & \makecell{262.4\small{($\pm$48.1)}} & \makecell{\textbf{446.1}\small{($\pm$1.1)}} \\
\bottomrule
\end{tabular}
\caption{Mean ($\pm$ standard error) of episode rewards over 3 runs, comparing source and transferred policy on target task. 
}
\label{tab:fixed_dataset}
\end{table}

\begin{figure}[!ht]
\centering
\begin{subfigure}{.15\textwidth}
  \centering
  \includegraphics[width=\linewidth]{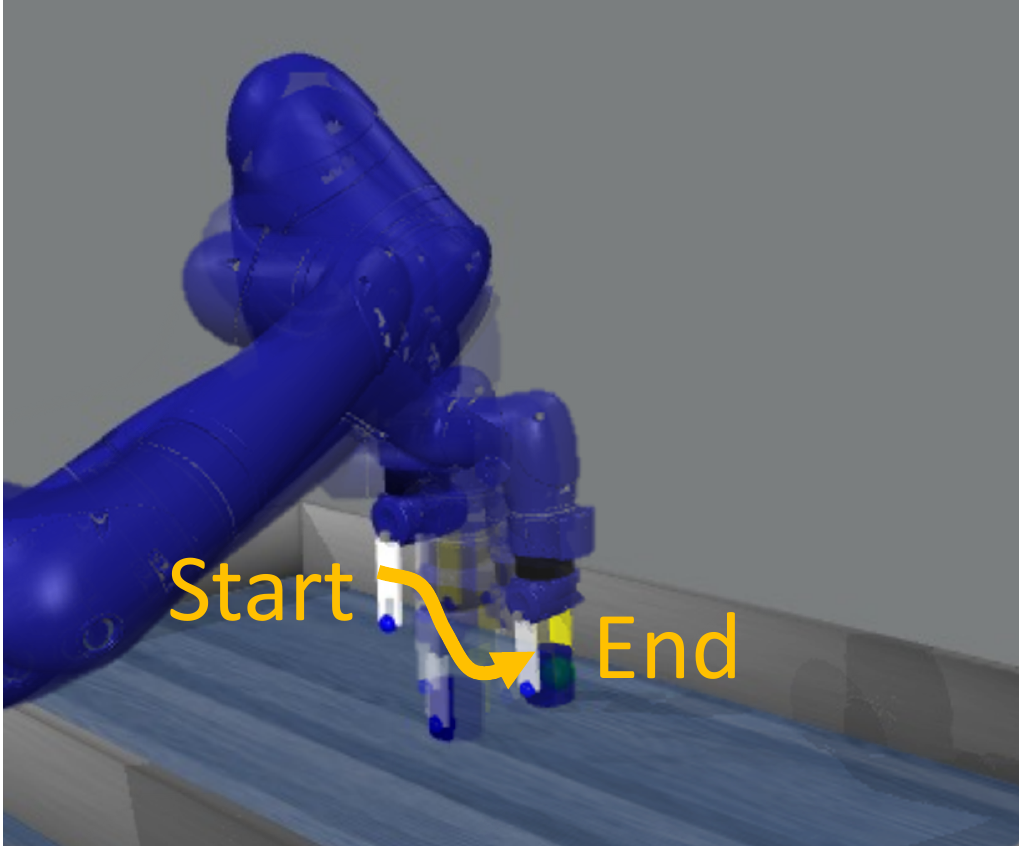}
  \caption{Source policy on source task}
  \label{fig:s2s}
\end{subfigure}%
\begin{subfigure}{.15\textwidth}
  \centering
  \includegraphics[width=\linewidth]{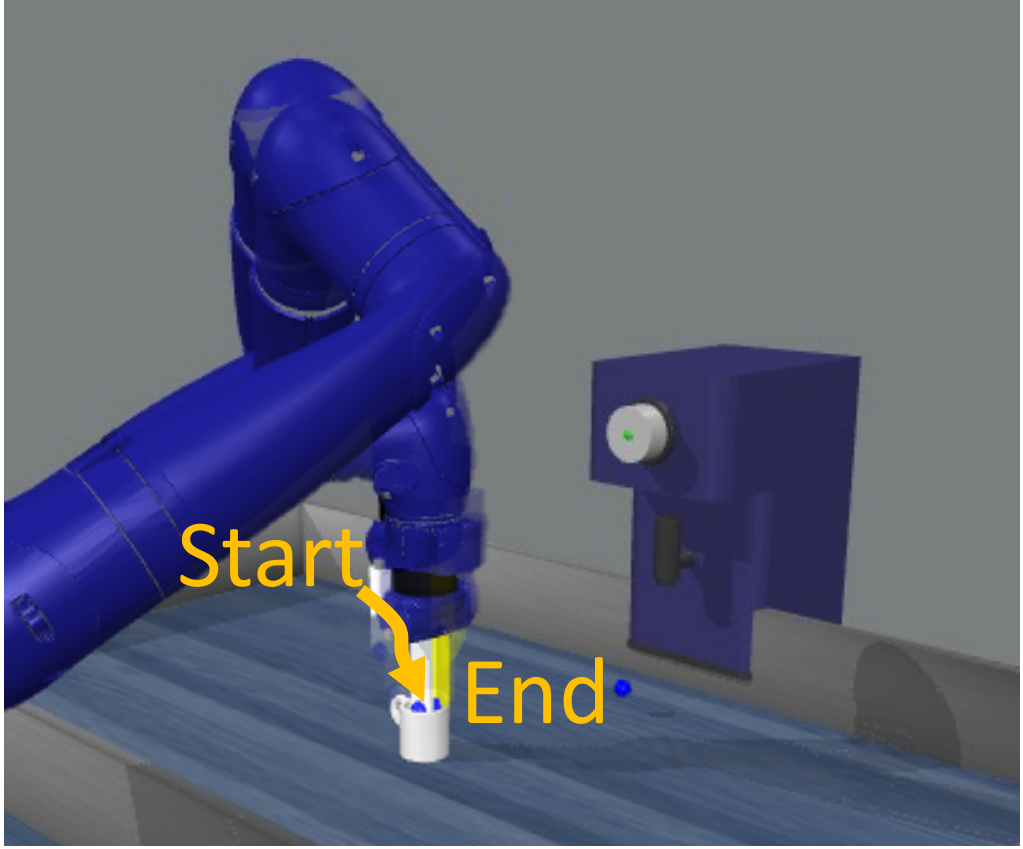}
  \caption{Source policy on target task}
  \label{fig:s2t}
\end{subfigure}%
\begin{subfigure}{.15\textwidth}
  \includegraphics[width=\linewidth]{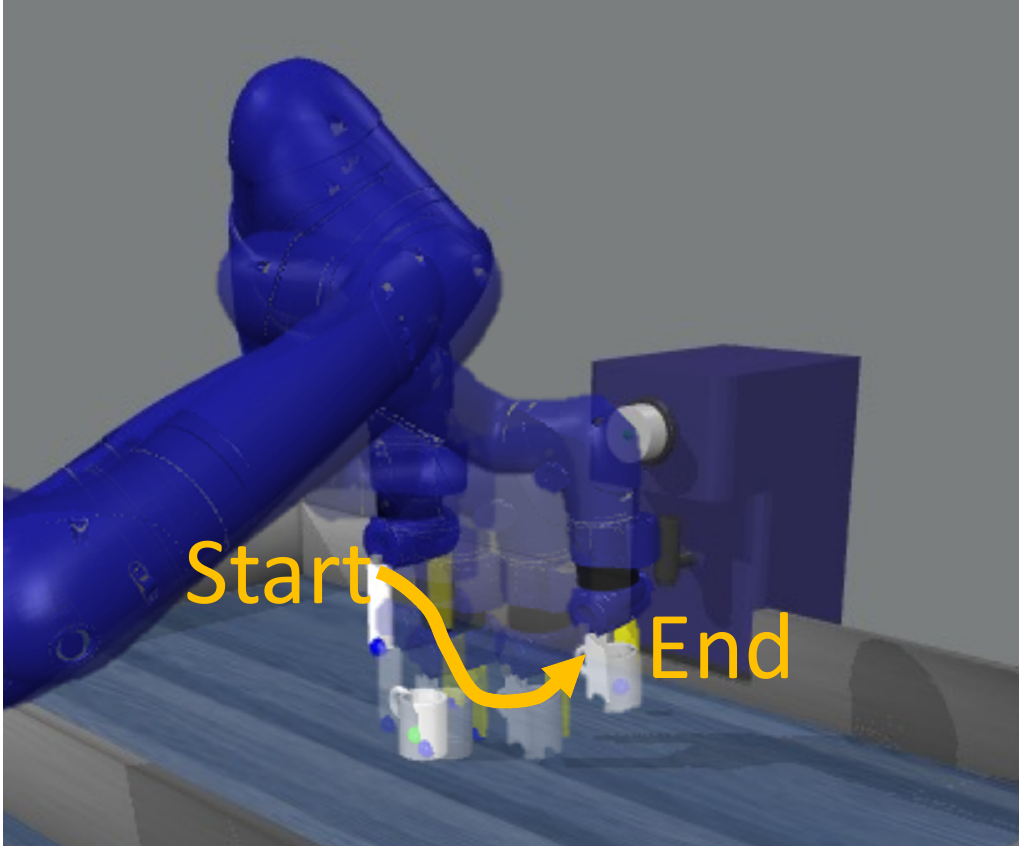}
  \caption{Transferred policy on target task}
  \label{fig:t2t}
\end{subfigure}
\caption{Robot arm moving paths on source (pushing a \textit{cylinder}) or target task (moving a \textit{mug} to a coffee machine).}
\label{fig:path}
\end{figure}

As presented in Tab.~\ref{tab:fixed_dataset}, directly applying a good source policy on the target task performs poorly. 
We learn dynamics model $F$ on target task with $\smash{\mathcal{L}_{forw}}$ and action translator $H$ with $\smash{\mathcal{L}_{trans}}$.
From a single source task to a single target task, the transferred policy with our action translator (without conditioning on the task context) yields episode rewards significantly better than the source policy on the target task.
Fig.~\ref{fig:path} visualizes moving paths of robot arms.
The transferred policy on target task resembles the source policy on source task, while the source policy has trouble grasping the coffee mug on target task.
Videos of agents' behavior are in supplementary materials. 
Tab.~\ref{tab:fixed_dataset} reports experimental results of baseline \citep{zhang2020learning} transferring the source policy based on action correspondence.
It proposes to learn an action translator with three loss terms: adversarial loss, domain cycle-consistency loss, and dynamic cycle-consistency loss.
Our loss $\smash{\mathcal{L}_{trans}}$ (Equation~\ref{eq:trans}) draws upon an idea analogous to dynamic cycle-consistency though we have a more expressive forward model $\smash{F}$ with context variables.
When $\smash{F}$ is strong and reasonably generalizable, domain cycle-consistency loss training the inverse action translator and adversarial loss constraining the distribution of translated action may not be necessary.
Ours with a simpler objective function is competitive with \citet{zhang2020learning}. 

\begin{figure}[!ht]
\centering
\begin{subfigure}{.23\textwidth}
  \centering
  \includegraphics[width=\linewidth]{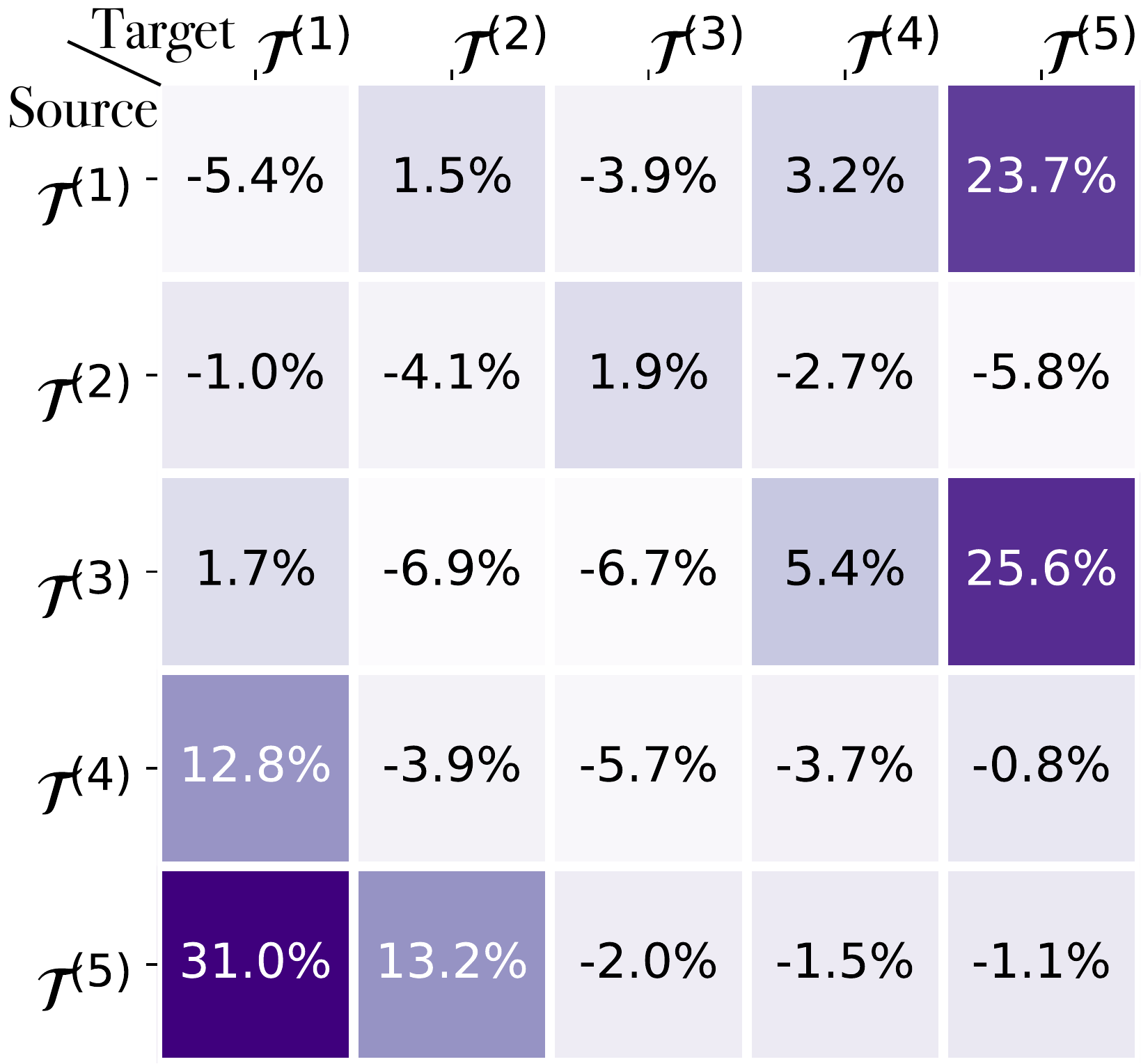}
  \caption{HalfCheetah}
  \label{fig:cheetah_heatmap}
\end{subfigure}%
\hspace{2pt}
\begin{subfigure}{.23\textwidth}
  \centering
  \includegraphics[width=\linewidth]{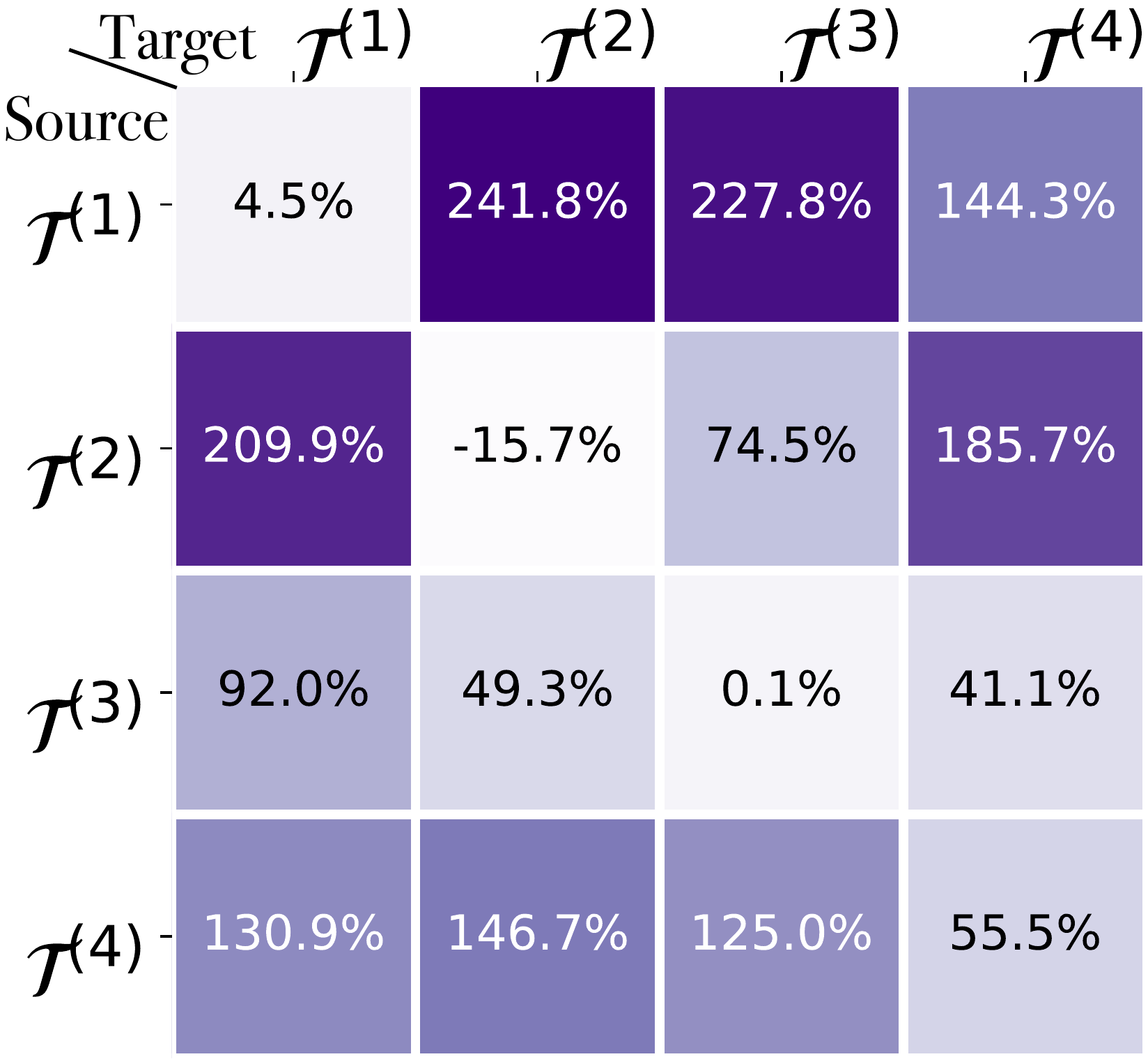}
  \caption{Ant}
  \label{fig:ant_heatmap}
\end{subfigure}
\caption{Improvement transferred policy over source policy. 
}
\label{fig:heatmap_improvement}
\end{figure}

We extend the action translator to multiple tasks by conditioning $\smash{H}$ on context variables of source and target tasks. We measure the improvement of our transferred policy over the source policy on the target tasks.
On HalfCheetah tasks $\smash{\mathcal{T}^{(1)} \cdots \mathcal{T}^{(5)}}$, the armature becomes larger.
As the physics parameter in the target task deviates more from source task, the advantage of transferred policy tends to be more significant (Fig.~\ref{fig:cheetah_heatmap}), because the performance of transferred policy does not drop as much as source policy.
We remark that the unified action translator is for any pair of source and target tasks.
So action translation for the diagonal elements might be less than $0\%$. 
For each task on Ant, we set one of its four legs crippled, so any action applied to the crippled leg joints is set as 0.
Ideal equivalent action does not always exist across tasks with different crippled legs in this setting.
Therefore, it is impossible to minimize $d$ in Proposition~\ref{cor:bound_policy} as 0.
Nevertheless, the inequality proved in Proposition~\ref{cor:bound_policy} still holds and policy transfer empirically shows positive improvement on most source-target pairs (Fig.~\ref{fig:ant_heatmap}). 

\subsection{Comparison with Context-based Meta-RL}
\label{sec:ours_combination}
We evaluate MCAT combining policy transfer with context-based TD3 in meta-RL problems. The action translator is trained dynamically with data maintained in replay buffer and the source policy keeps being updated.
On MuJoCo, we modify environment physics parameters (e.g. size, mass, damping) that affect the transition dynamics to design tasks. We predefine a fixed set of physics parameters for training tasks and unseen test tasks.
In order to test algorithms' ability in tackling difficult tasks, environment rewards are delayed to create sparse-reward RL problems \citep{oh2018self,tang2020self}.
In particular, we accumulate dense rewards over $n$ consecutive steps, and the agent receives the delayed feedback every $n$ step or when the episode terminates.
To fully exploit the good data collected from our transferred policy, we empirically incorporate self-imitation learning (SIL) \citep{oh2018self}, which imitates the agent's own successful past experiences to further improve the policy.

We compare with several context-based meta-RL methods: MQL \citep{fakoor2019meta}, 
PEARL \citep{rakelly2019efficient}, 
Distral \citep{teh2017distral}, 
and HiP-BMDP \citep{zhang2020robust}. 
Although the baselines perform well on MuJoCo environments with dense rewards, the delayed environment rewards degrade policy learning (Tab.~\ref{tab:ours_baseline}, Fig.~\ref{fig:ours_baseline}) because the rare transitions with positive rewards are not fully exploited.
In contrast, MCAT shows a substantial advantage in performance and sample complexity on both the training tasks and the test tasks.
Notably, the performance gap is more significant in more complex environments (e.g. HalfCheetah and Ant with higher-dimensional state and sparser rewards).
We additionally analyze effect of SIL in Appendix~D.4. SIL brings improvements to baselines but MCAT still shows obvious advantages.

\begin{table}[!ht]
\centering
\setlength{\tabcolsep}{1.5pt}
\begin{tabular}{c|ccccc}
\toprule
Setting & \makecell{Hopper \\ Size} & \makecell{Half\\Cheetah \\ Armature} &     \makecell{Half\\Cheetah \\ Mass} & \makecell{Ant \\ Damp} & 
\makecell{Ant \\ Cripple} \\ \midrule
MQL & \makecell{1607.5}
& \makecell{-77.9}
& \makecell{-413.9}
& \makecell{103.1}
& \makecell{38.2}
 \\
PEARL & \makecell{1755.8}
& \makecell{-18.8}
& \makecell{25.9}
& \makecell{73.2}
& \makecell{3.5}
\\
Distral & \makecell{1319.8}
& \makecell{566.9}
& \makecell{-29.5}
& \makecell{90.5}
& \makecell{-0.1}
\\
HiP-BMDP & \makecell{1368.3}
& \makecell{-102.4}
& \makecell{-74.8}
& \makecell{33.1}
& \makecell{7.3}
\\
MCAT(Ours) & \makecell{\textbf{1914.8}}
& \makecell{\textbf{2071.5}}
& \makecell{\textbf{1771.1}}
& \makecell{\textbf{624.6}}
& \makecell{\textbf{281.6}}
\\
\bottomrule
\end{tabular}
\caption{Test rewards at 2M timesteps, averaged over 3 runs.}
\label{tab:ours_baseline}
\end{table}

\subsection{Ablative Study}
\label{sec:ablation}
\textbf{Effect of Policy Transfer} Our MCAT is implemented by combining context-based TD3, self-imitation learning, and policy transfer (PT).
We investigate the effect of policy transfer. In Tab.~\ref{tab:effect_pt}.
MCAT significantly outperforms MCAT w/o PT, because PT facilitates more balanced performance across training tasks and hence better generalization to test tasks.
This empirically confirms that policy transfer is beneficial in meta-RL on sparse-reward tasks.

\begin{table}[!ht]
\centering
\setlength{\tabcolsep}{1.5pt}
\begin{tabular}{c|ccccc}
\toprule
Setting & \makecell{Hopper \\ Size} & \makecell{Half\\Cheetah \\ Armature} &     \makecell{Half\\Cheetah \\ Mass} & \makecell{Ant \\ Damp} & 
\makecell{Ant \\ Cripple} \\ \midrule
MCAT w/o PT & \makecell{1497.5}
& \makecell{579.1}
& \makecell{-364.3}
& \makecell{187.7}
&
\makecell{92.4}
 \\
MCAT & \makecell{1982.1}
& \makecell{1776.8}
& \makecell{67.1}
& \makecell{211.8}
& \makecell{155.7}
\\
Improve(\%) & 32.3  & 206.8 & 118.4 & 12.8 & 68.5 \\ \bottomrule
\end{tabular}
\caption{Test rewards at 1M timesteps. We report improvements brought by policy transfer (PT).}
\label{tab:effect_pt}
\end{table}

\noindent \textbf{Sparser Rewards}
We analyze MCAT when rewards are delayed for different numbers of steps (Tab.~\ref{tab:more_sparse}).
When rewards are relatively dense (i.e. delay step is 200), during training, the learned policy can reach a high score on each task without the issue of imbalanced performance among multiple tasks.
MCAT w/o PT and MCAT perform comparably well within the standard error.
However, as the rewards become sparser, it requires longer sequences of correct actions to obtain potentially high rewards.
Policy learning struggles on some tasks and policy transfer plays an important role to exploit the precious good experiences on source tasks.
Policy transfer brings more improvement on sparser-reward tasks.

\begin{table}[!ht]
\centering
\setlength{\tabcolsep}{1.5pt}
\begin{tabular}{c|ccc|ccc}
\toprule
\multicolumn{1}{c|}{Setting}     & \multicolumn{3}{c|}{Armature}                                                  & \multicolumn{3}{c}{Mass}            \\ 
\midrule
\multicolumn{1}{c|}{Delay steps} & \multicolumn{1}{c}{200} & \multicolumn{1}{c}{350} & \multicolumn{1}{c|}{500} & \multicolumn{1}{c}{200} &
\multicolumn{1}{c}{350} &
\multicolumn{1}{c}{500} \\ 
\midrule
MCAT w/o PT &  \makecell{2583.2}
&  \makecell{1771.7}
&  \makecell{579.1}
&  \makecell{709.6}
&  \makecell{156.6}
&  \makecell{-364.2}
\\ 
MCAT &\makecell{2251.8}
&\makecell{2004.5}
&\makecell{1776.8}
&\makecell{666.7}
&\makecell{247.8}
&\makecell{67.1}
\\ 
Improve(\%) &  -12.8 & 13.1 & 206.9 & -6.1 &  58.2 & 118.4 \\ 
\bottomrule
\end{tabular}
\caption{Test rewards at 1M timestpes averaged over 3 runs, on HalfCheetah with \textit{armature} / \textit{mass} changing across tasks.}
\label{tab:more_sparse}
\end{table}

 In Appendix, we further provide ablative study about More Diverse Tasks (D.3), Effect of SIL (D.4) and Effect of Contrastive Loss (D.5). Appendix~D.6 shows that trivially combining the complex action translator \citep{zhang2020learning} with context-based meta-RL underperforms MCAT.

\section{Discussion}
\label{sec:discussion}
The scope of MCAT is for tasks with varying dynamics, same as many prior works \citep{yu2017preparing,nagabandi2018learning,nagabandi2018deep,zhou2019environment}. our theory and method of policy transfer can be extended to more general cases (1) tasks with varying reward functions (2) tasks with varying state \& action spaces.

Following the idea in Sec.~\ref{sec:theory}, on two general MDPs, we are interested in equivalent state-action pairs achieving the same reward and transiting to equivalent next states.
Similar to Proposition~\ref{cor:bound_policy}, we can prove that, on two general MDPs, for two correspondent states $\smash{s^{(i)}}$ and $\smash{s^{(j)}}$, the value difference $\smash{|V^{\pi^{(i)}}(s^{(i)}, \mathcal{T}^{(i)}) - V^{\pi^{(j)}}(s^{(j)}, \mathcal{T}^{(j)})|}$ is upper bounded by $\frac{d}{1-\gamma}$, where $d$ depends on $\smash{D_{TV}}$ between the next state distribution on source task and the probability distribution of correspondent next state on target task.
As an extension, we learn a state translator jointly with our action translator to capture state and action correspondence.
Compared with \citet{zhang2020learning} learning both state and action translator, we simplify the objective function training action translator and afford the theoretical foundation.
For (1) tasks with varying reward functions, we conduct experiments on MetaWorld moving the robot arm to a goal location. The reward at each step is inversely proportional to its distance from the goal location. 
We fix a goal location 
on source task and set target tasks with distinct goal locations. 
Furthermore, we evaluate our approach on 2-leg and 3-leg HalfCheetah.
We can test our idea on (2) tasks with varying state and action spaces of different dimensions because the agents have different numbers of joints on the source and target task.
Experiments
demonstrate that ours with a simpler objective function than the baseline \citep{zhang2020learning} can transfer the source policy to perform well on the target task.
Details of theorems, proofs, and experiments are in Appendix~E.

\section{Conclusion}
Meta-RL with long-horizon, sparse-reward tasks is challenging because an agent can rarely obtain positive rewards, and handling multiple tasks simultaneously requires massive samples from distinctive tasks. We propose a simple yet effective objective function to learn an action translator for multiple tasks and provide the theoretical ground. We develop a novel algorithm MCAT using the action translator for policy transfer to improve the performance of off-policy, context-based meta-RL algorithms. We empirically show its efficacy in various environments and verify that our policy transfer can offer substantial gains in sample complexity.

\clearpage
\section*{Acknowledgements}
This work was supported in part by NSF CAREER IIS-1453651, NSF SES 2128623, and LG AI Research. Any opinions, findings, conclusions or recommendations expressed here are those of the authors and do not necessarily reflect views of the sponsor.

\bibliography{references}

\clearpage
\setcounter{secnumdepth}{2}
\onecolumn
{\LARGE \textsl{Appendix:}}

\appendix

\section{Bound Value Difference in Policy Transfer}
\label{app:proof}
In this section, we provide detailed theoretical ground for our policy transfer approach, as a supplement to Sec.~\ref{sec:theory}. 
%
We first define a binary relation for actions to describe the correspondent actions behaving equivalently on two MDPs (Definition~\ref{def:binary_relation}). 
%
Building upon the notion of action equivalence, we derive the upper bound of value difference between policies on two MDPs (Theorem~\ref{thm:bound_policy}). 
Finally, we reach a proposition
for the upper bound of value difference
(Proposition~\ref{cor:bound_policy}) to explain that minimizing our objective function results in bounding the value difference between the source and transferred policy.

\begin{restatable}{definition}{action_equivalence}
\label{def:binary_relation}
Given two MDPs $\smash{\mathcal{T}^{(i)}=\{\mathcal{S}, \mathcal{A}, p^{(i)}, r^{(i)}, \gamma, \rho_0\}}$ and  $\smash{\mathcal{T}^{(j)}=\{\mathcal{S}, \mathcal{A}, p^{(j)}, r^{(j)}, \gamma, \rho_0\}}$ with the same state space and action space, for each state $\smash{s \in \mathcal{S}}$,
we define a binary relation $\smash{B_s \in \mathcal{A} \times \mathcal{A}}$ called \textbf{action equivalence relation}.
For any action $a^{(i)}\in \mathcal{A}$, $a^{(j)}\in \mathcal{A}$, if $(a^{(i)}, a^{(j)}) \in B_s$ (i.e. $a^{(i)} B_s a^{(j)}$), the following conditions hold:
\begin{equation}
    r^{(i)}(s, a^{(i)}) = r^{(j)}(s, a^{(j)}) \text{ and } p^{(i)}(\cdot|s, a^{(i)}) = p^{(j)}(\cdot|s, a^{(j)})
\end{equation}

\end{restatable}

Based on Definition~\ref{def:binary_relation}, at state $s$, action $a^{(i)}$ on $\smash{\mathcal{T}^{(i)}}$ is equivalent to action $a^{(j)}$ on $\smash{\mathcal{T}^{(j)}}$ if $a^{(i)} B_s a^{(j)}$. Note that the binary relation $B_s$ is defined for each $s$ separately. The action equivalence relation might change on varied states.
On two MDPs with the same dynamic and reward functions, it is trivial to get the equivalent action with identity mapping. However, we are interested in more complex cases where the reward and dynamic functions are not identical on two MDPs.

Ideally, the equivalent action always exists on the target MDP $\smash{\mathcal{T}^{(i)}}$ for any state-action pair on the source MDP $\smash{\mathcal{T}^{(j)}}$ and there exists an action translator function $\smash{H:\mathcal{S}\times \mathcal{A} \rightarrow \mathcal{A}}$ to identify the exact equivalent action.
Starting from state $s$, the translated action $\smash{\tilde{a}=H(s, a)}$ on the task $\smash{\mathcal{T}^{(i)}}$ generates reward and next state distribution the same as action $a$ on the task $\smash{\mathcal{T}^{(j)}}$ (i.e. $\smash{\tilde{a} B_s a}$).
Then any deterministic policy $\smash{\pi^{(j)}}$ on the source task $\smash{\mathcal{T}^{(j)}}$ can be perfectly transferred to the target task $\smash{\mathcal{T}^{(i)}}$ with $\smash{\pi^{(i)}(s) = H(s, \pi^{(j)}(s))}$.
The value of the policy $\smash{\pi^{(j)}}$ on the source task $\smash{\mathcal{T}^{(j)}}$ is equal to the value of transferred policy $\smash{\pi^{(i)}}$ on the target task $\smash{\mathcal{T}^{(i)}}$.

Without the assumption of existence of a perfect correspondence for each action, given any two deterministic policies $\smash{\pi^{(j)}}$ and $\smash{\pi^{(i)}}$, we prove that the difference in the policy value is upper bounded by a scalar $\frac{d}{1-\gamma}$ depending on L1-distance between reward functions $\smash{|r^{(j)}(s, \pi^{(j)}(s)) - r^{(i)}(s, \pi^{(i)}(s))|}$ and total-variation distance between next state distributions $\smash{D_{TV}(p^{(j)}(\cdot|s,\pi^{(j)}(s)), p^{(i)}(\cdot|s,\pi^{(i)}(s)))}$.

\begin{restatable}{thm}{boundpolicytransfer}
\label{thm:bound_policy}
Let $\mathcal{T}^{(i)}=\{\mathcal{S}, \mathcal{A}, p^{(i)}, r^{(i)}, \gamma, \rho_0\}$ and  $\mathcal{T}^{(j)}=\{\mathcal{S}, \mathcal{A}, p^{(j)}, r^{(j)}, \gamma, \rho_0\}$ be two MDPs sampled from the distribution of tasks $p(\mathcal{T})$.
$\pi^{(i)}$ is a deterministic policy on $\mathcal{T}^{(i)}$ and $\pi^{(j)}$ is a deterministic policy on $\mathcal{T}^{(j)}$.
Let $M=\sup_{s\in \mathcal{S}} |V^{\pi^{(j)}}(s, \mathcal{T}^{(j)})|$, $d = \sup_{s\in\mathcal{S}} \left[|r^{(j)}(s, \pi^{(j)}(s)) - r^{(i)}(s, \pi^{(i)}(s))|+2\gamma M D_{TV}(p^{(j)}(\cdot|s,\pi^{(j)}(s)), p^{(i)}(\cdot|s,\pi^{(i)}(s)))\right]$. For $\forall s \in \mathcal{S}$, we have
\begin{equation}
    \left|V^{\pi^{(i)}}(s,\mathcal{T}^{(i)}) - V^{\pi^{(j)}}(s, \mathcal{T}^{(j)})\right| \leq \frac{d}{1-\gamma}
\end{equation}
\end{restatable}

\begin{proof}
Let $a^{(i)} = \pi^{(i)}(s)$ and $a^{(j)} = \pi^{(j)}(s)$. $s'$ denotes the next state following state $s$. $s''$ denotes the next state following $s'$.

We rewrite the value difference as:

\begin{eqnarray*}
V^{\pi^{(j)}}(s, \mathcal{T}^{(j)}) - V^{\pi^{(i)}}(s, \mathcal{T}^{(i)}) &=& r^{(j)}(s, a^{(j)})+ \gamma \sum_{s'}p^{(j)}(s'|s, a^{(j)}) V^{\pi^{(j)}}(s', \mathcal{T}^{(j)}) \\ \nonumber
&-& r^{(i)}(s, a^{(i)})  - \gamma  \sum_{s'}p^{(i)}(s'|s, a^{(i)}) V^{\pi^{(i)}}(s',\mathcal{T}^{(i)}) \\ \nonumber
&=& (r^{(j)}(s, a^{(j)})- r^{(i)}(s, a^{(i)})) \\ \nonumber
&+& \gamma \left[ \sum_{s'}p^{(j)}(s'|s, a^{(j)}) V^{\pi^{(j)}}(s', \mathcal{T}^{(j)}) - \sum_{s'}p^{(i)}(s'|s, a^{(i)}) V^{\pi^{(i)}}(s', \mathcal{T}^{(i)}) \right] \\ \nonumber
&&\text{\hspace{1in} {\color{blue} *minus and plus $\gamma \sum_{s'}p^{(i)}(s'|s, a^{(i)})V^{\pi^{(j)}}(s', \mathcal{T}^{(j)})$}} \\ \nonumber
&=& (r^{(j)}(s, a^{(j)})- r^{(i)}(s, a^{(i)})) \\ \nonumber
&+& \gamma \sum_{s'}\left[p^{(j)}(s'|s, a^{(j)})-p^{(i)}(s'|s, a^{(i)})\right] V^{\pi^{(j)}}(s', \mathcal{T}^{(j)}) \\ \nonumber
&+& \gamma \sum_{s'}p^{(i)}(s'|s, a^{(i)})\left[V^{\pi^{(j)}}(s', \mathcal{T}^{(j)}) -V^{\pi^{(i)}}(s', \mathcal{T}^{(i)})\right]
\end{eqnarray*}

\clearpage
Then we consider the absolute value of the value difference:

\begin{eqnarray*}
\left|V^{\pi^{(j)}}(s, \mathcal{T}^{(j)}) - V^{\pi^{(i)}}(s, \mathcal{T}^{(i)})\right|
&\leq & \left|r^{(j)}(s, a^{(j)})- r^{(i)}(s, a^{(i)})\right| \\ \nonumber
&+& \gamma \left|\sum_{s'}\left[p^{(j)}(s'|s, a^{(j)})-p^{(i)}(s'|s, a^{(i)})\right] V^{\pi^{(j)}}(s', \mathcal{T}^{(j)})\right| \\ \nonumber
&+&\gamma \left|\sum_{s'}p^{(i)}(s'|s, a^{(i)})\left[V^{\pi^{(j)}}(s', \mathcal{T}^{(j)}) -V^{\pi^{(i)}}(s',\mathcal{T}^{(i)}) \right]\right| \\ \nonumber
&&\text{ {\color{blue} *property of total variation distance when the set is countable}} \\ \nonumber
&=& \left|r^{(j)}(s, a^{(j)})- r^{(i)}(s, a^{(i)})\right| \\ \nonumber
&+&  2\gamma M D_{TV}(p^{(j)}(\cdot|s, a^{(j)}),p^{(i)}(\cdot|s, a^{(i)}) ) \\ \nonumber
&+&\gamma \left|\sum_{s'}p^{(i)}(s'|s, a^{(i)})\left[V^{\pi^{(j)}}(s',\mathcal{T}^{(j)}) -V^{\pi^{(i)}}(s', \mathcal{T}^{(i)}) \right]\right| \\ \nonumber
&\leq& d + \gamma \left|\sum_{s'}p^{(i)}(s'|s, a^{(i)})\left[V^{\pi^{(j)}}(s', \mathcal{T}^{(j)}) -V^{\pi^{(i)}}(s', \mathcal{T}^{(i)}) \right]\right|\\ \nonumber
&\leq& d + \gamma \sup_{s'} \left|V^{\pi^{(j)}}(s', \mathcal{T}^{(j)})-V^{\pi^{(i)}}(s', \mathcal{T}^{(i)})\right|\\ \nonumber
&&\text{ {\hspace{2.6in}\color{blue} *by induction}} \\ \nonumber
&\leq& d + \gamma \left[d + \gamma\sup_{s''}\left|V^{\pi^{(j)}}(s'',\mathcal{T}^{(j)})-V^{\pi^{(i)}}(s'', \mathcal{T}^{(i)})\right| \right]\\ \nonumber
&\leq& d+\gamma d +\gamma^2 \sup_{s''}\left|V^{\pi^{(j)}}(s'', \mathcal{T}^{(j)})-V^{\pi^{(i)}}(s'', \mathcal{T}^{(i)})\right| \\ \nonumber
&\leq& \cdots \\ \nonumber
&\leq& d +\gamma d + \gamma^2 d + \gamma^3 d + \cdots 
= \frac{d}{1-\gamma} \\ \nonumber
\end{eqnarray*}
\end{proof}

For a special case where reward function $\smash{r(s,a,s')}$ only depends on the current state $s$ and next state $s'$, the upper bound of policy value difference is only related to the distance in next state distributions.
\boundpolicytransfertv*

\begin{proof}
Let $a^{(i)} = \pi^{(i)}(s)$ and $a^{(j)} = \pi^{(j)}(s)$. $s'$ denotes the next state following state $s$. $s''$ denotes the next state following $s'$.

In the special case of $r^{(i)}(s, a^{(i)}, s')= r(s, s')$, the value of policy can be written as:

\begin{eqnarray*}
V^{\pi^{(i)}}(s, \mathcal{T}^{(i)}) &=& r^{(i)}(s, a^{(i)}) + \gamma \sum_{s'}p^{(i)}(s'|s, a^{(i)}) V^{\pi^{(i)}}(s',\mathcal{T}^{(i)}) \\ \nonumber
&=& \sum_{s'}p^{(i)}(s'|s, a^{(i)}) r(s,s') + \gamma \sum_{s'}p^{(i)}(s'|s, a^{(i)}) V^{\pi^{(i)}}(s',\mathcal{T}^{(i)}) \\ \nonumber
&=& \sum_{s'}p^{(i)}(s'|s, a^{(i)}) \left[r(s, s') + \gamma V^{\pi^{(i)}}(s', \mathcal{T}^{(i)})\right] 
\end{eqnarray*}

\clearpage
We can derive the value difference:

\begin{eqnarray*}
V^{\pi^{(j)}}(s, \mathcal{T}^{(j)}) - V^{\pi^{(i)}}(s, \mathcal{T}^{(i)}) &=&  \sum_{s'}p^{(j)}(s'|s, a^{(j)})\left[r(s, s')+\gamma V^{\pi^{(j)}}(s', \mathcal{T}^{(j)})\right] - \sum_{s'}p^{(i)}(s'|s, a^{(i)}) \left[r(s, s') + \gamma V^{\pi^{(i)}}(s', \mathcal{T}^{(i)})\right] \\ \nonumber
&&\text{ {\hspace{1.5in}\color{blue} *minus and plus $ \sum_{s'}p^{(i)}(s'|s, a^{(i)})\left[r(s, s')+\gamma V^{\pi^{(j)}}(s', \mathcal{T}^{(j)})\right]$}} \\ \nonumber
&&\text{ {\hspace{1.5in}\color{blue} **combine the first two terms, combine the last two terms}} \\ \nonumber
&=& \sum_{s'}\left[p^{(j)}(s'|s, a^{(j)}) - p^{(i)}(s'|s, a^{(i)})\right] \left[r(s, s')+\gamma V^{\pi^{(j)}}(s', \mathcal{T}^{(j)})\right] \\ \nonumber
&+& \gamma\sum_{s'}p^{(i)}(s'|s, a^{(i)})\left[V^{\pi^{(j)}}(s', \mathcal{T}^{(j)}) -V^{\pi^{(i)}}(s', \mathcal{T}^{(i)}) \right]  \\ \nonumber
\end{eqnarray*}

Then we take absolute value of the value difference:

\begin{eqnarray*}
\left| V^{\pi^{(j)}}(s, \mathcal{T}^{(j)}) - V^{\pi^{(i)}}(s, \mathcal{T}^{(i)})\right|
&\leq &   2 M D_{TV}(p^{(j)}(\cdot|s, a^{(j)}),p^{(i)}(\cdot|s, a^{(i)}) ) \\ \nonumber
&+& \gamma \left|\sum_{s'}p^{(i)}(s'|s, a^{(i)})\left[V^{\pi^{(j)}}(s', \mathcal{T}^{(j)}) -V^{\pi^{(i)}}(s', \mathcal{T}^{(i)}) \right]\right| \\ \nonumber
&\leq& d + \gamma \left|\sum_{s'}p^{(i)}(s'|s, a^{(i)})\left[V^{\pi^{(j)}}(s', \mathcal{T}^{(j)}) -V^{\pi^{(i)}}(s', \mathcal{T}^{(i)}) \right]\right| \\ \nonumber
&\leq& d + \gamma \sup_{s'} \left|V^{\pi^{(j)}}(s', \mathcal{T}^{(j)})-V^{\pi^{(i)}}(s', \mathcal{T}^{(i)})\right|\\ \nonumber
&\leq& d + \gamma \left[d + \gamma\sup_{s''}\left|V^{\pi^{(j)}}(s'',\mathcal{T}^{(j)})-V^{\pi^{(i)}}(s'', \mathcal{T}^{(i)})\right| \right]\\ \nonumber
&\leq& d+\gamma d +\gamma^2 \sup_{s''}\left|V^{\pi^{(j)}}(s'', \mathcal{T}^{(j)})-V^{\pi^{(i)}}(s'', \mathcal{T}^{(i)})\right| \\ \nonumber
&\leq& \cdots \\ \nonumber
&\leq& d +\gamma d + \gamma^2 d + \gamma^3 d + \cdots = \frac{d}{1-\gamma}
\end{eqnarray*}
\end{proof}

\clearpage
\section{Algorithm of MCAT}

\label{app:algo}
\begin{algorithm}
    \caption{MCAT combining context-based meta-RL algorithm with policy transfer}
    \label{alg:mcat}
    \begin{algorithmic}[1]
        \STATE Initialize critic networks $Q_{\theta_1}$, $Q_{\theta_2}$ and actor network $\pi_{\phi}$ with random parameters $\theta_1$, $\theta_2$, $\phi$
        \STATE Initialize target networks $\theta'_1 \gets \theta_1$, $\theta'_2 \gets \theta_2$, $\phi'\gets \phi$
        \STATE Initialize replay buffer $\mathcal{B}=\mathcal{B}^{(1)}\cup \mathcal{B}^{(2)}\cup \cdots \cup \mathcal{B}^{(|\mathcal{T}|)}$ and $\mathcal{B}^{(i)} \gets \emptyset$ for each $i$.
        \STATE Initialize SIL replay buffer $\mathcal{D} \gets \emptyset$
        \STATE Initialize context encoder $C_{\psi_C}$, forward model $F_{\psi_F}$, action translator $H_{\psi_H}$ 
        \STATE Initialize set of trajectory rewards for shared policy on each task in recent timesteps as $R^{(i)}=\emptyset$, set of trajectory rewards for transferred policy from $\mathcal{T}^{(j)}$ to $\mathcal{T}^{(i)}$ in recent timesteps as $R^{(j)\rightarrow(i)}=\emptyset$. $\bar{R}$ denotes average episode rewards in the set.
        \FOR{each iteration}
            \STATE // Collect training samples
            \FOR{ each task $\mathcal{T}^{(i)}$}
                \IF{$R^{(i)}=\emptyset$}
                    \STATE use the shared policy in this episode
                \ELSIF{there exist $j \in {1, 2, \cdots, |\mathcal{T}|}$ such that $R^{(j)\rightarrow(i)}=\emptyset$ and $\bar{R}^{(j)} > \bar{R}^{(i)}$}
                    \STATE use transferred policy from source task $\mathcal{T}^{(j)}$ to target task $\mathcal{T}^{(i)}$ in this episode
                \ELSIF{there exist $j \in {1, 2, \cdots, |\mathcal{T}|}$, such that $j=\arg\max_{j'} \bar{R}^{(j')\rightarrow (i)}$ and $\bar{R}^{(j)\rightarrow (i)}>\bar{R}^{(i)}$}
                    \STATE use transferred policy from source task $\mathcal{T}^{(j)}$ to target task $\mathcal{T}^{(i)}$ in this episode
                \ELSE
                    \STATE use the shared policy in this episode
                \ENDIF
                \FOR{$t=1$ to TaskHorizon}
                    \STATE Get context latent variable $z_t=C_{\psi_C}(\tau_{t, K})$
                    \STATE Select the action $a$ based on the transferred policy or shared policy, take the action with noise
                    $a_t = a+\epsilon$ where
                    $\epsilon \sim \mathcal{N}(0, \sigma)$, observe reward $r_t$ and new state $s_{t+1}$.
                    Update $\mathcal{B}^{(i)}\gets \mathcal{B}^{(i)} \cup \{s_t, a_t, r_t, s_{t+1}, \tau_{t, K}\}$
                \ENDFOR
                \STATE Compute returns $R_t=\sum_{k=t}^\infty \gamma^{k-t}r_k$ and update $\mathcal{D}\gets \mathcal{D} \cup \{s_t, a_t, r_t, s_{t+1}, \tau_{t, K}, R_t\}$ for every step $t$ in this episode. 
                \STATE Update the average reward of shared policy on task $\mathcal{T}^{(i)}$ (i.e. $R^{(i)}$) if we took shared policy in this episode, or update the average reward of the transferred policy from $\mathcal{T}^{(j)}$ to $\mathcal{T}^{(i)}$ (i.e. $R^{(j)\rightarrow(i)}$) if we took the transferred policy.
            \ENDFOR
            \STATE // Update the context encoder $C_{\psi_C}$ and forward model $F_{\psi_F}$ with $\mathcal{L}_{forw}$ and $\mathcal{L}_{cont}$
            \STATE // Update the action translator $H_{\psi_H}$ with $\mathcal{L}_{trans}$
            \STATE // Update the critic network $Q_{\theta_1}$, $Q_{\theta_2}$ and actor network $\pi_{\phi}$ with TD3 and SIL objective function
            \FOR{step in training steps}
                \STATE Update $\theta_1$, $\theta_2$ for the critic networks to minimize $\mathcal{L}_{td3}+\mathcal{L}_{sil}$ (see Algorithm~\ref{alg:td3_critic})
                \STATE Update $\phi$ for the actor network with deterministic policy gradient
                \STATE Update the $\theta'_1$, $\theta'_2$, $\phi'$ for target networks with soft assignment
            \ENDFOR
            \STATE // Update the trajectory reward for shared policy and transferred policy if necessary
            \FOR{each task $\mathcal{T}^{(i)}$}
                \STATE pop out trajectory rewards in $R^{(i)}$ which were stored before the last G timesteps
                \STATE pop out trajectory rewards in $R^{(j)\rightarrow(i)} (\forall j)$ which were stored before the last G timesteps
            \ENDFOR
        \ENDFOR
    \end{algorithmic}
\end{algorithm}

\begin{algorithm}[!h]
\begin{algorithmic}[1]
    \caption{Compute critic loss based on TD3 algorithm and SIL algorithm}
    \label{alg:td3_critic}
    \STATE Sample batch data of transitions $(s_t, a_t, r_t, s_{t+1}, \tau_{t, K})\in\mathcal{B}$
    \STATE Get context variable $z_t=C_{\psi_C}(\tau_{t, K})$. Get next action $a_{t+1}\sim \pi_{\phi'}(z_t, s_{t+1})+\epsilon$, $\epsilon\sim \text{clip}(\mathcal{N}(0, \tilde{\sigma}), -c, c)$
    \STATE Get target value for critic network $y=r_t+\gamma \min_{l=1,2}Q_{\theta'_l}(z_t, s_{t+1}, a_{t+1})$.
    \STATE Compute TD error $\mathcal{L}_{td3}=\min_{l=1,2} (y-Q_{\theta_l}(z_t, s_{t+1}, a_{t+1}))^2$
    \STATE Sample batch data of transitions $(s_t, a_t, \tau_{t, K}, R_t)\in\mathcal{D}$
    \STATE Get context variable $z_t=C_{\psi_C}(\tau_{t, K})$. 
    Compute SIL loss $\mathcal{L}_{sil}= \sum_{l=1,2}\max(R_t-Q_{\theta_l}(z_t, s_t, a_t), 0)^2$
\end{algorithmic}
\end{algorithm}

\clearpage
\section{Experiment Details}
\label{app:experiment}
In this section, we explain more details for Section~\ref{sec:experiment} and show additional experimental results.

\subsection{Environment}
\paragraph{MuJoCo} We use Hopper, HalfCheetah and Ant environments from OpenAI Gym based on the MuJoCo physics engine \citep{todorov2012mujoco}. The goal is to move forward while keeping the control cost minimal. 

\begin{itemize}
\setlength\itemsep{2pt}
\item{\textbf{Hopper}} Hopper agent consists of 5 rigid links with 3 joints. Observation $s_t$ is an 11-dimension vector consisting of root joint's position (except for x-coordinate) and velocity angular position and velocity of all 3 joints.
Action $a_t$ lies in the space $\smash{[-1.0, 1.0]^3}$, which corresponds to the torques applied to 3 joints.
Reward $r_t=v_{\text{torso}, t}-0.001\|a_t\|^2+1.0$ means the forward velocity of the torso $v_{\text{torso}, t}$ minus the control cost for action $0.001\|a_t\|^2$ and plus the survival bonus $1.0$ at each step.
We modify the size of each rigid part to enlarge/contract the body of the agent, so we can create tasks with various dynamics.

\item{\textbf{HalfCheetah}} Half-cheetah agent consists of of 7 rigid links (1 for torso, 3 for forelimb, and 3 for hindlimb), connected
by 6 joints.
State $s_t$ is a 17-dimension vector consisting of root joint's position (except for x-coordinate) and velocity, angular position and velocity of all 6 joints.
Action $a_t$ is sampled from the space $\smash{[-1.0, 1.0]^6}$, representing the torques applied to each of the 6 joints.
Reward $r_t=v_{\text{torso}, t}-0.1\|a_t\|^2$ is the forward velocity of the torso minus the control cost for action.
In order to design multiple tasks with varying dynamics on HalfCheetah, we modify the armature value (similarly to \citep{zhang2020learning}) or scale the mass of each rigid link by a fixed scale factor (similarly to \citep{lee2020context}).

\item{\textbf{Ant}} Ant agent consists of 13 rigid links connected by 8 joints. Observation $s_t$ is a 27-dimension vector including information about the root joint's position and velocity, angular position and velocity of all 8 joints, and frame orientations.
Action $a_t\in[-1.0, 1.0]^8$ is the torques applied to each of 8 joints.
Reward is $r_t=v_{\text{torso}, t}+0.05$, meaning the velocity of moving forward plus the survival bonus $0.05$ for each step.
To change the environment dynamics, we modify the damping of every leg. Specifically, given a scale factor $d$, we
modify two legs to have damping multiplied by $d$, and the other two legs to have damping multiplied by $1/d$ (similarly to \citep{lee2020context}). Alternatively, we can cripple one of the agent's four legs to change the dynamics function. The torques applied to two joints on the crippled leg (i.e. two correspondent elements in actions) are set as 0. (similarly to \citep{seo2020trajectory}). 
\end{itemize}

\paragraph{MetaWorld} Additionally, we consider the tasks of pushing Cylinder, Coffee Mug and Cube.
They are named as push-v2, coffee-push-v2, and sweep-into-goal-v2 on MetaWorld benchmark \citep{yu2020meta} respectively.
The goal is to move the objects from a random initial location to a random goal location.
The observation is of dimension 14, consisting of the location of the robot hand, the distance between two gripper fingers, the location and position of the object, and the target location.
The action $a\in[-1.0, 1.0]^4$ controls the movement of the robot hand and opening/closing of the gripper. The reward is 1.0 when the object is close to the target location (i.e. distance less than 0.05).
Otherwise, the environment reward is 0.0.
The length of an episode is 500 steps.
The tasks of manipulating different objects have different dynamics.
We change the physics parameters armature and damping across tasks to make the policy transfer more challenging.


\subsection{Implementation Details for Policy Transfer with Fixed Dataset \& Source Policy}
\label{app:fixed_details}
In Section~\ref{sec:fixed}, we study the performance of policy transfer with our action translator with a fixed dataset and source policy.
In this experiment, we demonstrate our proposed policy transfer approach trained with fixed datasets and source policy outperforms the baselines. 
We provide the experimental details as follows.

\paragraph{Source Policy and Dataset}
\begin{itemize}
\item{\textbf{MuJoCo}}
On HalfCheetah, the armature value on the source and target task is 0.1 and 0.5 respectively. On Ant, the leg 0 is crippled on the source task while the leg 3 is crippled on the target task.
We train well-performing policies on the source tasks as source policies, and we also train mediocre policies on both source tasks and target tasks to obtain training data.

We apply the TD3 algorithm\citep{fujimoto2018addressing} and dense rewards to learn policies.
The hyperparameters for the TD3 algorithm are listed in Table~\ref{tab:TD3_5_1}. Specifically, during the start 25K timesteps, the TD3 agent collects data by randomly sampling from the action space. After the first 25K timesteps, the agent learns an deterministic policy based on the data collected in the replay buffer.
During training, the agent collects data with actions following the learned policy with Gaussian noise, and updates the replay buffer as well.
On HalfCheetah environment, we use the learned policy at 300K timesteps as good policy, and use the learned policy at 80K timesteps as mediocre policy.
On Ant environment, the learned policy at 400K timesteps and 20K timesteps are used as good policy and mediocre policy respectively.

With the mediocre policies, we collect 100K transition samples on the source and target tasks respectively.
During data collection, at each step, we record the following information: (a) current state; (b) current action drawn from the mediocre policies; (c) next state; (d) historical observations in the past 10 steps; (e) historical actions in the past 10 steps. The historical transition information are employed to learn the context model for forward dynamics prediction.

\clearpage
\item{\textbf{MetaWorld}} On source tasks, we keep the default physics parameters. However, on the target task,the value of armature and damping for the gripper joints is 0.1 multiplying the default. We get the manually designed good policies from official public code\footnote{https://github.com/rlworkgroup/metaworld/tree/master/metaworld/policies}. The performance of the good source policy is shown in Tab.~\ref{tab:app_fixed_dataset}.
By adding Gaussian noise following 
$\mathcal{N}(0, 1.0)$ to action drawn from the good policies, we collect 100K transition samples on the source and target tasks respectively.
\end{itemize}

\begin{table}[!t]
\centering

\setlength{\tabcolsep}{0.5pt}
\begin{tabular}{c|c}
\toprule
\makecell{Parameter name} & \makecell{Value} \\ 
\midrule
Start Timesteps & 2.5e4\\
Gaussian exploration noise $\sigma$ & 0.1\\
Batch Size & 256\\
Discount $\gamma$ & 0.99\\
Target network update rate & 5e-3\\
Policy noise $\tilde{\sigma}$ & 0.2\\
Noise clip $c$ & 0.5\\
Policy update frequency & 2\\
Replay buffer size & 1e6\\
Actor learning rate & 3e-4\\
Critic learning rate & 3e-4\\
Optimizer & Adam\\
Actor layers & 3\\
Hidden dimension & 256\\
\bottomrule
\end{tabular}
\caption{The hyperparameters for TD3 algorithm.}
\label{tab:TD3_5_1}
\end{table}


With the fixed datasets on both source and targe tasks, we can train action translator to transfer the fixed source policy.
First, we learn the forward dynamics model.
Then we learn the action translator based on the well-trained forward dynamics model.
For fair comparison, we train the baseline \citep{zhang2020learning} and our action translator with the same dataset and source policy.
The hyperparameters and network structures applied in the baseline and our approach are introduced as follows

\paragraph{Transferred Policy \citep{zhang2020learning}}
This baseline is implemented using the code provided by \citet{zhang2020learning} \footnote{https://github.com/sjtuzq/Cycle\_Dynamics}.
The forward dynamics model first encodes the state and action as 128-dimensional vectors respectively via a linear layer with ReLU activation.
The state embedding and action embedding is then concatenated to predict the next state with an MLP with 2 hidden layers of 256 units and ReLU activation.
We train the forward dynamics model with batch size 32 and decaying learning rate from 0.001, 0.0003 to 0.0001.
In order to optimize the forward dynamics model, the objective function is L1-loss between the predicted next state and the actual next state.
With these hyper-parameters settings, we train the forward modelFand the context modelCfor30 epochs, each epoch with 10K steps.

The action translator first encodes the state and action as 128-dimensional vectors respectively via a linear layer with ReLU activation.
The state embedding and action embedding are then concatenated to generate the translated action via an MLP with 2 hidden layers of 256 units and ReLU activation.
As for the objective function with three terms: adversarial loss, domain cycle-consistency loss, and dynamics cycle-consistency loss, we tune three weights.
We train the action translator for 30 epochs.
After each epoch, the performance of transferred policy with the action translator is evaluated on the target task.
We average episode rewards in 100 episodes as the epoch performance.
Finally, we report the best epoch performance over the 30 epochs.

\begin{table}[!ht]
\centering

\setlength{\tabcolsep}{4pt}
\begin{tabular}{c|c|ccc}
\toprule
\makecell{Setting}  &
\footnotesize{\makecell{Source policy\\on source task}} & \footnotesize{\makecell{Source policy\\on target task}} & \footnotesize{\makecell{Transferred policy \\ \citep{zhang2020learning} on target task}} & \footnotesize{\makecell{Transferred policy \\(Ours) on target task}} \\ 
\midrule
\makecell{HalfCheetah} & 5121.4 & \makecell{2355.0} & 
\makecell{\textbf{3017.1}\tiny{($\pm$44.2)}} &
\makecell{2937.2\tiny{($\pm$9.5)}} \\ 
\makecell{Ant} & 476.8 &
\makecell{55.8} &
\makecell{97.2\tiny{($\pm$2.5)}} &
\makecell{\textbf{208.1}\tiny{($\pm$8.2)}}  \\ 
\midrule
\makecell{Cylinder-Mug}& 317.3&0.0 & 
\makecell{308.1\tiny{($\pm$75.3)}} &
\makecell{\textbf{395.6}\tiny{($\pm$19.4)}}\\
\makecell{Cylinder-Cube}& 439.7 &0.0 & \makecell{262.4\tiny{($\pm$48.1)}} & \makecell{\textbf{446.1}\tiny{($\pm$1.1)}} \\
\bottomrule
\end{tabular}

\caption{Performance of source and transferred policy on target task. This is expanding Tab.~\ref{tab:fixed_dataset} in the main text.
}
\label{tab:app_fixed_dataset}
\end{table}

\clearpage
\paragraph{Transferred Policy (Ours)} We encode the context features with $K=10$ past transitions. 
The historical state information is postprocessed as state differences between two consecutive states.
The historical transition at one step is concatenation of past 10 actions and past 10 postprocessed states.
The historical transition data are fed into an MLP with 3 hidden layers with [256, 128, 64] hidden units and Swish activation.
The context vector is of dimension 10.  The forward dynamics model is an MLP with 4 hidden layers of 200 hidden units and ReLU activation, predicting the state difference between two consecutive states in the future M=10 steps.
The learning rate is 0.001 and the batch size is 1024.
The objective function is simply $\mathcal{L}_{forw}+\mathcal{L}_{cont}$ (Equation~\ref{eq:forw} and Equation~\ref{eq:cont}).
With these hyper-parameters settings, we train the forward model $F$ and the context model $C$ for 30 epochs, each epoch with 10K steps.

The action translator $H$ first encodes state and action as 128-dimensional vectors respectively.
Then, the state embedding and action embedding is concatenated and fed into an MLP with 3 hidden layers of 256 units and ReLU activations.
We train the action translator with a decaying learning rate from 3e-4, 5e-5 to 1e-5, and the batch size is also 1024. With these hyper-parameters settings, we train the action translator for 30 epochs, each epoch with 3,000 steps.
The objective function is simply $\mathcal{L}_{trans}$ (Equation~\ref{eq:trans}).
After each epoch, the performance of the action translator is also evaluated on the target task via averaging the episode rewards in 100 episodes.
Finally, the best epoch performance over the 30 epochs is reported.

\paragraph{Context-conditioned Action Translator} We also demonstrate the performance of policy transfer on more than two tasks as heatmaps in Fig.~\ref{fig:heatmap_improvement}.
The heatmaps demonstrate performance gain when comparing our transferred policy against the source policy on the target task.
We calculate the improvement in the average episode rewards for every pair of source-target tasks sampled from the training task set.
The tasks in the HalfCheetah environment are $\smash{\mathcal{T}^{(1)} \cdots \mathcal{T}^{(5)}}$ with different armature values, namely \{0.1, 0.2, 0.3, 0.4, 0.5\}.
The tasks in the Ant environment are $\smash{\mathcal{T}^{(1)} \cdots \mathcal{T}^{(4)}}$ with different leg crippled, namely \{0, 1, 2, 3\}.
As mentioned above, we apply the TD3 algorithm\citep{fujimoto2018addressing} and dense rewards to learn source policies and mediocre policies for each task in training set.
Then we collect 100K transition data on each training tasks with the corresponding mediocre policies.

The architecture of context model $C$ and the forward model $F$ remains the same as above, while the learning rate is kept as 5e-4 instead.
The architecture of action translator $H$ is expanded to condition on the source task embeddings and target task embeddings. 
As mentioned in Sec.~\ref{sec:context}, in order to get the representative task feature for any arbitrary training task, we sample 1024 historical transition samples on this task, calculate the their context embedding through context model $C$ and average the 1024 context embedding to get the task feature as an 10-dimensional context vector.
The source target feature and target task feature are then encoded as 128-dimensional vectors respectively via a linear layer with ReLU activation.
Then the state embedding, action embedding, source task embedding and target task embedding are concatenated to produce the translated action via an MLP with 3 linear layers of 256 hidden units and ReLU activation.
The learning rate and batch size for $H$ are 3e-4 and 1024.
With these hyper-parameters settings, we train the action translator with 100 epochs, each with 1,000 steps.
We report the percentage gain comparing well-trained transferred policies with source policies on each pair of source-target tasks.

\subsection{Policy transfer on tasks sharing a general reward function, differing in dynamics}
\label{app:trans_r}
As explained in Sec.~\ref{sec:theory}, many real-world sparse-reward tasks fall under the umbrella of Proposition~\ref{cor:bound_policy}.
Thus, we are mainly interested in policy transfer across tasks with the same reward function $r(s,s')$ but different dynamics.
To solve policy transfer across these tasks, our objective function $\mathcal{L}_{trans}$ can be applied so that the transferred policy achieves a value on the target task similar to the source policy on the source task. 
Experiments in Sec.~\ref{sec:experiment} validate the efficacy of $\mathcal{L}_{trans}$ for learning policy transfer.

As for a more general case, we further consider tasks with different dynamics that have \textbf{the same state space, action space and reward function, where the general reward function $r(s,a,s')$ cannot be expressed as $r(s,s')$}.
Theorem 1 in Appendix~\ref{app:proof} covers this scenario.
For source task $\smash{\mathcal{T}^{(j)}=\{ \mathcal{S}, \mathcal{A}, p^{(j)}, r, \gamma, \rho_0\}}$ and target task $\smash{\mathcal{T}^{(i)}=\{ \mathcal{S}, \mathcal{A}, p^{(i)}, r, \gamma, \rho_0\}}$, we can bound the value difference between source policy $\smash{\pi^{(j)}}$ and transferred policy $\smash{\pi^{(i)}}$ by minimizing both reward difference $\smash{|r(s, \pi^{(i)}(s))-r(s,\pi^{(j)}(s))|}$ and total-variation difference in next state distribution $\smash{D_{TV}(p^{(i)}(\cdot|s, \pi^{(i)}(s)),p^{(j)}(\cdot|s, \pi^{(j)}(s))}$.
Accordingly, we modify transfer loss $\mathcal{L}_{trans}$ (Equation~\ref{eq:trans}) with an additional term of reward difference. 

Formally, $\smash{\mathcal{L}_{trans, r}=|r^{(j)}_t - R(s^{(j)}_t, \tilde{a}^{(i)}, z^{(i)})|- \lambda \log F(s^{(j)}_{t+1}|s^{(j)}_t, \tilde{a}^{(i)}, z^{(i)})}$, where $R$ is a learned reward prediction model,  $\lambda$ is a hyper-parameter weight of next state distribution loss, and $\smash{\tilde{a}^{(i)} = H(s^{(j)}_t, a^{(j)}_t, z^{(j)}, z^{(i)})}$ is the translated action.
This objective function drives the action translator H to find an action on the target task leading to a reward and next state, similarly to the source action on the source task.

As explained in Appendix C.1.1, MuJoco environments award the agent considering its velocity of moving forward $v_{torso}$ and the control cost $||a||^2$, i.e. $r = v_{torso} - c||a||^2$.
If the coefficient $c=0$, we can simplify this reward function as $r(s, s')$ because $v_{torso}$ is calculated only based on the current state $s$ and next state $s'$.
If $c>0$, $r$ becomes a general reward function $r(s, a, s')$.
We evaluate our action translator trained with $\mathcal{L}_{trans}$ and $\mathcal{L}_{trans, r}$ for this general case of reward function.
We search the hyper-parameter value of $\lambda$ in $\mathcal{L}_{trans, r}$ and $\lambda=10$ performs well across settings.

\begin{table}[h!]
\centering
\setlength{\tabcolsep}{2pt}
\begin{tabular}{c|c|cccc}
\toprule
\makecell{Control cost \\ coefficient} & \begin{tabular}[c]{@{}c@{}}Source policy \\ on source task\end{tabular} & \begin{tabular}[c]{@{}c@{}}Source policy\\ on target task\end{tabular} & \begin{tabular}[c]{@{}c@{}}Transferred policy\\ \citep{zhang2020learning} \\ on target task \end{tabular} & \begin{tabular}[c]{@{}c@{}}Transferred policy\\ (ours with $\mathcal{L}_{trans}$) \\on target task\end{tabular} & \begin{tabular}[c]{@{}l@{}}Transferred policy\\ (ours with $\mathcal{L}_{trans,r}$) \\ on target task\end{tabular} \\ \midrule
c=0.001  & 511.1  & 54.7   & 133.27  &  193.7 & \textbf{203.1}  \\
c=0.002  & 488.4  & 53.7   & 129.86  &  179.3 & \textbf{195.4}  \\
c=0.005  & 475.8 & 38.9 & 112.36 & 148.5 & \textbf{171.8} \\ \bottomrule
\end{tabular}
\caption{Average episode rewards on Ant environments. We consider the settings with different coefficients for control cost.}
\end{table}

Our action translator with either $\mathcal{L}_{trans}$ or $\mathcal{L}_{trans, r}$ performs well for policy transfer.
When the rewards depend on the action more heavily (i.e. $c$ becomes larger), the advantage of $\mathcal{L}_{trans, r}$ becomes more obvious.
However, ours with $\mathcal{L}_{trans, r}$ requires the extra complexity of learning a reward prediction model $R$.
When the reward function is mostly determined by the states and can be approximately simplified as $r(s,s')$, we recommend $\mathcal{L}_{trans}$ because it is simpler and achieves a competitive performance.

On Hopper and HalfCheetah, the control cost coefficient is $c>0$ by default. Our proposed policy transfer and MCAT achieve performance superior to the baselines on these environments (Sec.~\ref{sec:experiment}). This verifies the merits of our objective function $\mathcal{L}_{trans}$ on tasks with a general reward function $r(s, a, s')$.

\subsection{Implementation Details for Comparison with Context-based Meta-RL Algorithms}
\label{app:mcat_details}
\subsubsection{Environment}
We modify the physics parameters in the environments to get multiple tasks with varying dynamics functions. We delay the environment rewards to make sparse-reward tasks so that the baseline methods may struggle in these environments. The episode length is set as 1000 steps. The details of the training task set and test task set are shown in Table~\ref{tab:env_physics_modification}.
\begin{table}[!ht]
\centering
\setlength{\tabcolsep}{2pt}
\begin{tabular}{ccccc}
\toprule
Environment                  & Reward Delay Steps        & Physics Parameter & Training Tasks                            & Test Tasks                 \\
\midrule
Hopper                       & 100                  & Size              & \{0.02, 0.03, 0.04, 0.05, 0.06\} & \{0.01, 0.07\}         \\ \hline
\multirow{2}{*}{HalfCheetah} & \multirow{2}{*}{500} & Armature          & \{0.2, 0.3, 0.4, 0.5, 0.6\}      & \{0.05,0.1,0.7,0.75\}  \\
                             &                      & Mass              & \{0.5, 1.0, 1.5, 2.0, 2.5\}      & \{0.2, 0.3, 2.7, 2.8\} \\ \hline
\multirow{2}{*}{Ant}   &       \multirow{2}{*}{500}  & Damping           & \{1.0, 10.0, 20.0, 30.0\}      & \{0.5,35.0\} \\
                             &                      & Crippled Leg      & \{ No crippled leg, crippled leg 0, 1, 2\}                      & \{crippled leg 3\}\\
                             \bottomrule

\end{tabular}
\caption{Modified physics parameters used in the experiments.}
\label{tab:env_physics_modification}
\end{table}

\subsubsection{Implementation Details}
In Section~\ref{sec:ours_combination}, we compare our proposed method with other context-based meta-RL algorithms on environments with sparse rewards.
Below we describe the implementation details of each method.

\paragraph{PEARL\citep{rakelly2019efficient}} We use the implementation provided by the authors\footnote{https://github.com/katerakelly/oyster}. The PEARL agent consists of the context encoder model and the policy model. Following the default setup, the context encoder model is an MLP encoder with 3 hidden layers of 200 units each and ReLU activation. We model the policy as Gaussian, where the mean and log variance is also parameterized by MLP with 3 hidden layers of 300 units and ReLU activation. Same to the default setting, the log variance is clamped to [-2, 20]. We mostly use the default hyper-parameters and search the dimension of the context vector in $\{5, 10, 20\}$. We report the performance of the best hyper-parameter, which achieves highest average score on training tasks. 

\paragraph{MQL\citep{fakoor2019meta}} We use the implementation provided by the authors\footnote{https://github.com/amazon-research/meta-q-learning}. The context encoder is a Gated Recurrent Unit model compressing the information in recent historical transitions. The actor network conditioning on the context features is an MLP with 2 hidden layers of 300 units each and a ReLU activation function. The critic network is of the same architecture as the actor network. We search the hyper-parameters: learning rate in $\{0.0003, 0.0005, 0.001\}$, history length in $\{10, 20\}$, GRU hidden units in $\{20, 30\}$, TD3 policy noise in $\{0.1, 0.2\}$, TD3 exploration noise in $\{0.1, 0.2\}$. We report the performance of the best set of hyper-parameters, which achieves highest average score on training tasks.

\paragraph{Distral\citep{teh2017distral}} We use the implementation in the MTRL repository\footnote{https://github.com/facebookresearch/mtrl}. The Distral framework consists of a central policy and several task-specific policies. The actor network of the central policy is an MLP with 3 hidden layers of 400 units each and a ReLU activation function. The actor and critic networks of the task-specific policies are of the same architecture as the actor network of the central policy. As for the hyperparameters, we set $\alpha$ to 0.5 and search $\beta$ in $\{1, 10, 100\}$, where $\smash{\frac{\alpha}{\beta}}$ controls the divergence between central policy and task-specific policies, and $\smash{\frac{1}{\beta}}$ controls the entropy of task-specific policies. When optimizing the actor and critic networks, the learning rates are 1e-3. We report the performance of the best hyper-parameter, which achieves highest average score on training tasks.

\paragraph{HiP-BMDP\citep{zhang2020robust}} We use the implementation in the MTRL repository (same as the Distral baseline above). The actor and critic networks are also the same as the ones in Distral above. When optimizing the actor and critic network, the learning rates for both of them are at 1e-3. The log variance of the policy is bound to [-20, 2]. We search the $\Theta$ learning error weight $\alpha_{\psi}$ in $\{0.01, 0.1, 1\}$, which scales their task bisimulation metric loss. We report the performance of the best hyper-parameter, which achieves highest average score on training tasks.

\paragraph{MCAT (Ours)} The architectures of the context model $C$, forward dynamics model $F$ and the action translator $H$ are the same as introduced in Appendix~\ref{app:fixed_details}.  
The actor network and critic network are both MLPs with 2 hidden layers of 256 units and ReLU activations. 
As described in Algorithm~\ref{alg:mcat}, at each iteration, we collect 5K transition data from training tasks.
Then we train the context model $C$ and forward dynamics model $F$ for 10K training steps.
We train the action translator $H$ for 1K training steps.
The actor and critic networks are updated for 5K training steps.
In order to monitor the performance of transferred and learned policy in recent episodes, we clear the information about episode reward in $R^{(i)}$ and $R^{(j)\rightarrow (i)}$ before the last $G=20000$ steps. 

The learning rate and batch size of training $C$, $F$ and $H$ are the same as introduced in ``Context-conditioned Action Translator" in Appendix~\ref{app:fixed_details}.
The hyper-parameters of learning the actor and critic are the same as listed in Table~\ref{tab:TD3_5_1}.
Besides, we adapt the official implementation\footnote{https://github.com/junhyukoh/self-imitation-learning} to maintain SIL replay buffer with their default hyper-parameters on MuJoCo environments. 

Even though there are a couple of components, they are trained alternatively not jointly.
The dynamics model is learned with $\mathcal{L}_{forw}$ to accurately predict the next state.
The learned context embeddings for different tasks can separate well due to the regularization term $\mathcal{L}_{const}$.
With the fixed context encoder and dynamics model, the action translator can be optimized.
Then, with the fixed context encoder, the context-conditioned policy learns good behavior from data collected by the transferred policy.
These components are not moving simultaneously and this fact facilitates the learning process.
To run our approach on MuJoCo environments, for each job, we need to use one GPU card (NVIDIA GeForce GTX TITAN X) for around 4 days.
Fig.~\ref{fig:app_ours_baseline} and Tab.~\ref{tab:app_ours_baseline} show the performance of our approach and baselines on various environments. 

\begin{table}[!ht]
\centering

\setlength{\tabcolsep}{2pt}
\begin{tabular}{c|ccccc}
\toprule
Setting & \makecell{Hopper \\ Size} & \makecell{HalfCheetah \\ Armature} &     \makecell{HalfCheetah \\ Mass} & \makecell{Ant \\ Damping} & 
\makecell{Ant \\ Cripple} \\ \midrule
MQL & \makecell{1607.5\\[-3pt]\tiny{($\pm$327.5)}} & \makecell{-77.9\\[-3pt]\tiny{($\pm$214.2)}} & \makecell{-413.9\\[-3pt]\tiny{($\pm$11.0)}} & \makecell{103.1\\[-3pt]\tiny{($\pm$35.7)}} & \makecell{38.2\\[-3pt]\tiny{($\pm$4.0)}}
 \\
PEARL & \makecell{1755.8\\[-3pt]\tiny{($\pm$115.3)}} & \makecell{-18.8\\[-3pt]\tiny{($\pm$69.3)}} & \makecell{25.9\\[-3pt]\tiny{($\pm$69.2)}} & \makecell{73.2\\[-3pt]\tiny{($\pm$13.3)}} & \makecell{3.5\\[-3pt]\tiny{($\pm$2.4)}} 
\\
Distral & \makecell{1319.8\\[-3pt]\tiny{($\pm$162.2)}} & \makecell{566.9\\[-3pt]\tiny{($\pm$246.7)}} & \makecell{-29.5\\[-3pt]\tiny{($\pm$3.0)}} & \makecell{90.5\\[-3pt]\tiny{($\pm$28.4)}} & \makecell{-0.1\\[-3pt]\tiny{($\pm$0.7)}}
\\
HiP-BMDP & \makecell{1368.3\\[-3pt]\tiny{($\pm$150.7)}} & \makecell{-102.4\\[-3pt]\tiny{($\pm$24.9)}} & \makecell{-74.8\\[-3pt]\tiny{($\pm$35.4)}} & \makecell{33.1\\[-3pt]\tiny{($\pm$6.0)}} & \makecell{7.3\\[-3pt]\tiny{($\pm$2.6)}} 
\\
MCAT(Ours) & \makecell{\textbf{1914.8}\\[-3pt]\tiny{($\pm$373.2)}} & \makecell{\textbf{2071.5}\\[-3pt]\tiny{($\pm$447.4)}} & \makecell{\textbf{1771.1}\\[-3pt]\tiny{($\pm$617.7)}} & \makecell{\textbf{624.6}\\[-3pt]\tiny{($\pm$218.8)}} & \makecell{\textbf{281.6}\\[-3pt]\tiny{($\pm$65.6)}} 
\\
\bottomrule
\end{tabular}

\caption{Test rewards at 2M timesteps, averaged over 3 runs. This corrects the last column of Tab.~\ref{tab:ours_baseline} in the main text.}
\label{tab:app_ours_baseline}
\end{table}

\begin{figure*}[!ht]
\centering
\begin{minipage}{0.18\textwidth}
    \centering
    \includegraphics[width=\linewidth]{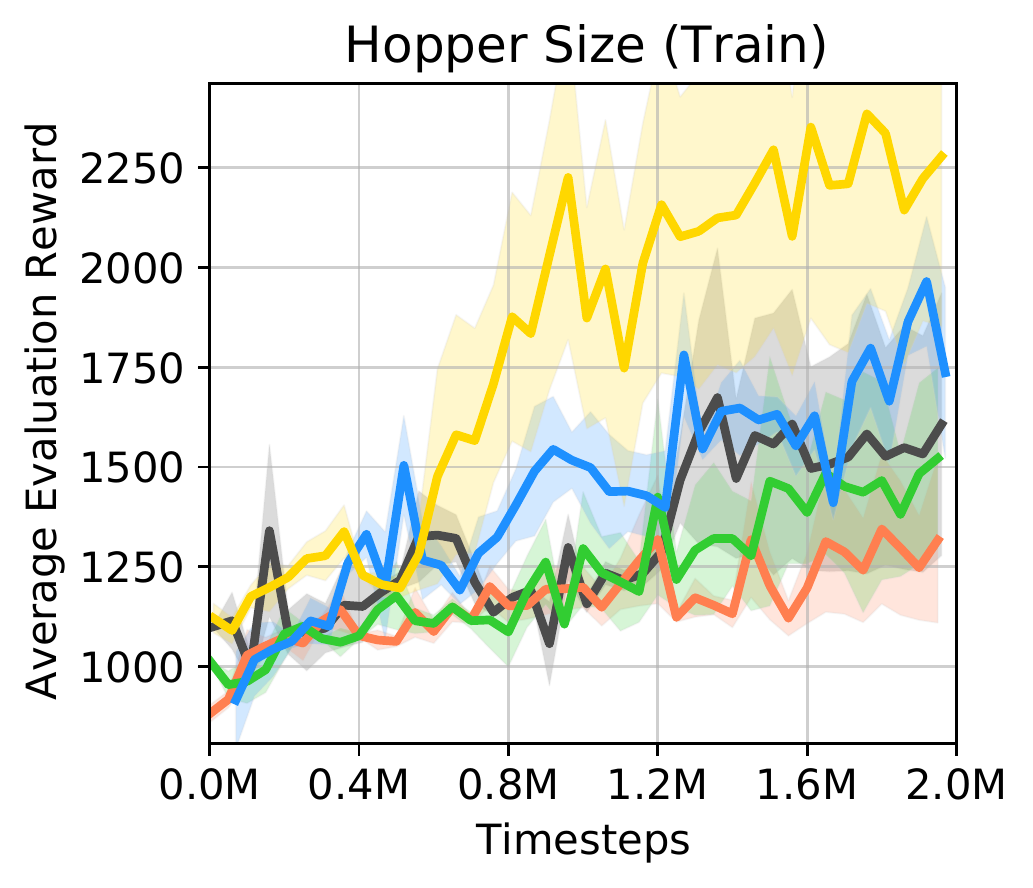}
\end{minipage}%
\hfill
\begin{minipage}{0.18\textwidth}
    \includegraphics[width=\linewidth]{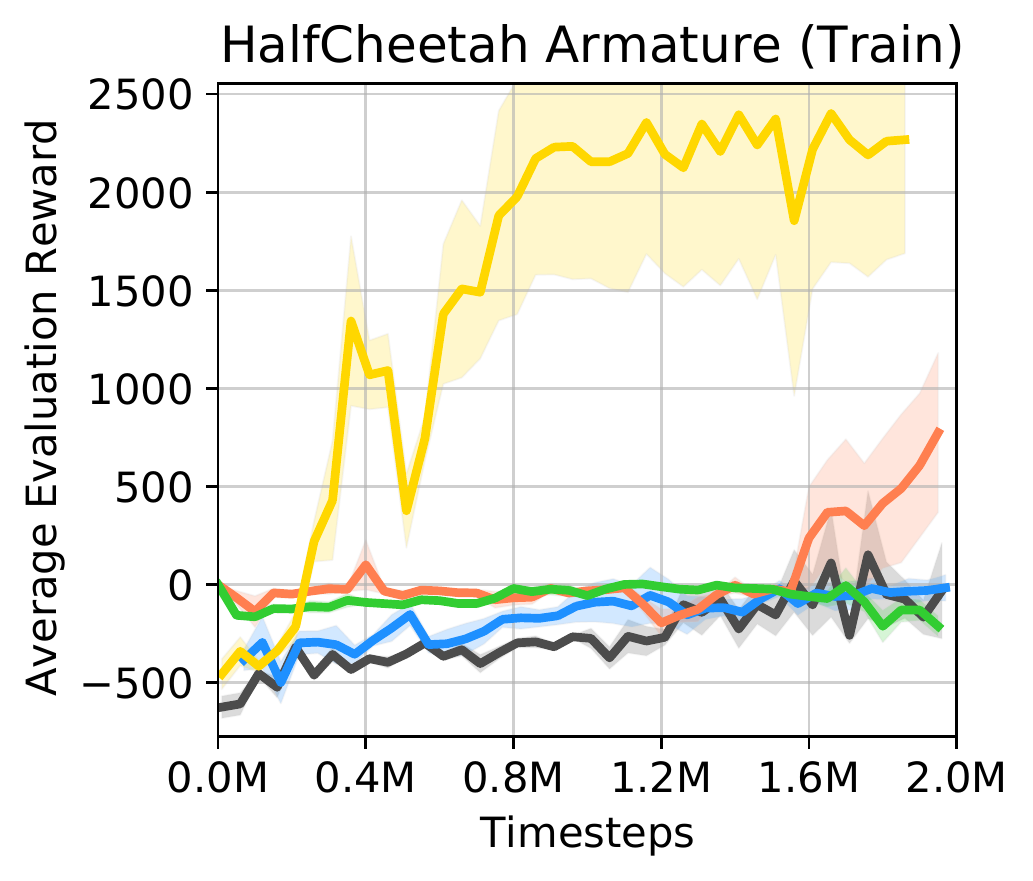}
\end{minipage}%
\hfill
\begin{minipage}{0.18\textwidth}
    \centering
    \includegraphics[width=\linewidth]{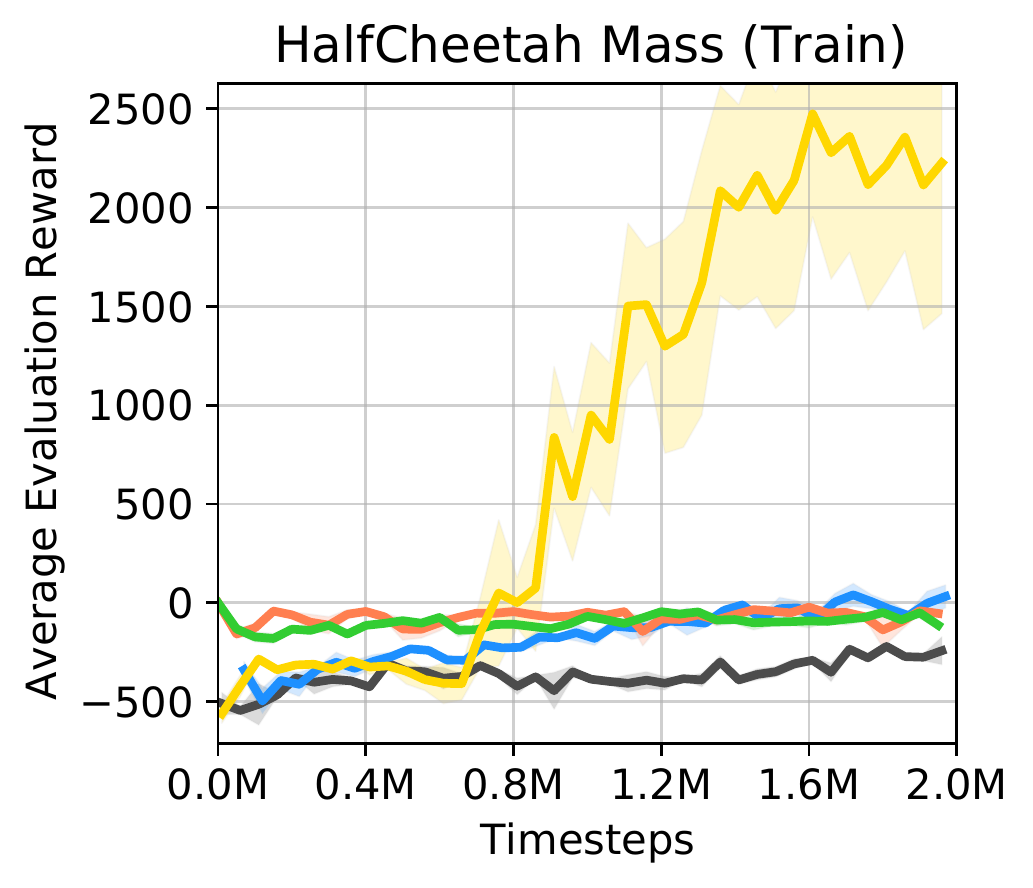}
\end{minipage}
\hfill
\begin{minipage}{0.18\textwidth}
    \centering
    \includegraphics[width=\linewidth]{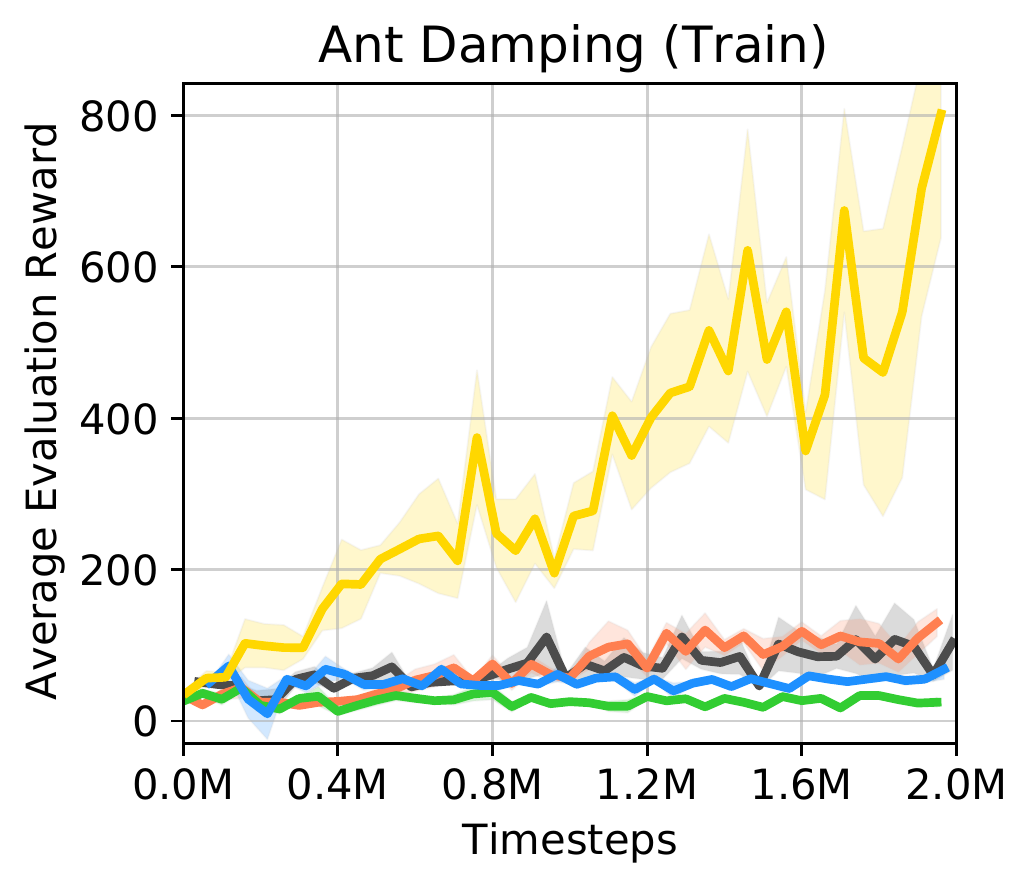}
\end{minipage}
\hfill
\begin{minipage}{0.18\textwidth}
    \centering
    \includegraphics[width=\linewidth]{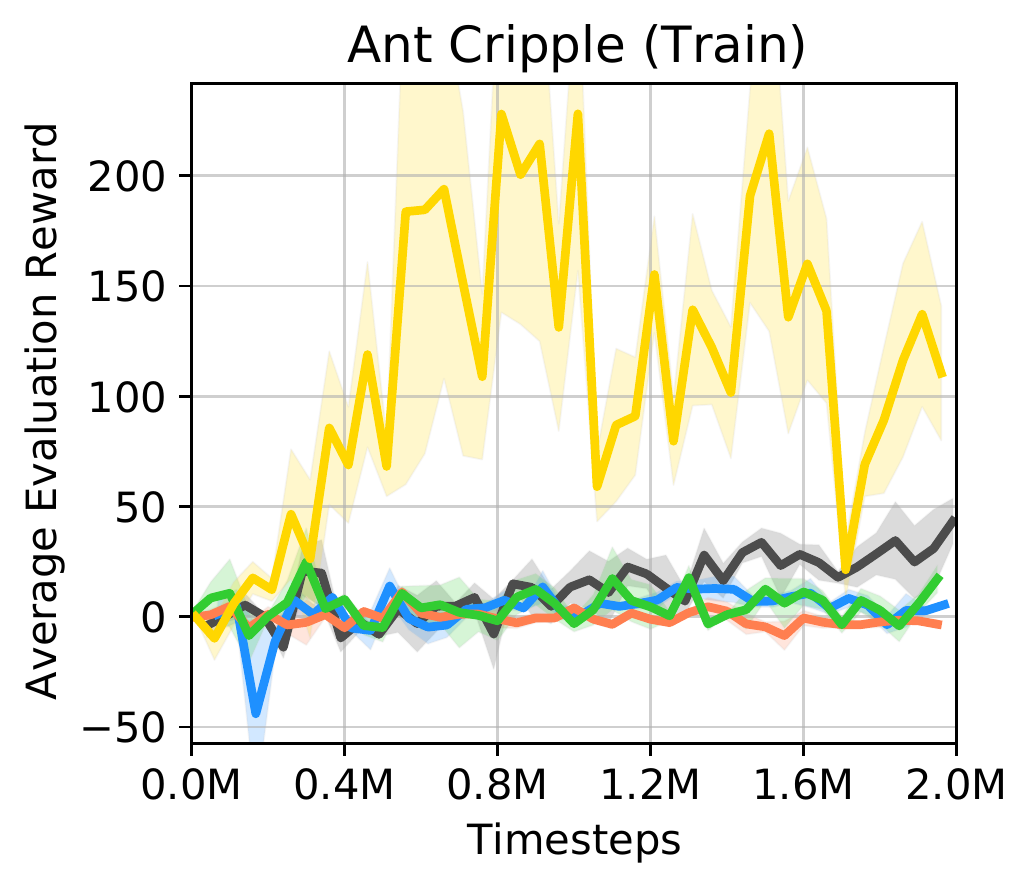}
\end{minipage}
\vfill
\begin{minipage}{0.18\textwidth}
    \centering
    \includegraphics[width=\linewidth]{figures/Hopper_size_test_4_1.pdf}
\end{minipage}%
\hfill
\begin{minipage}{0.18\textwidth}
    \includegraphics[width=\linewidth]{figures/Halfcheetah_arma_test_4_1.pdf}
\end{minipage}%
\hfill
\begin{minipage}{0.18\textwidth}
    \centering
    \includegraphics[width=\linewidth]{figures/Halfcheetah_mass_test_4_1.pdf}
\end{minipage}
\hfill
\begin{minipage}{0.18\textwidth}
    \centering
    \includegraphics[width=\linewidth]{figures/Ant_damping_test_4_1.pdf}
\end{minipage}
\hfill
\begin{minipage}{0.18\textwidth}
    \centering
    \includegraphics[width=\linewidth]{figures/Ant_cripple_test_4_1.pdf}
\end{minipage}

    \caption{Learning curves of episode rewards on both training and test tasks, averaged over 3 runs. 
    Shadow areas indicate standard error. This adds the performance on training tasks in comparison to Fig.~\ref{fig:ours_baseline} in the main text and changes the learning curves on test tasks in Ant Cripple setting to 2M timesteps}
    \label{fig:app_ours_baseline}
\end{figure*}

Furthermore, we present additional experimental results on MetaWorld environment.
In Section~\ref{sec:fixed}, we introduced the tasks of moving objects to target locations and the reward is positive only when the object is close to the goal. We combine context-based TD3 with policy transfer to learn a policy operating multiple objects: drawer, coffee mug, soccer, cube, plate.
Then we test whether the policy could generalize to moving a large cylinder.
In Tab.~\ref{tab:app_metaworld_mcat}, MCAT agent earns higher success rate than the baselines on both training and test tasks after 2M timesteps in the sparse-reward tasks.

\begin{table}[!ht]
\centering
\begin{tabular}{c|cccc}
\toprule
         & \makecell{MQL \\ \citep{fakoor2019meta}} &
         \makecell{PEARL \\ \citep{rakelly2019efficient}} &
         \makecell{PCGrad \\ \citep{yu2020gradient}} & MCAT  \\ \midrule
Training tasks (reward) & 164.8\tiny{($\pm 23.6$)}& 161.2\tiny{($\pm 25.3$)}  & 44.8\tiny{($\pm 31.7$)} & \textbf{204.1}\tiny{($\pm 43.1$)} \\
Test tasks (reward)  & 0.0\tiny{($\pm 0.0$)}& 0.0\tiny{($\pm 0.0$)} & 0.0\tiny{($\pm 0.0$)}  & \textbf{10.2}\tiny{($\pm 8.3$)} \\
Training tasks (success rate) & 40.0\%\tiny{($\pm 0.0\%$)}&
33.3\%\tiny{($\pm 5.4\%$)}& 10.0\%\tiny{($\pm 7.1\%$)} & \textbf{53.3\%}\tiny{($\pm 5.4\%$)} \\
Test tasks (success rate) & 0.0\%\tiny{($\pm 0.0\%$)} & 0.0\%\tiny{($\pm 0.0\%$)} & 0.0\%\tiny{($\pm 0.0\%$)} & \textbf{16.7\%} \tiny{($\pm 13.6\%$)} \\ \bottomrule
\end{tabular}
\caption{Performance of learned policies at 2M timesteps, averaged over 3 runs. }
\label{tab:app_metaworld_mcat}
\end{table}

\clearpage
\section{Ablative Study}
\subsection{Effect of Policy Transfer}
In Section~\ref{sec:ablation}, we investigate the effect of policy transfer (PT). In Figure~\ref{fig:effect_pt} we provide the learning curves of MCAT and MCAT w/o PT on both training tasks and test tasks. 
\begin{figure*}[!ht]
\begin{minipage}{0.19\textwidth}
    \includegraphics[width=\linewidth]{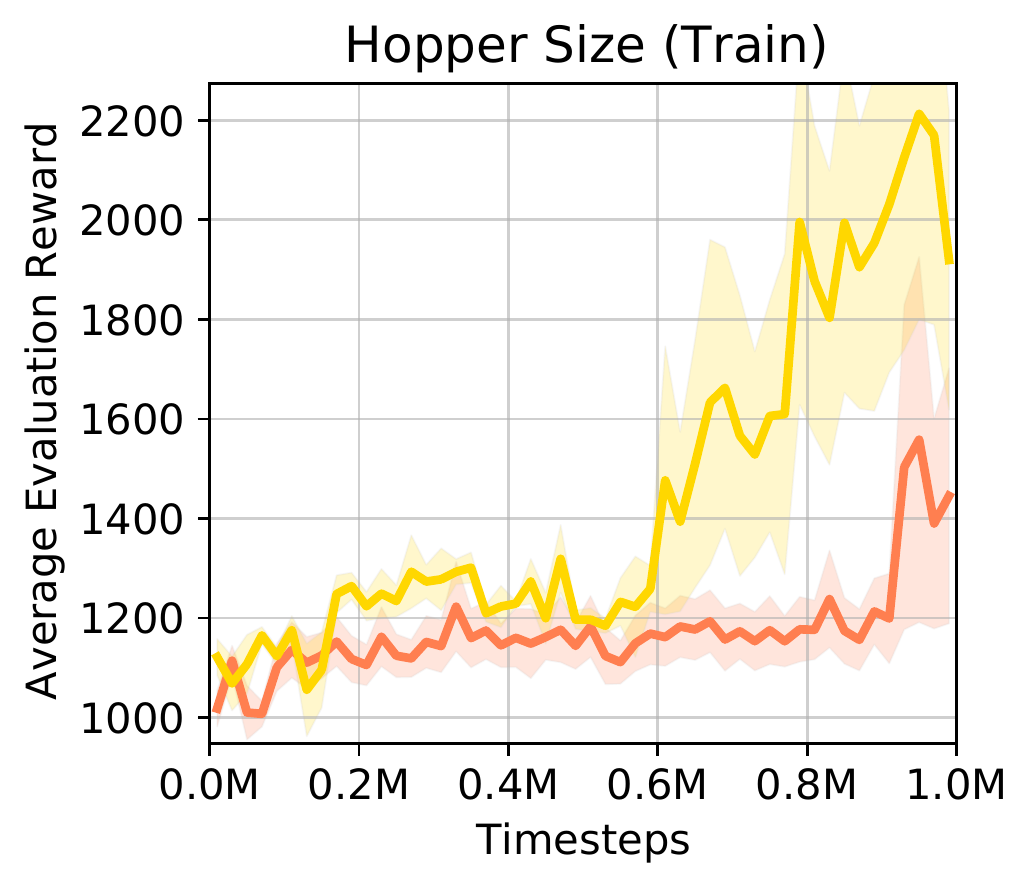}
\end{minipage}%
\hfill
\begin{minipage}{0.19\textwidth}
    \includegraphics[width=\linewidth]{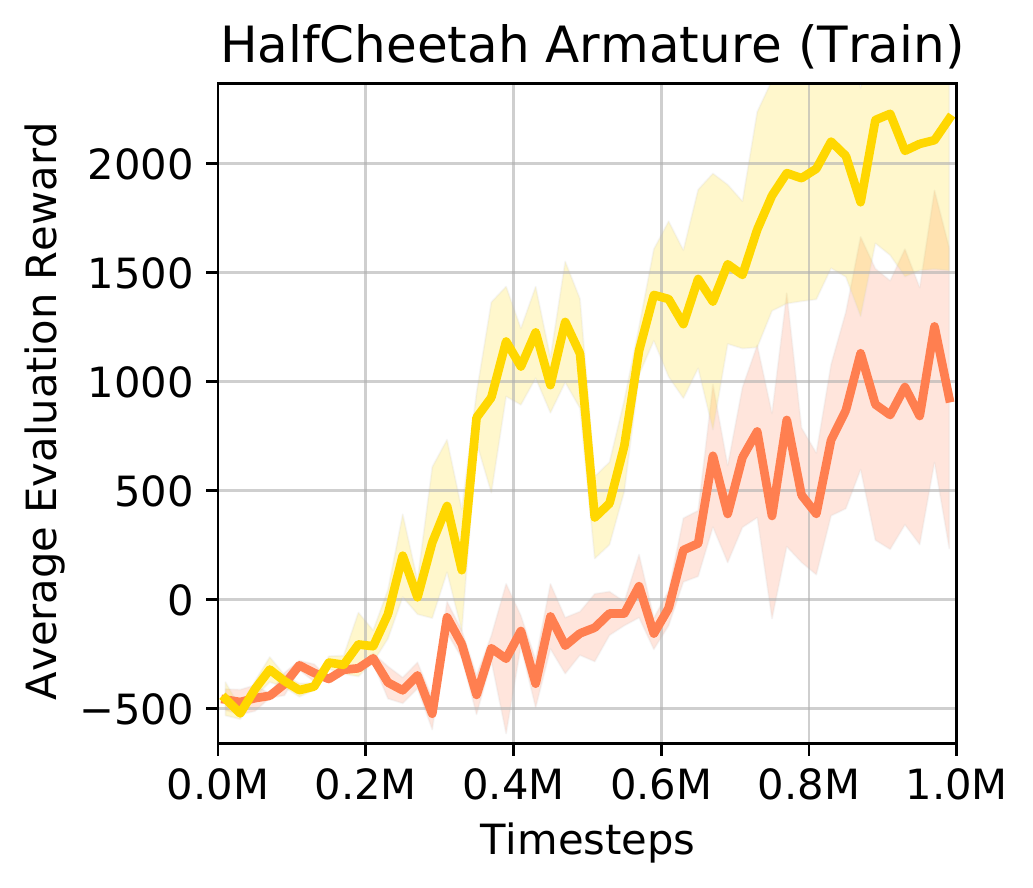}
\end{minipage}%
\hfill
\begin{minipage}{0.19\textwidth}
    \centering
    \includegraphics[width=\linewidth]{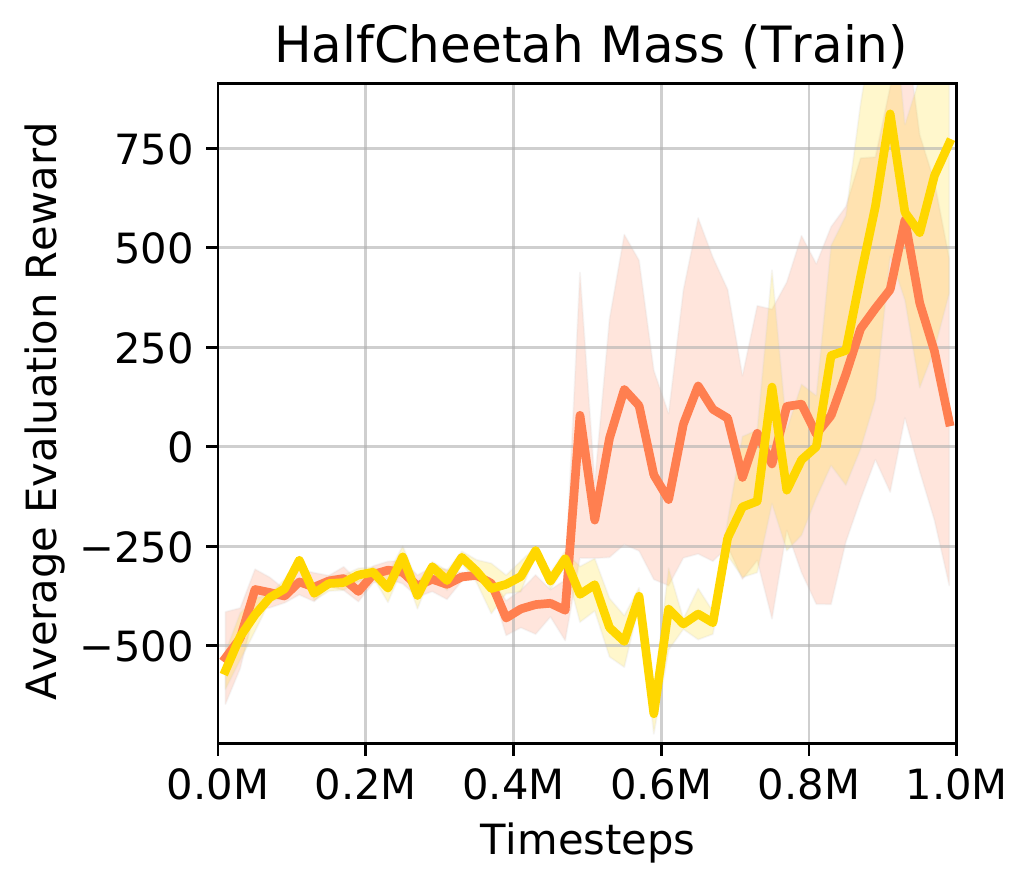}
\end{minipage}
\hfill
\begin{minipage}{0.19\textwidth}
    \centering
    \includegraphics[width=\linewidth]{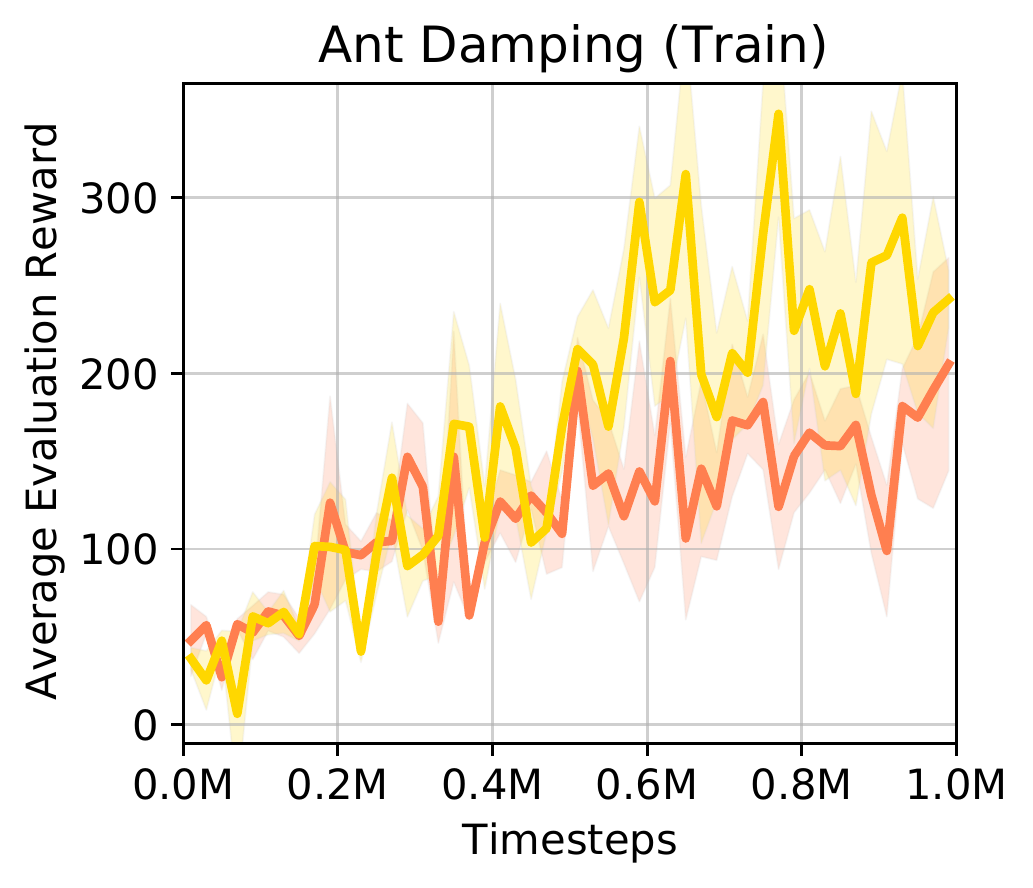}
\end{minipage}
\hfill
\begin{minipage}{0.19\textwidth}
    \centering
    \includegraphics[width=\linewidth]{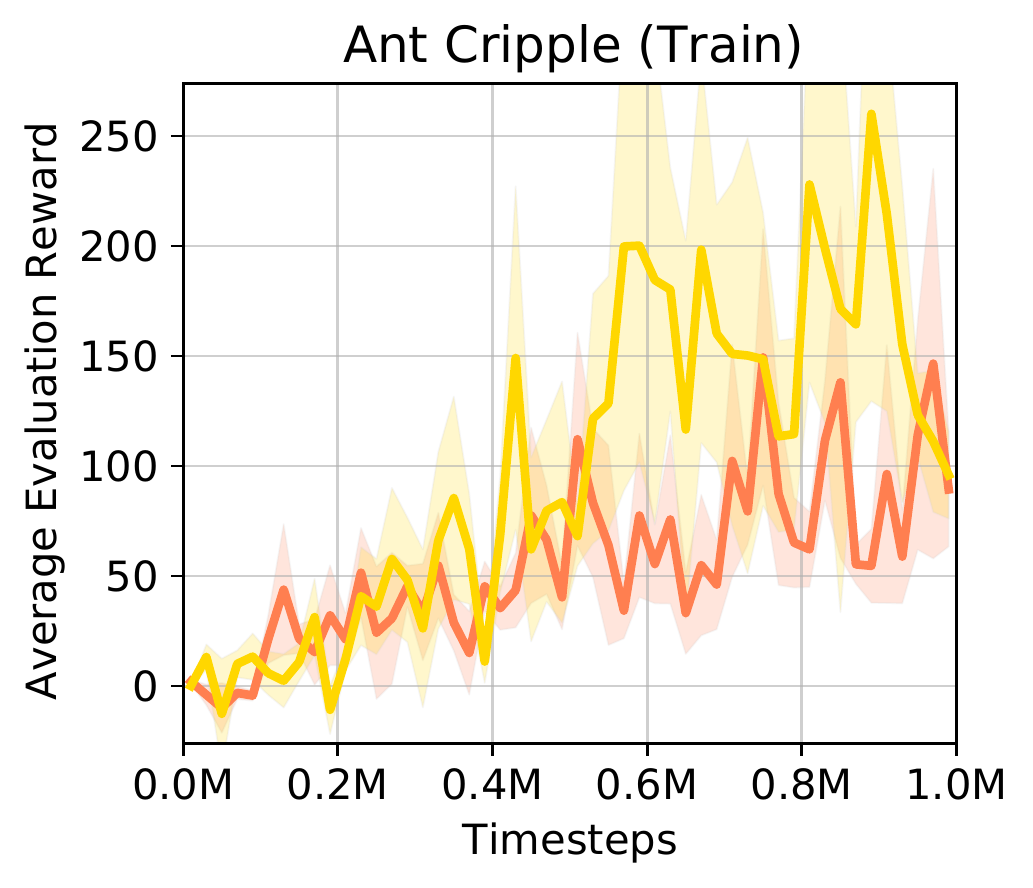}
\end{minipage}
\hfill
\begin{minipage}{0.19\textwidth}
    \centering
    \includegraphics[width=\linewidth]{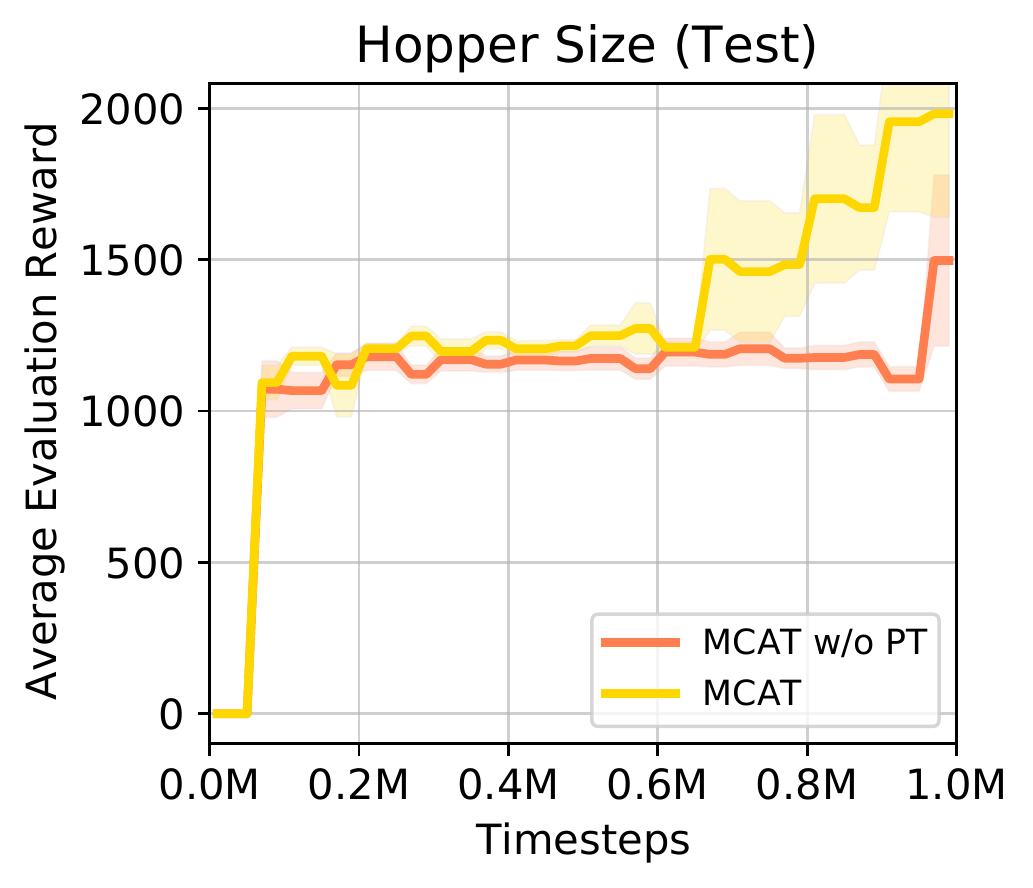}
\end{minipage}%
\hfill
\begin{minipage}{0.19\textwidth}
    \includegraphics[width=\linewidth]{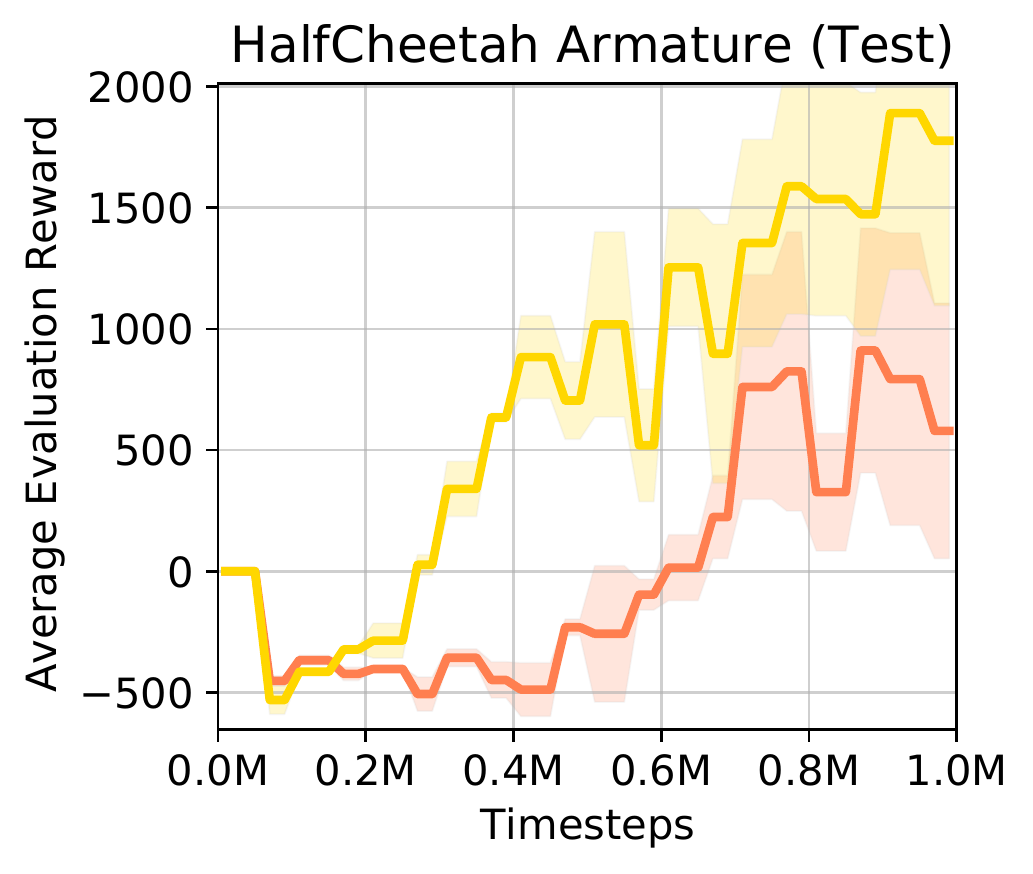}
\end{minipage}%
\hfill
\begin{minipage}{0.19\textwidth}
    \centering
    \includegraphics[width=\linewidth]{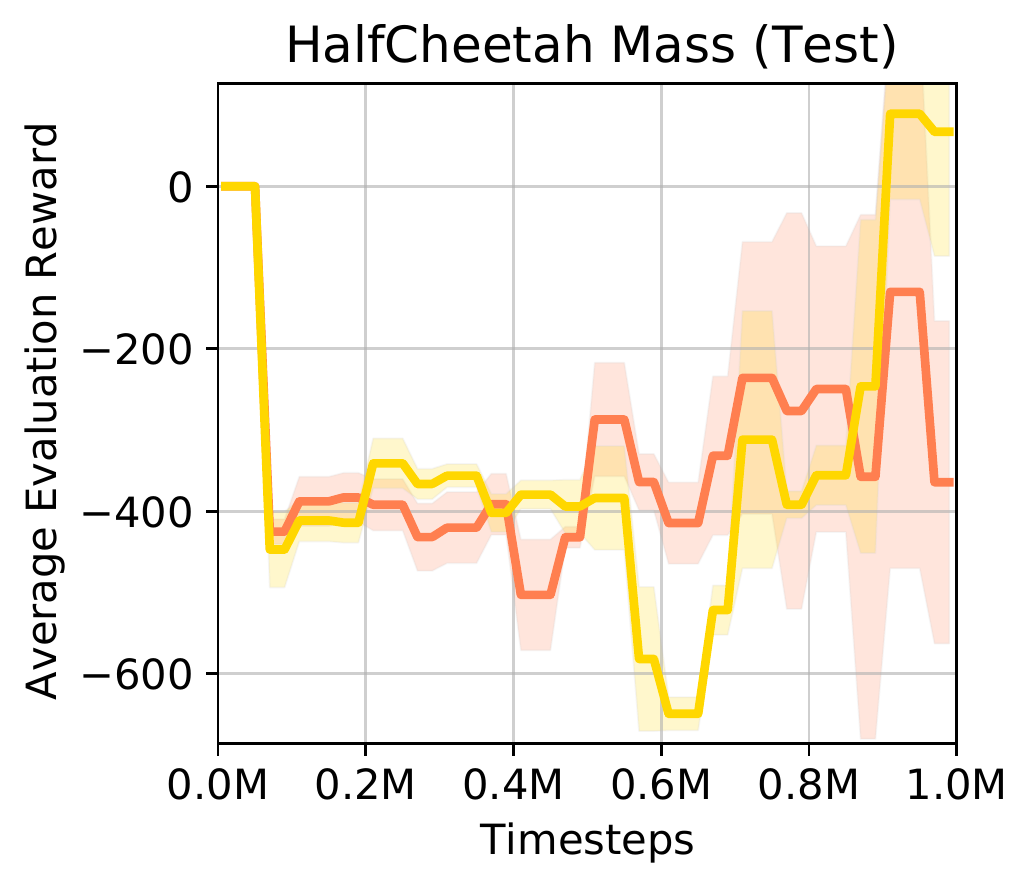}
\end{minipage}
\hfill
\begin{minipage}{0.19\textwidth}
    \centering
    \includegraphics[width=\linewidth]{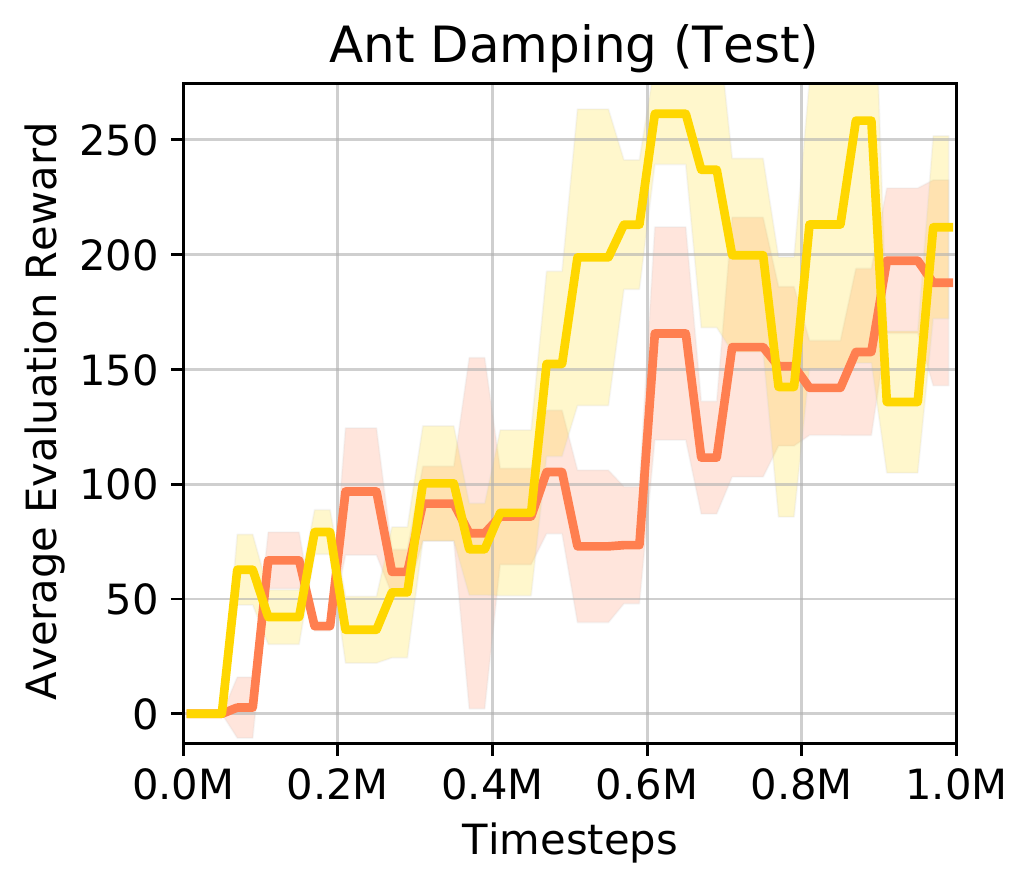}
\end{minipage}
\hfill
\begin{minipage}{0.19\textwidth}
    \centering
    \includegraphics[width=\linewidth]{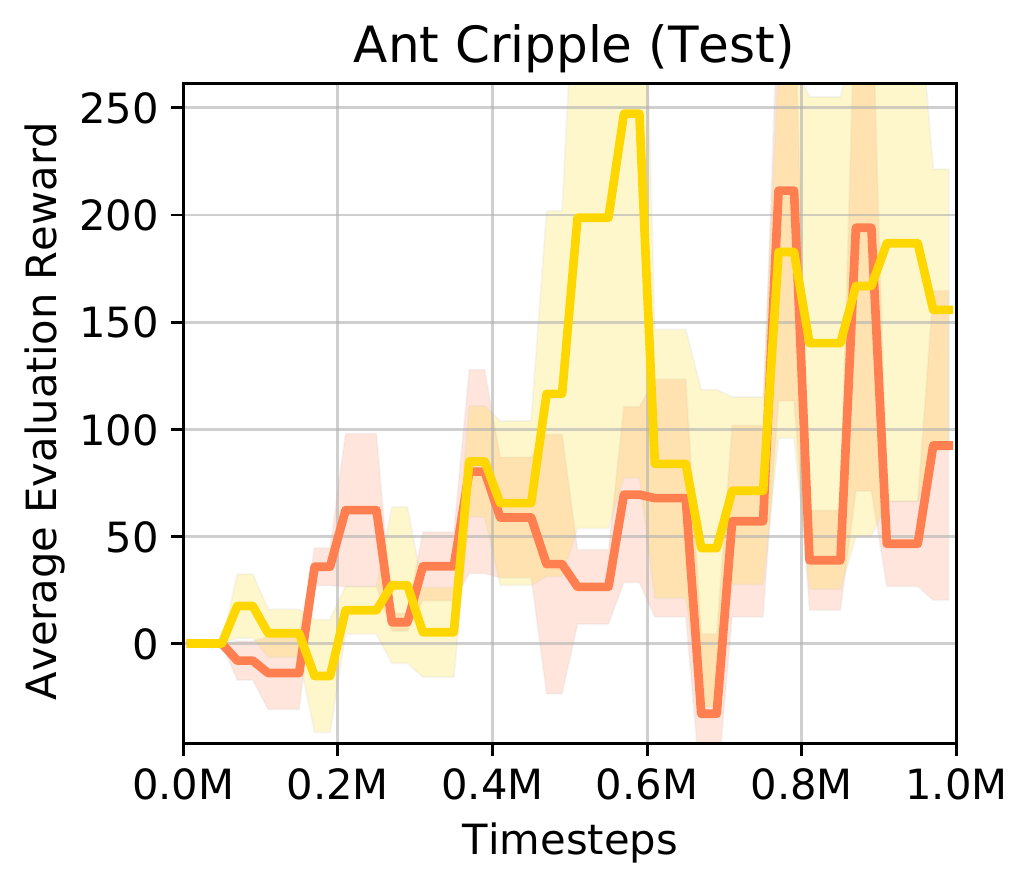}
\end{minipage}
\hfill

    \vspace*{-8pt}
    \caption{Learning curves of the average episode reward, averaged over 3 runs. The average episode reward and standard error are reported on training tasks and test tasks respectively. This repeats Figure~\ref{fig:effect_pt} with addition of learning curves on training tasks.
}
    \label{fig:effect_pt}
    \vspace*{-0.2in}
\end{figure*}

\subsection{More Sparse Rewards}
In Section~\ref{sec:ablation}, we report the effect of policy transfer when the rewards become more sparse in the environments. On HalfCheetah, we delay the environment rewards for different number of steps 200, 350, 500. In Figure~\ref{fig:effect_sparse}, we show the learning curves on training and test tasks. In Table~\ref{tab:more_sparse}, we report the average episode rewards and standard error over 3 runs at 1M timesteps.
\begin{figure*}[!ht]
\begin{minipage}{0.16\textwidth}
    \includegraphics[width=\linewidth]{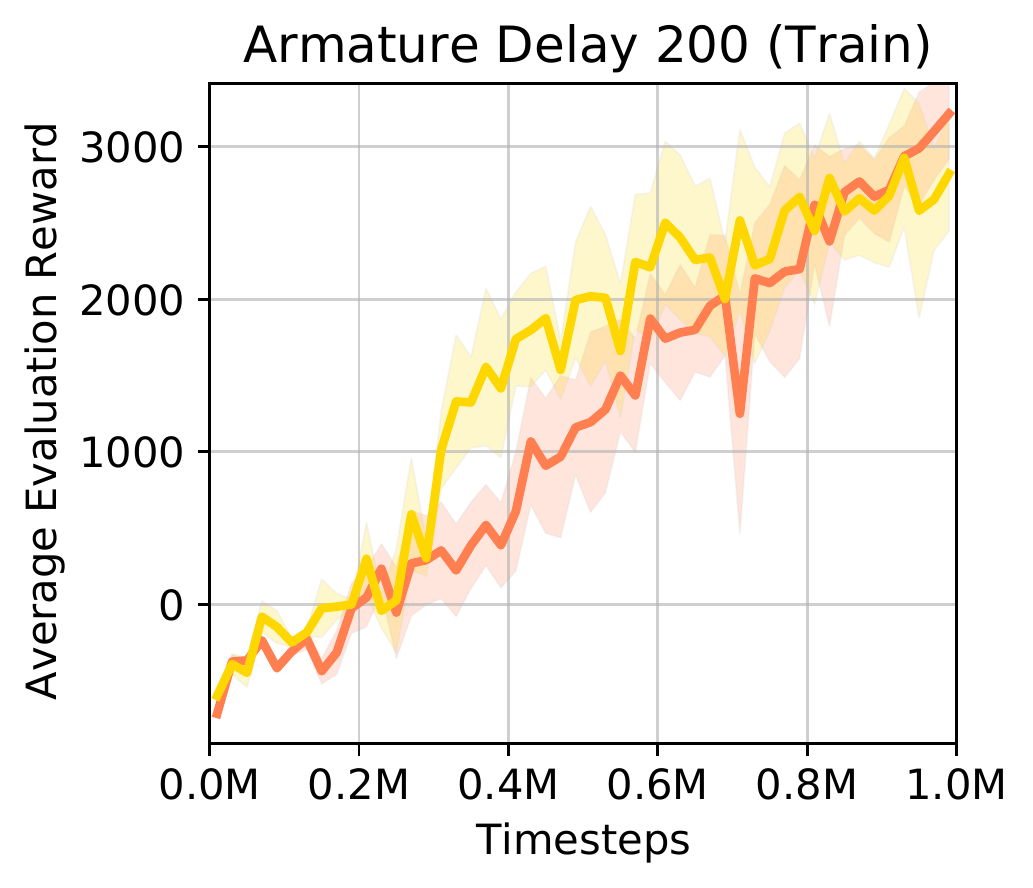}
\end{minipage}%
\hfill
\begin{minipage}{0.16\textwidth}
    \includegraphics[width=\linewidth]{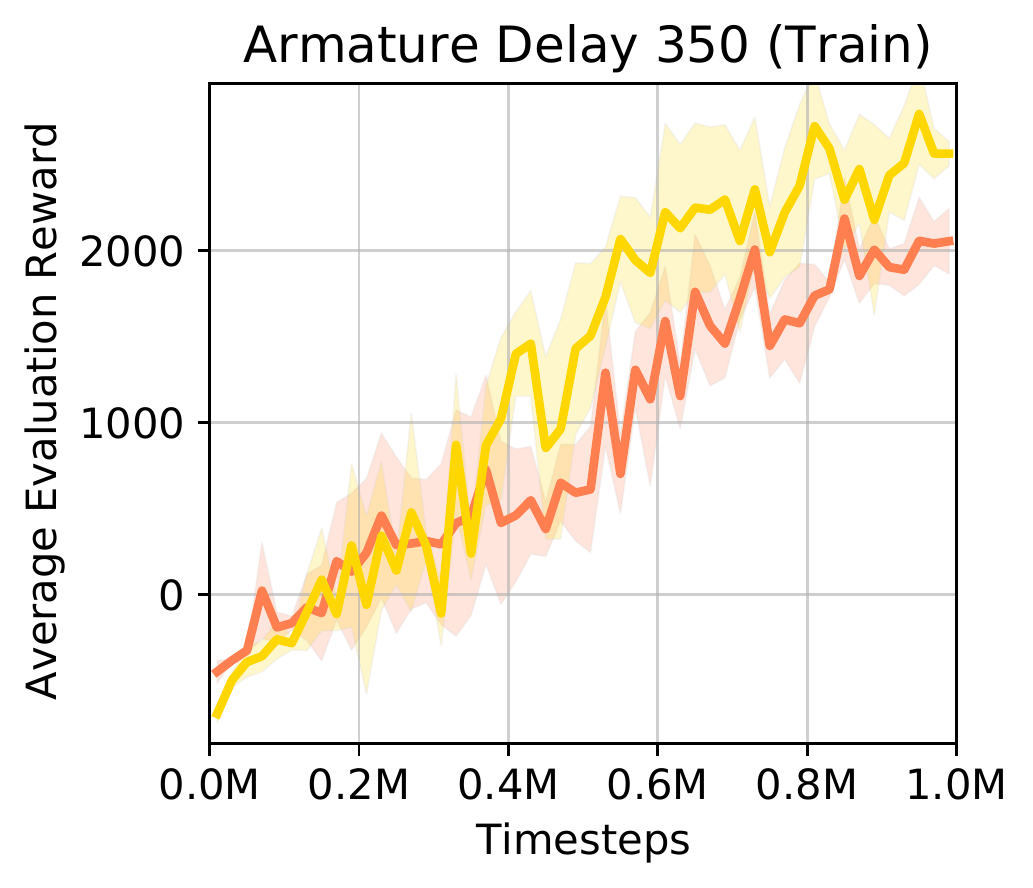}
\end{minipage}%
\hfill
\begin{minipage}{0.16\textwidth}
    \centering
    \includegraphics[width=\linewidth]{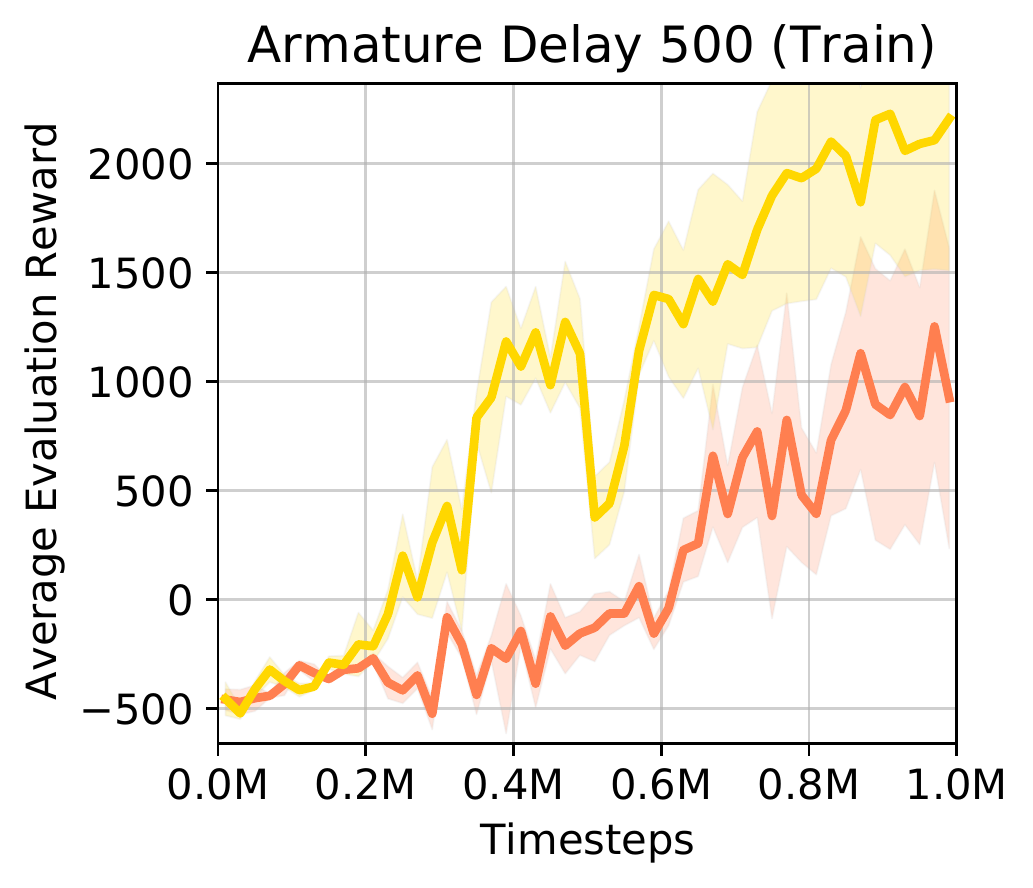}
\end{minipage}
\hfill
\begin{minipage}{0.16\textwidth}
    \centering
    \includegraphics[width=\linewidth]{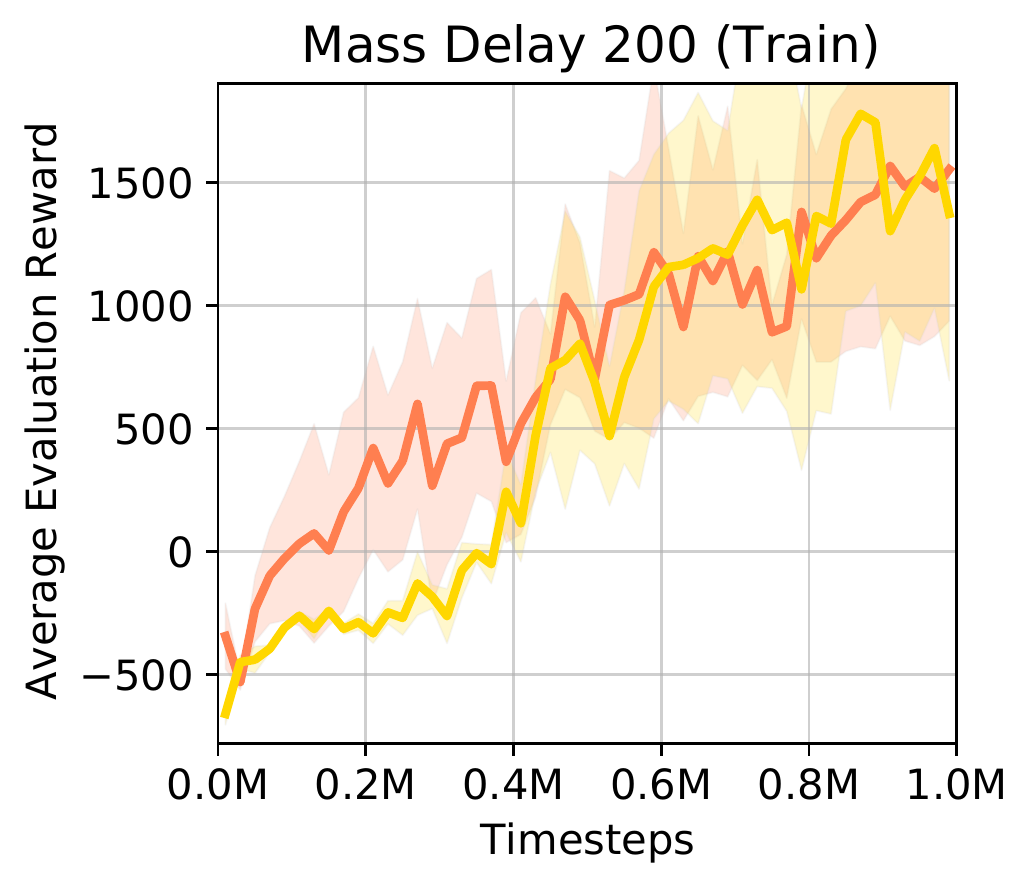}
\end{minipage}
\hfill
\begin{minipage}{0.16\textwidth}
    \centering
    \includegraphics[width=\linewidth]{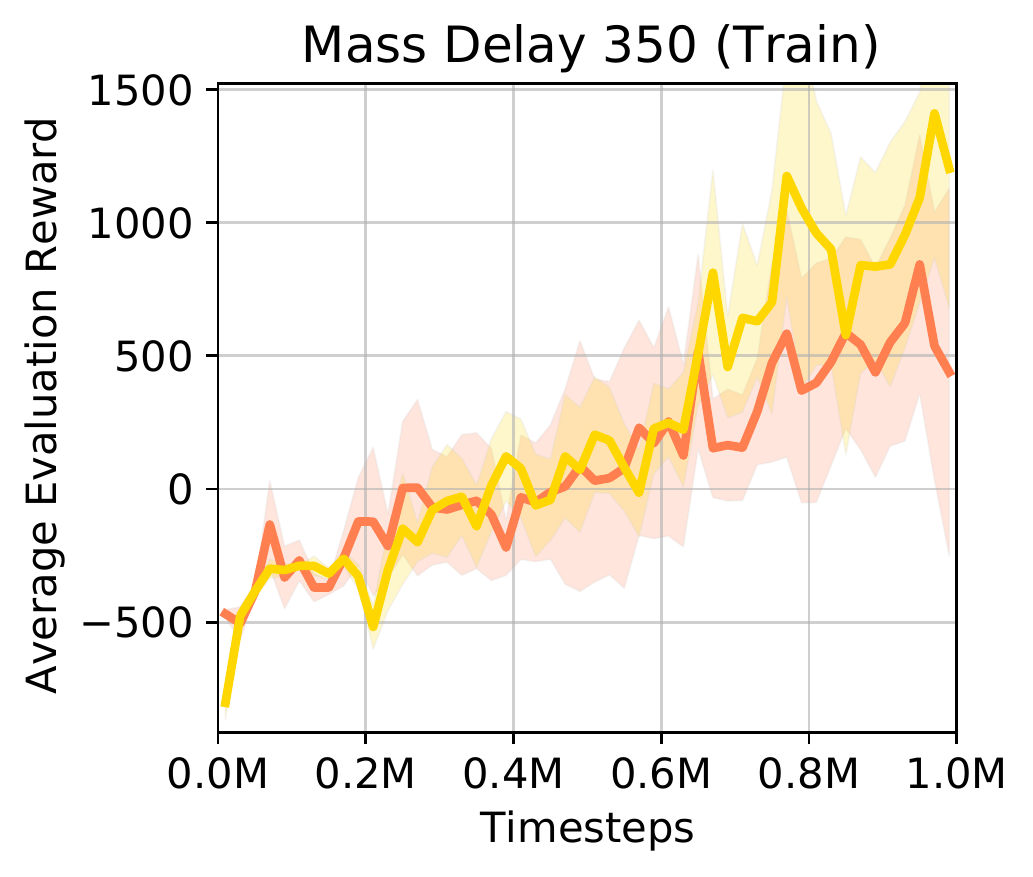}
\end{minipage}
\hfill
\begin{minipage}{0.16\textwidth}
    \centering
    \includegraphics[width=\linewidth]{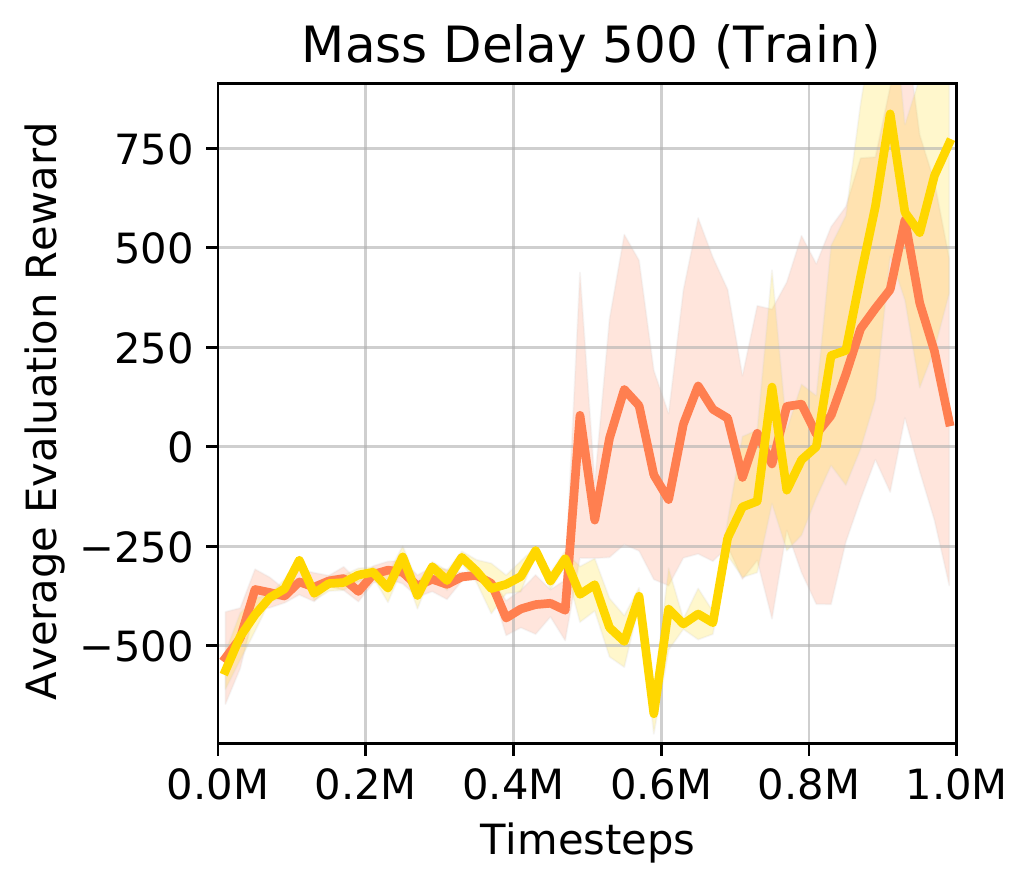}
\end{minipage}
\hfill
\begin{minipage}{0.16\textwidth}
    \centering
    \includegraphics[width=\linewidth]{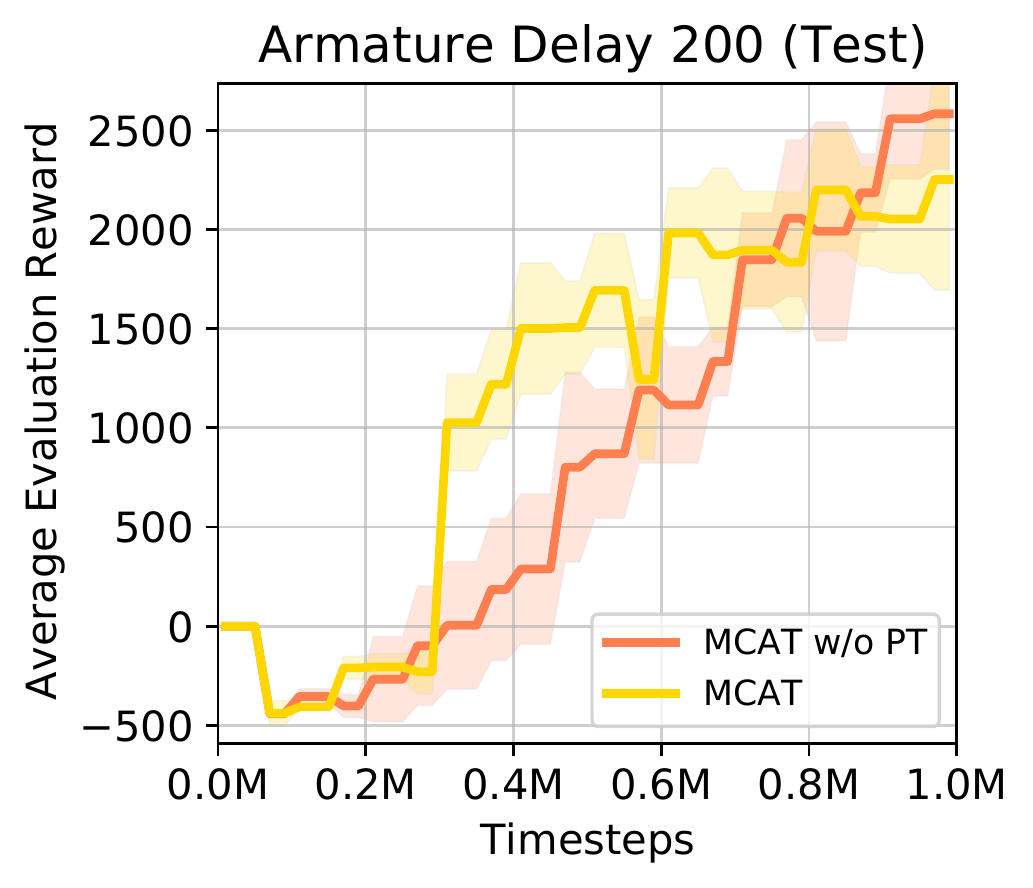}
\end{minipage}%
\hfill
\begin{minipage}{0.16\textwidth}
    \includegraphics[width=\linewidth]{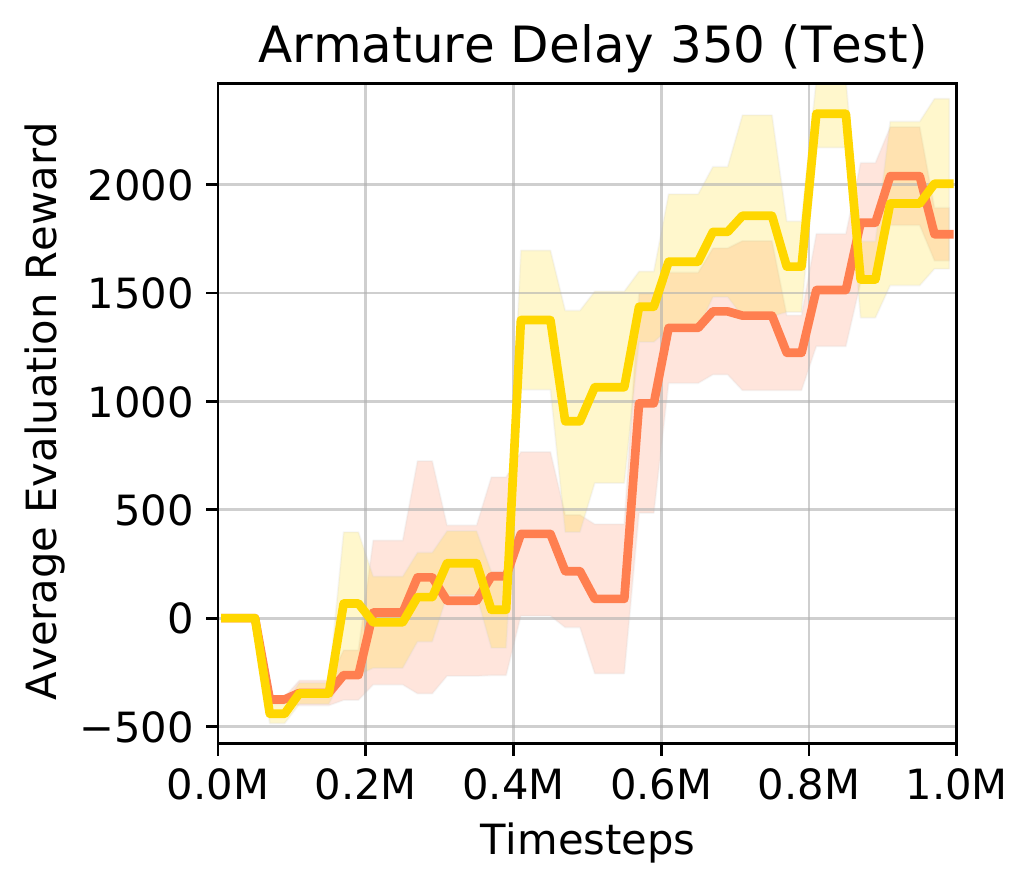}
\end{minipage}%
\hfill
\begin{minipage}{0.16\textwidth}
    \centering
    \includegraphics[width=\linewidth]{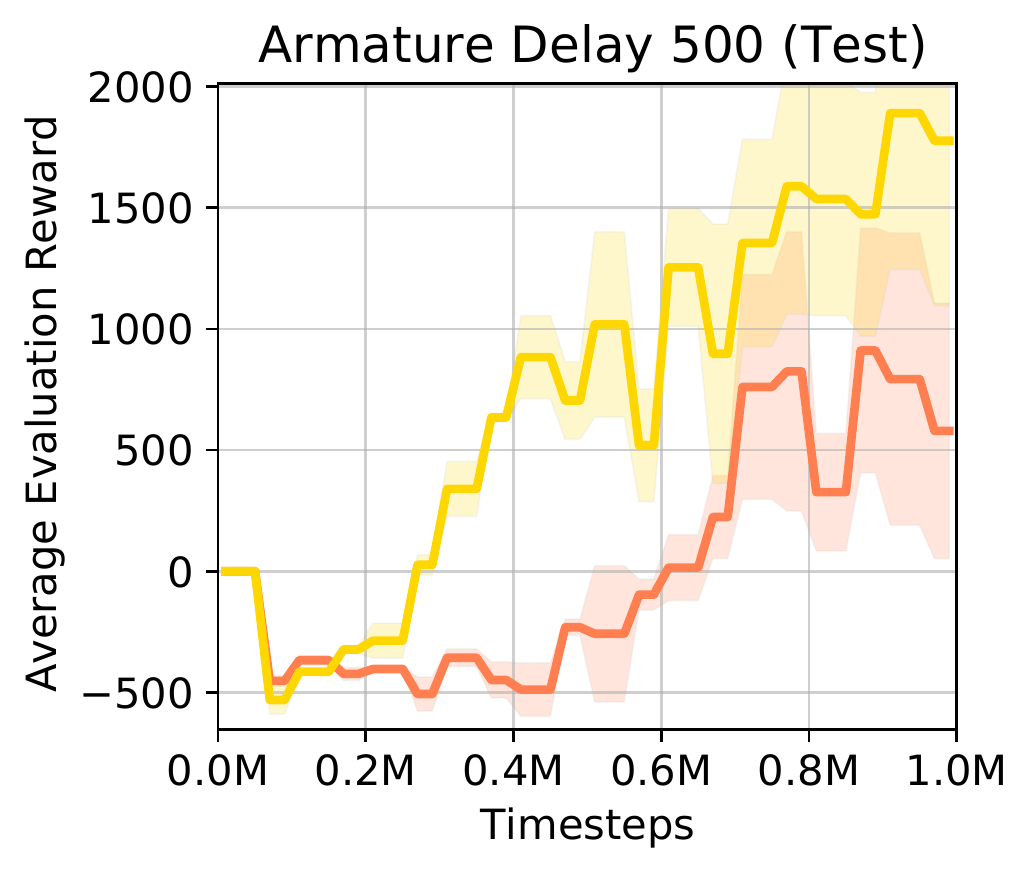}
\end{minipage}
\hfill
\begin{minipage}{0.16\textwidth}
    \centering
    \includegraphics[width=\linewidth]{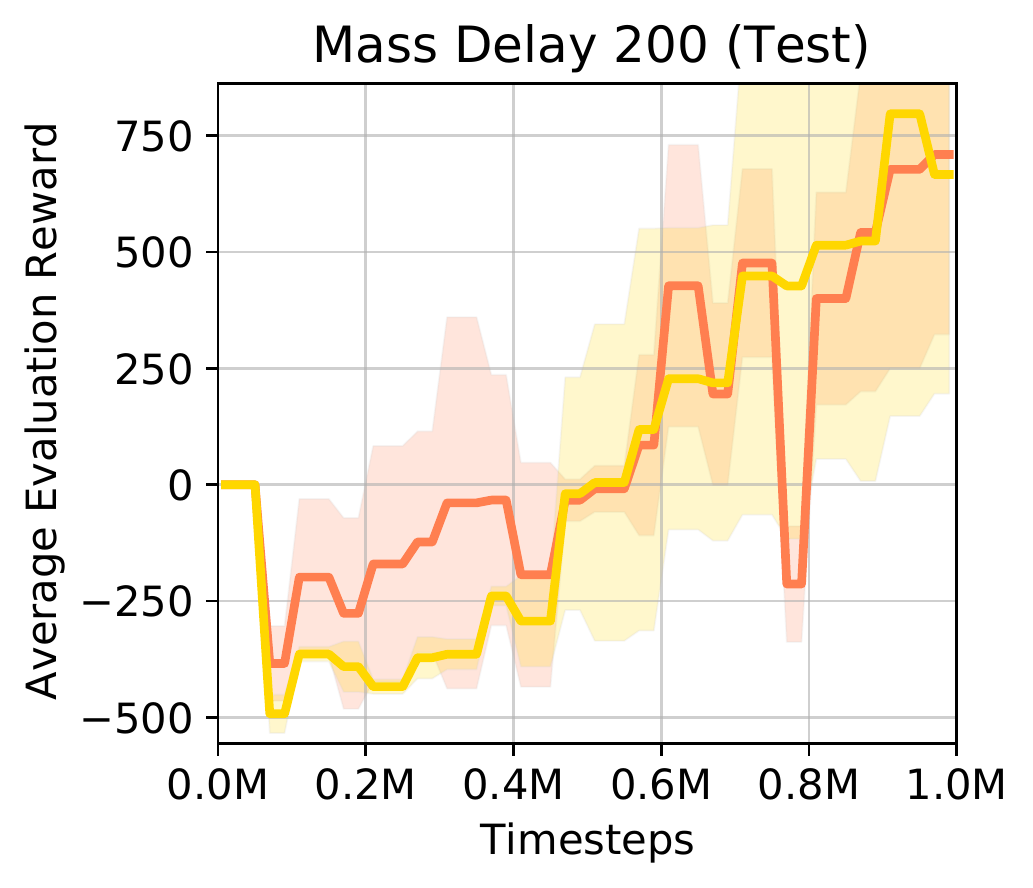}
\end{minipage}
\hfill
\begin{minipage}{0.16\textwidth}
    \centering
    \includegraphics[width=\linewidth]{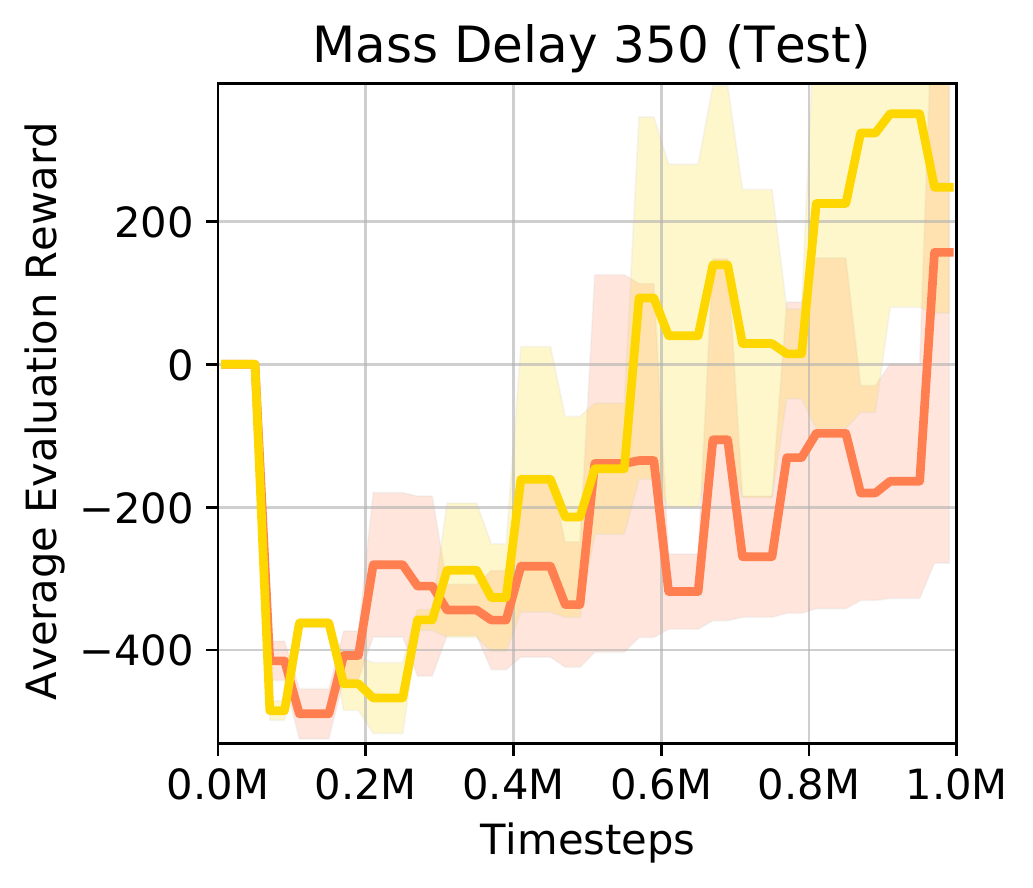}
\end{minipage}
\hfill
\begin{minipage}{0.16\textwidth}
    \centering
    \includegraphics[width=\linewidth]{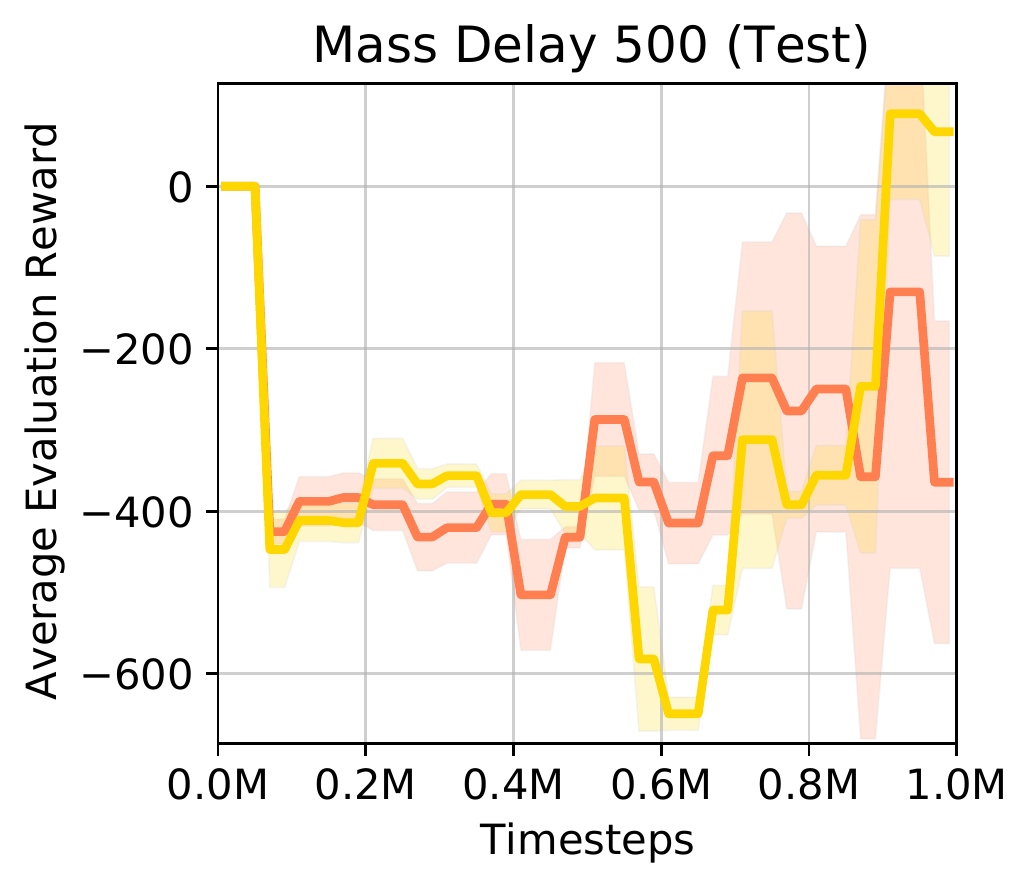}
\end{minipage}
\hfill

    \vspace*{-8pt}
    \caption{Learning curves of the average episode reward, averaged over 3 runs. The average episode reward and standard error are reported on training tasks and test tasks respectively.
}
    \label{fig:effect_sparse}
    \vspace*{-0.2in}
\end{figure*}

\subsection{More Diverse Tasks}
\label{app:diverse}
We include more settings of training and test tasks where the discrepancy among training tasks varies. On HalfCheetah, the environment rewards are delayed for 500 steps.
In Table~\ref{tab:more_sets}, we list the details of the settings.
\begin{table}[!ht]
\centering

\begin{tabular}{cccc}
\toprule
Physics Parameter         & Setting & Train                         & Test                     \\
\midrule
\multirow{3}{*}{Armature} & Set 1   & \{0.2, 0.25, 0.3, 0.35, 0.4\} & \{0.05, 0.1, 0.5, 0.55\} \\
                          & Set 2   & \{0.2, 0.3, 0.4, 0.5, 0.6\}   & \{0.2, 0.3, 0.7, 0.75\}  \\
                          & Set 3   & \{0.2, 0.35, 0.5, 0.65, 0.8\} & \{0.2, 0.3, 0.9, 0.95\}  \\ \hline
\multirow{3}{*}{Mass}     & Set 1   & \{0.5, 0.75, 1.0, 1.25, 1.5\} & \{0.2, 0.3, 1.7, 1.8\}   \\
                          & Set 2   & \{0.5, 1.0, 1.5, 2.0, 2.5\}   & \{0.2, 0.3, 2.7, 2.8\}   \\
                          & Set 3   & \{0.5, 1.25, 2.0, 2.75, 3.5\} & \{0.2, 0.3, 3.7, 3.8\} \\
\bottomrule
\end{tabular}

\caption{Modified physics parameters used in the experiments.}
\label{tab:more_sets}
\end{table}

We consider baseline MQL because it performs reasonably well on HalfCheetah among all the baselines (Figure~\ref{fig:ours_baseline}). Table~\ref{tab:more_distinctive} demonstrates that policy transfer (PT) is generally and consistently effective.
In Figure~\ref{fig:effect_distinct}, we show the learning curves on training and test tasks. In Table~\ref{tab:more_distinctive}, we report the average episode rewards and standard error over 3 runs at 1M timesteps.

\begin{table}[!ht]
\centering
\small
\setlength{\tabcolsep}{3pt}
\begin{tabular}{c|cccccc|cccccc}
\toprule
Setting & \multicolumn{2}{c}{\makecell{Armature\\Set 1}}&
\multicolumn{2}{c}{\makecell{Armature\\Set 2}}&
\multicolumn{2}{c|}{\makecell{Armature\\Set 3}}&
\multicolumn{2}{c}{\makecell{Mass\\Set 1}}  & \multicolumn{2}{c}{\makecell{Mass\\Set 2}}  & \multicolumn{2}{c}{\makecell{Mass\\Set 3}}          \\ \midrule
Task & Train & Test & Train & Test & Train & Test & Train & Test & Train & Test & Train & Test \\ \midrule
MQL & \makecell{-129.3\\[-3pt]\tiny{($\pm$46.7)}} & \makecell{-248.0\\[-3pt]\tiny{($\pm$32.0)}} &
\makecell{-277.2\\[-3pt]\tiny{($\pm$25.2)}} & \makecell{-335.0\\[-3pt]\tiny{($\pm$20.8)}} & \makecell{-85.0\\[-3pt]\tiny{($\pm$33.5)}} & \makecell{-214.7\\[-3pt]\tiny{($\pm$28.9)}} &
\makecell{-100.8\\[-3pt]\tiny{($\pm$37.8)}} & \makecell{-291.3\\[-3pt]\tiny{($\pm$25.8)}} &
\makecell{-403.7\\[-3pt]\tiny{($\pm$16.1)}} & \makecell{-467.8\\[-3pt]\tiny{($\pm$6.5)}} &   \makecell{-175.3\\[-3pt]\tiny{($\pm$6.2)}}& \makecell{-287.9\\[-3pt]\tiny{($\pm$11.7)}} \\[0.5em] 
MCAT w/o PT & \makecell{837.6\\[-3pt]\tiny{($\pm$646.5)}} & \makecell{785.3\\[-3pt]\tiny{($\pm$733.1)}} &
\makecell{924.0\\[-3pt]\tiny{($\pm$690.1)}} & \makecell{579.1\\[-3pt]\tiny{($\pm$527.1)}} & \makecell{452.8\\[-3pt]\tiny{($\pm$386.6)}} & \makecell{616.5\\[-3pt]\tiny{($\pm$305.0)}} &
\makecell{-60.5\\[-3pt]\tiny{($\pm$313.4)}} & \makecell{-258.2\\[-3pt]\tiny{($\pm$151.1)}} &
\makecell{62.5\\[-3pt]\tiny{($\pm$411.0)}} & \makecell{-364.3\\[-3pt]\tiny{($\pm$198.5)}} &   \makecell{-328.1\\[-3pt]\tiny{($\pm$55.8)}}& \makecell{-412.4\\[-3pt]\tiny{($\pm$7.7)}} \\[0.5em] 
MCAT & \makecell{3372.1\\[-3pt]\tiny{($\pm$186.4)}} & \makecell{2821.9\\[-3pt]\tiny{($\pm$137.7)}}  &
\makecell{2207.3\\[-3pt]\tiny{($\pm$697.7)}} & \makecell{1776.8\\[-3pt]\tiny{($\pm$680.8)}} & \makecell{1622.2\\[-3pt]\tiny{($\pm$402.2)}}  & \makecell{918.3\\[-3pt]\tiny{($\pm$142.5)}} &
\makecell{1222.2\\[-3pt]\tiny{($\pm$754.9)}} & \makecell{482.4\\[-3pt]\tiny{($\pm$624.2)}} &
\makecell{763.4\\[-3pt]\tiny{($\pm$377.7)}} & \makecell{67.1\\[-3pt]\tiny{($\pm$152.9)}} & \makecell{705.7\\[-3pt]\tiny{($\pm$503.4)}} & \makecell{-86.2\\[-3pt]\tiny{($\pm$111.8)}} \\ [0.5em] 
Improvement(\%) & 302.6 & 259.3 & 133.9 & 206.8 & 258.3 & 49.0 & 2120.2 &286.8 & 1121.4 & 118.4 & 315.1 & 79.1 \\ \bottomrule
\end{tabular}

\vspace*{-8pt}
\caption{The performance of learned policy on various task settings. We modify \textit{armature} and \textit{mass} to get 5 training tasks and 4 test tasks in each setting. We compute the improvement of MCAT over MCAT w/o PT.}
\label{tab:more_distinctive}
\vspace*{-0.1in}
\end{table}

\begin{figure*}[!ht]
\begin{minipage}{0.16\textwidth}
    \includegraphics[width=\linewidth]{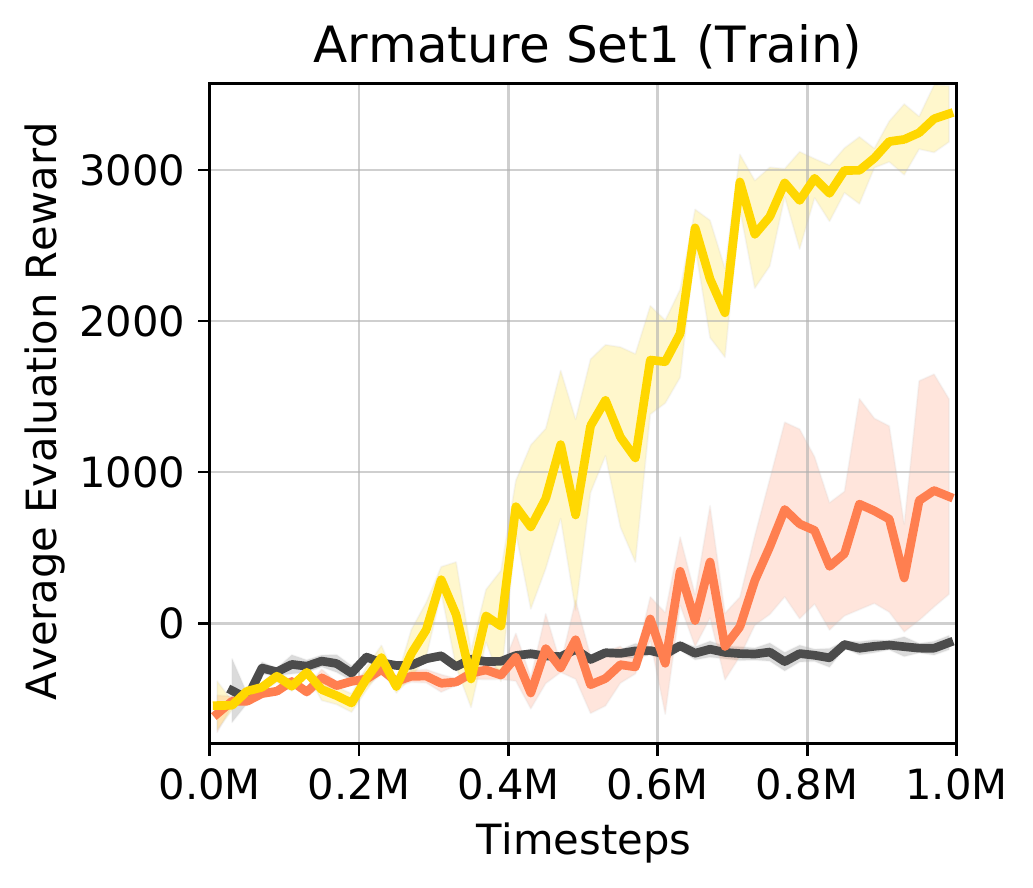}
\end{minipage}%
\hfill
\begin{minipage}{0.16\textwidth}
    \includegraphics[width=\linewidth]{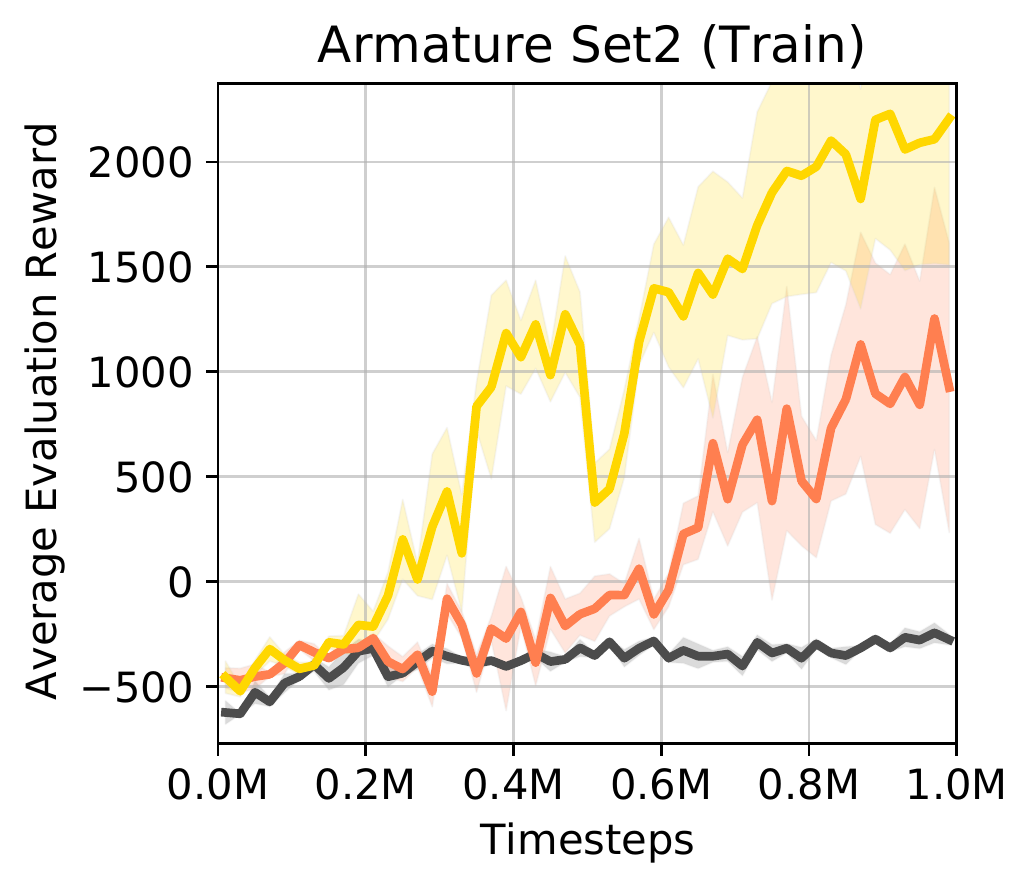}
\end{minipage}%
\hfill
\begin{minipage}{0.16\textwidth}
    \centering
    \includegraphics[width=\linewidth]{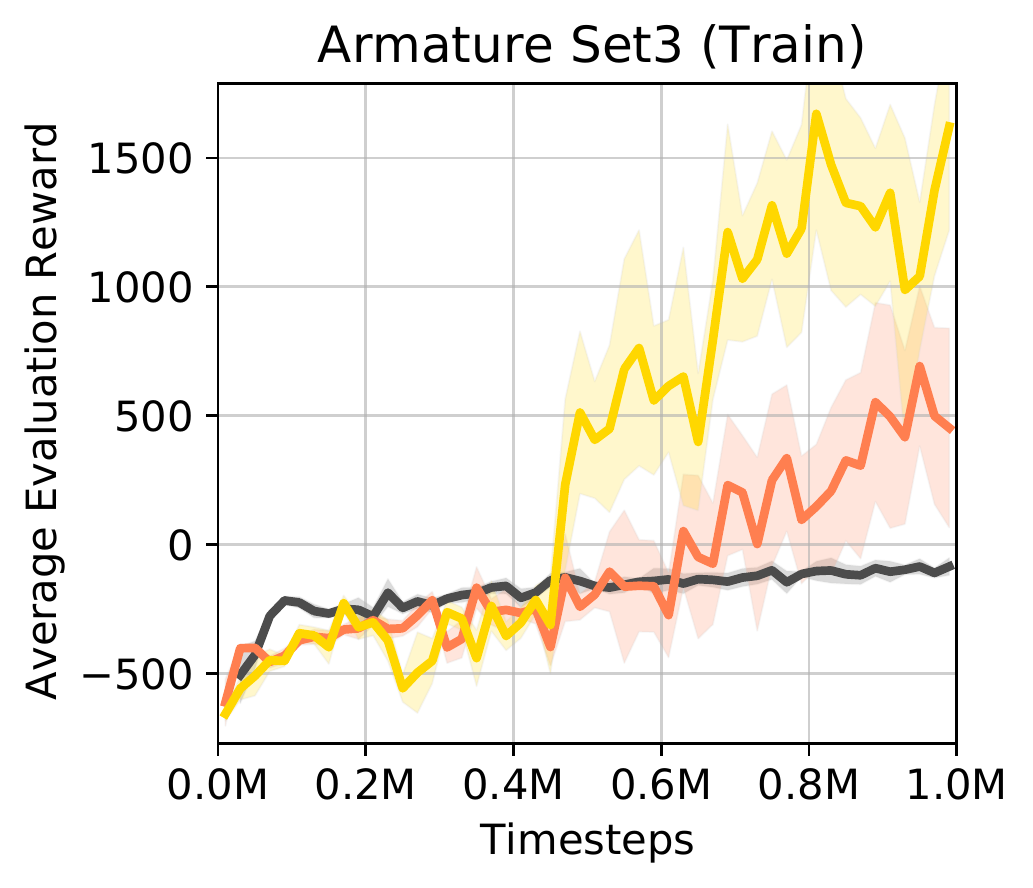}
\end{minipage}
\hfill
\begin{minipage}{0.16\textwidth}
    \centering
    \includegraphics[width=\linewidth]{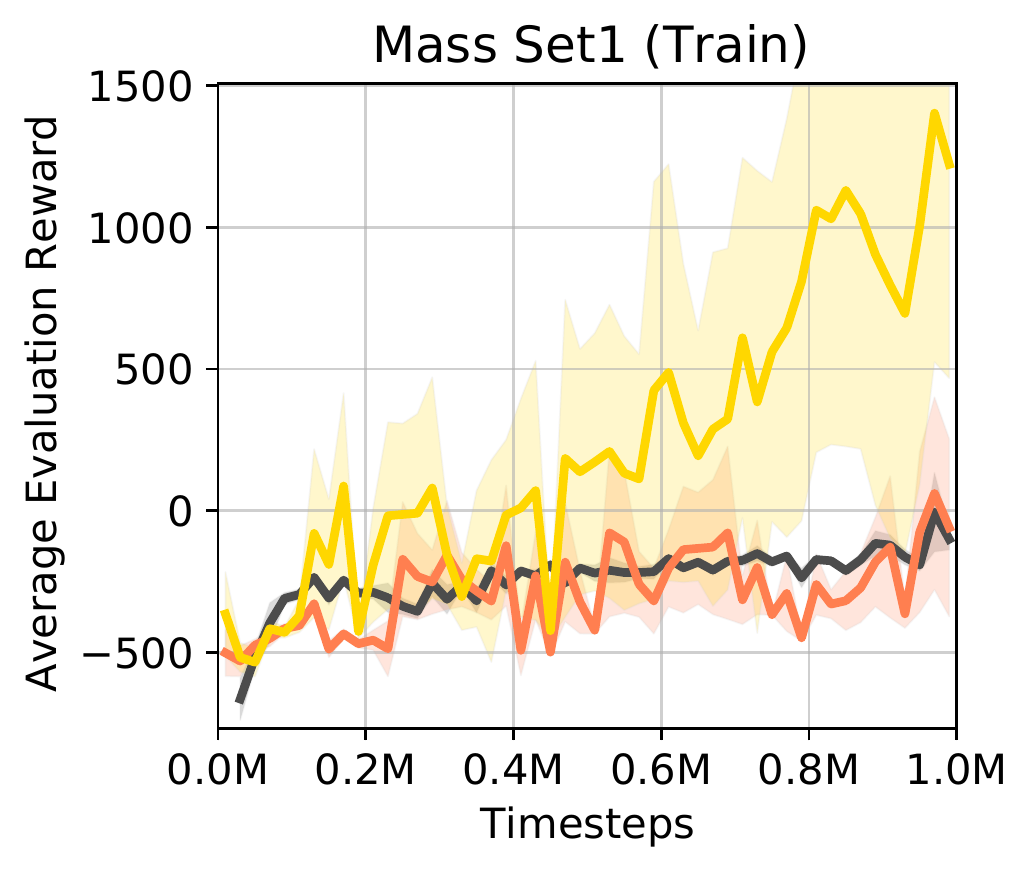}
\end{minipage}
\hfill
\begin{minipage}{0.16\textwidth}
    \centering
    \includegraphics[width=\linewidth]{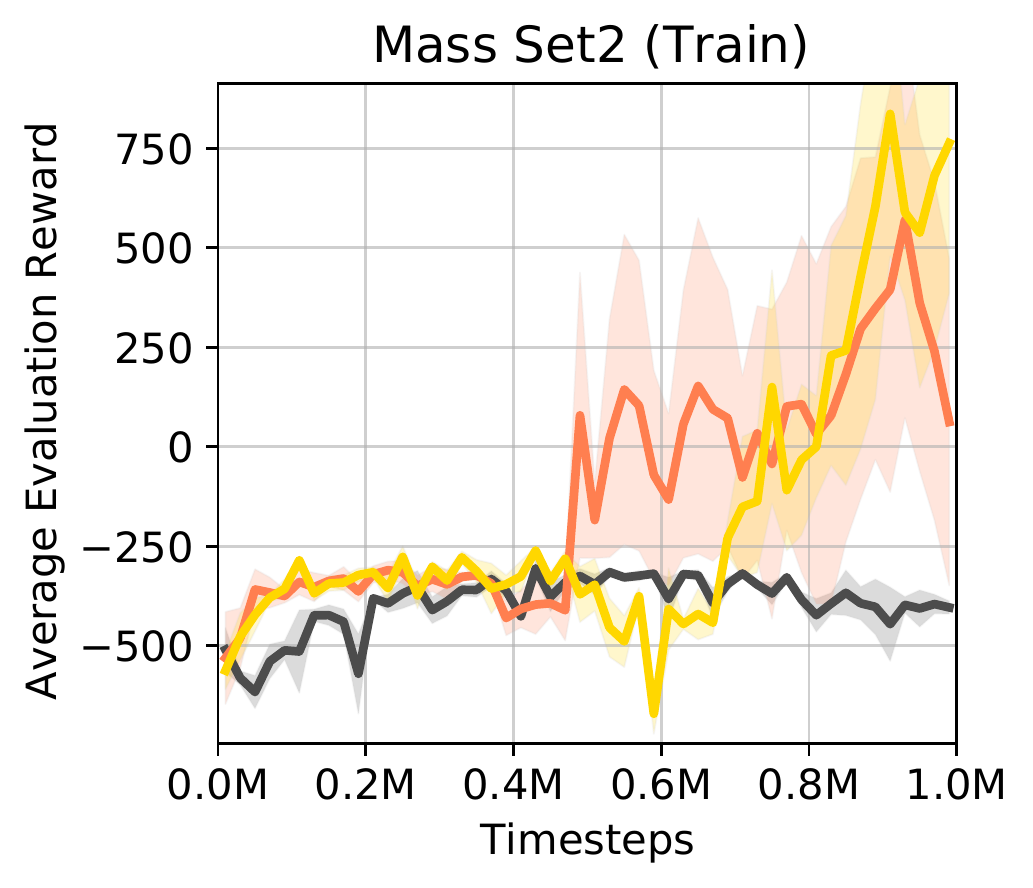}
\end{minipage}
\hfill
\begin{minipage}{0.16\textwidth}
    \centering
    \includegraphics[width=\linewidth]{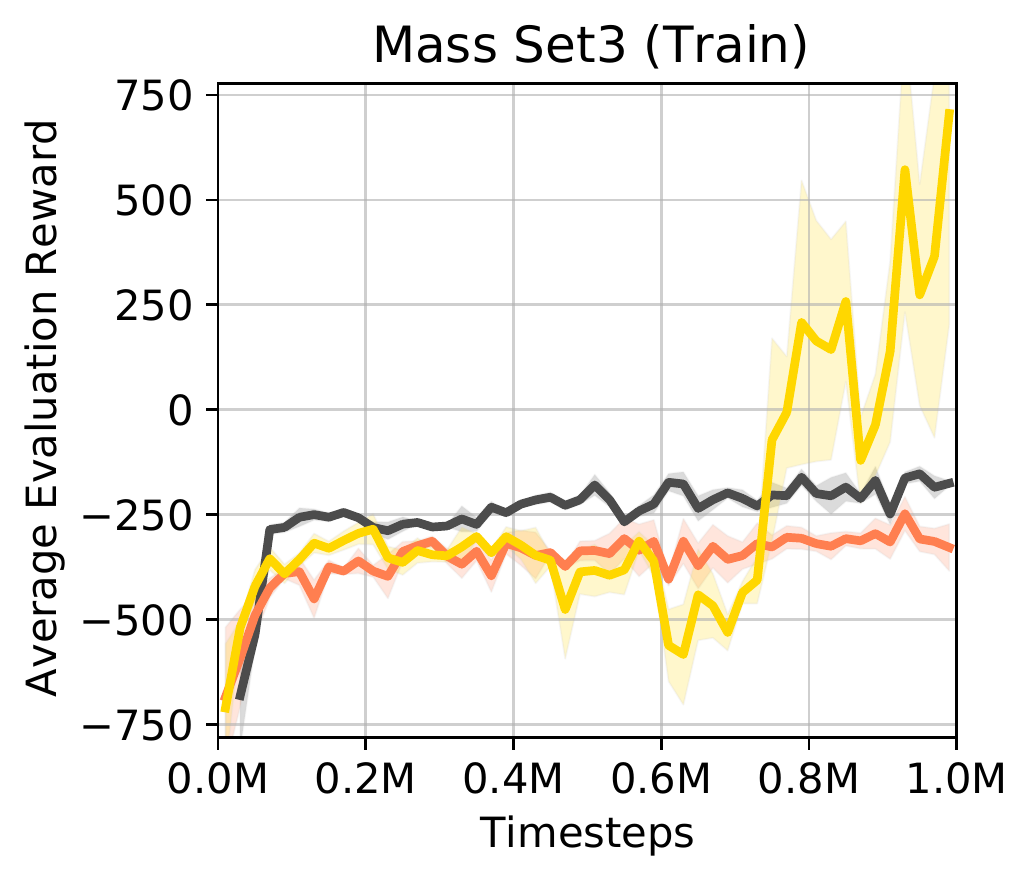}
\end{minipage}
\hfill
\begin{minipage}{0.16\textwidth}
    \centering
    \includegraphics[width=\linewidth]{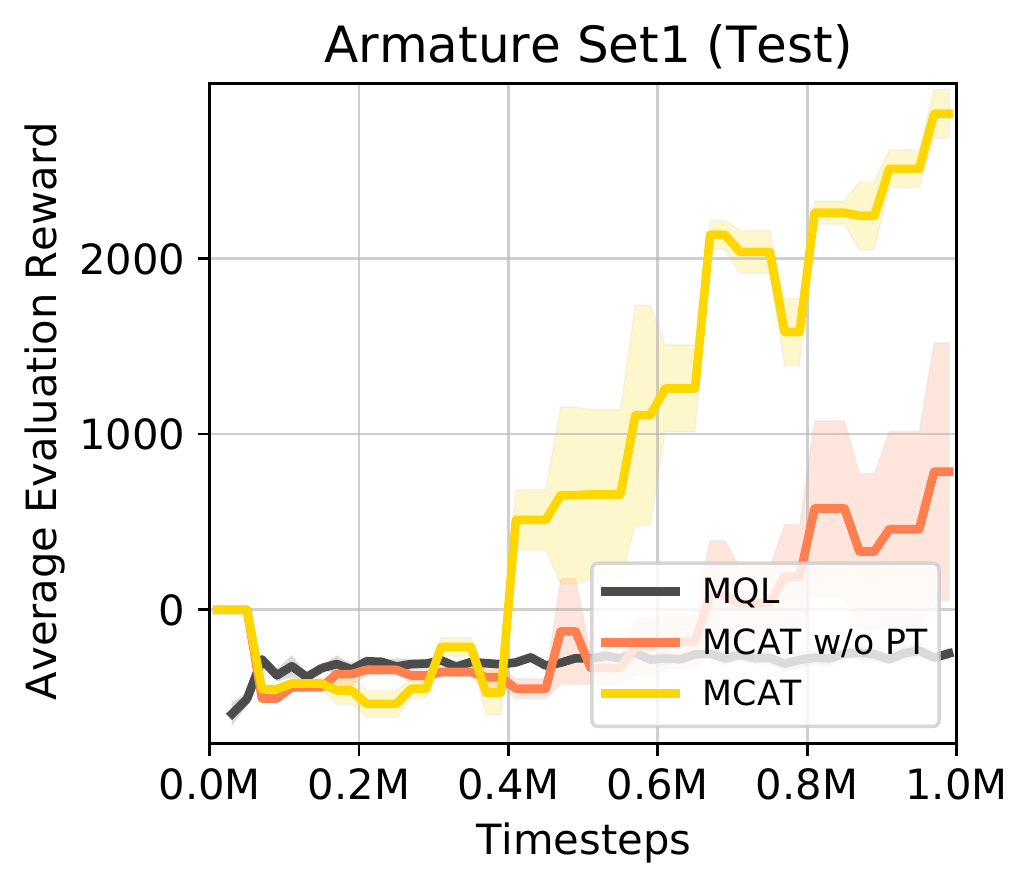}
\end{minipage}%
\hfill
\begin{minipage}{0.16\textwidth}
    \includegraphics[width=\linewidth]{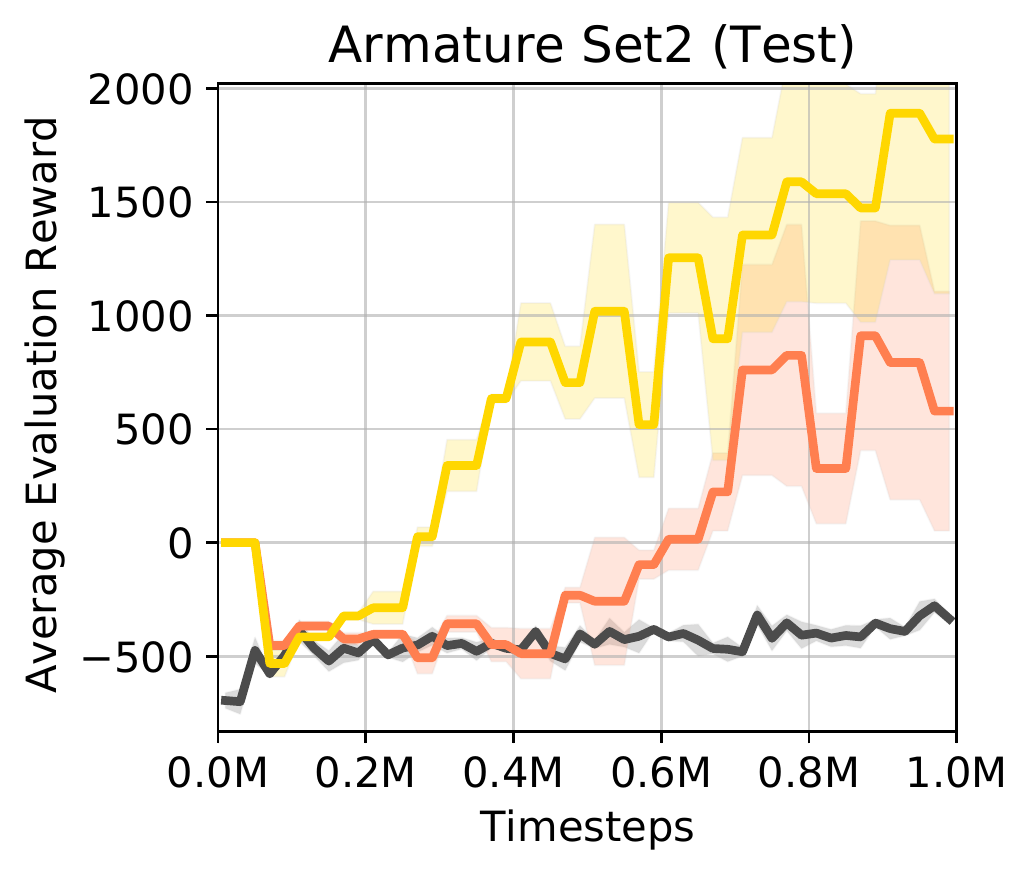}
\end{minipage}%
\hfill
\begin{minipage}{0.16\textwidth}
    \centering
    \includegraphics[width=\linewidth]{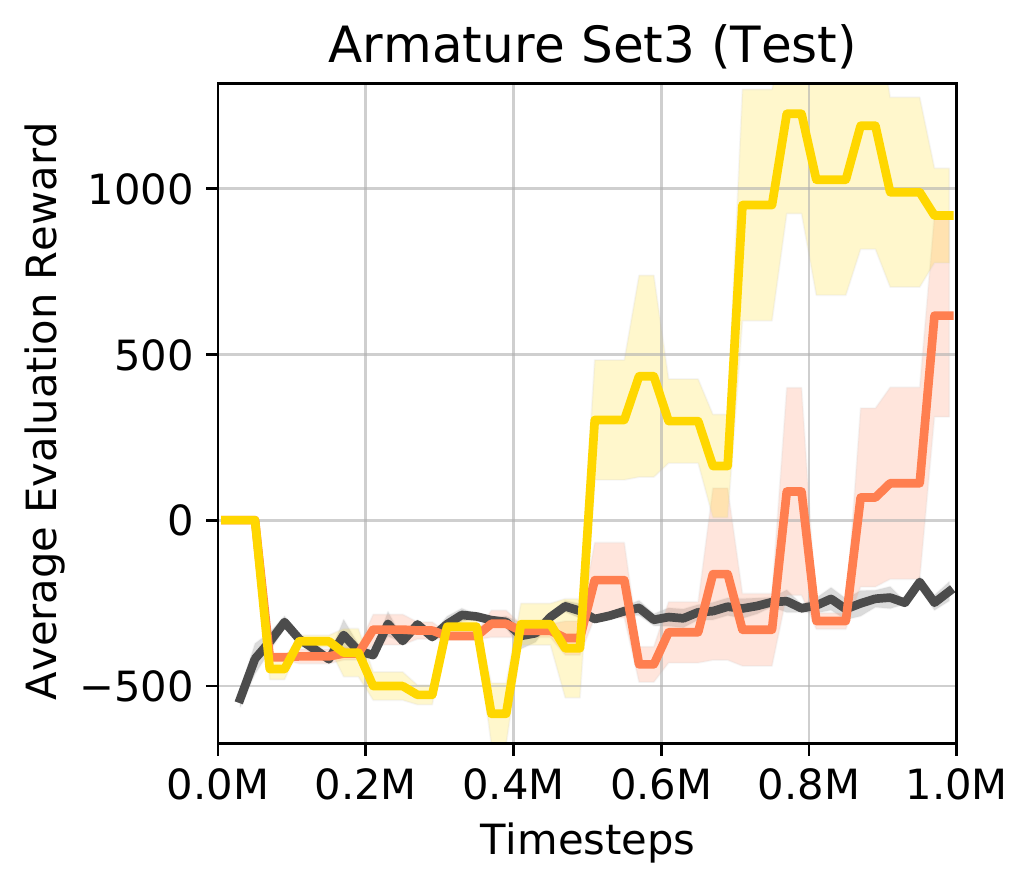}
\end{minipage}
\hfill
\begin{minipage}{0.16\textwidth}
    \centering
    \includegraphics[width=\linewidth]{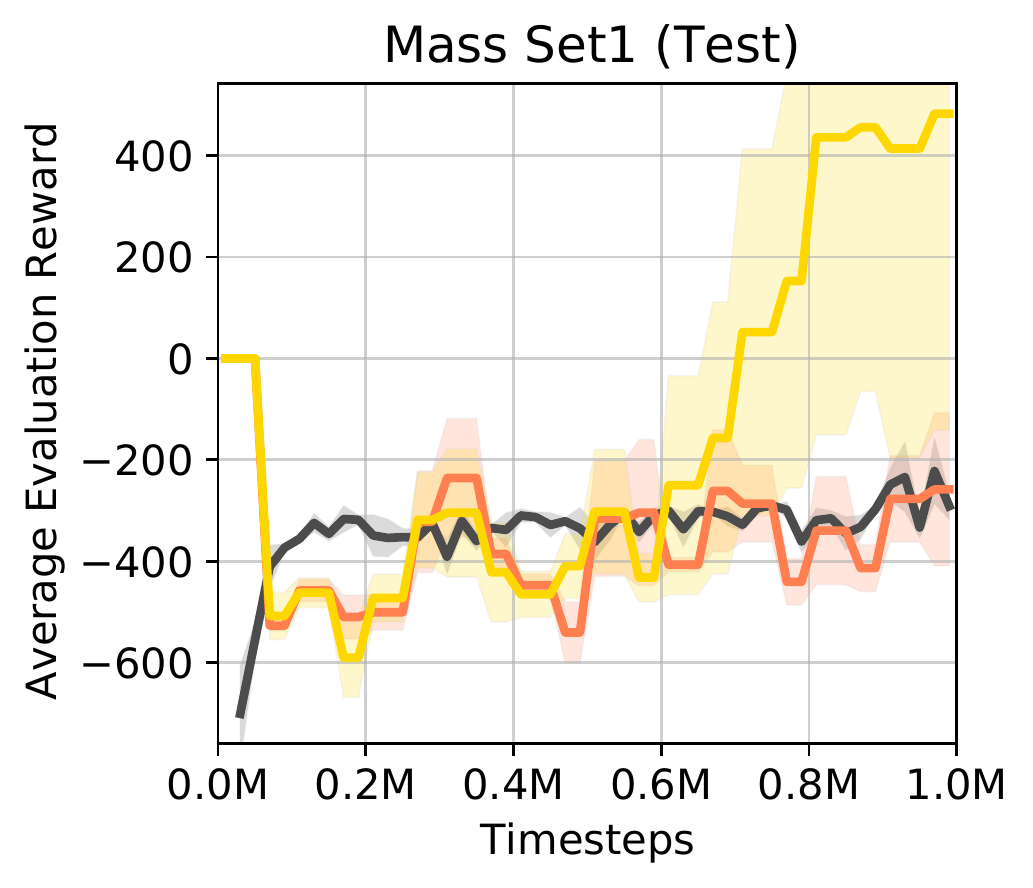}
\end{minipage}
\hfill
\begin{minipage}{0.16\textwidth}
    \centering
    \includegraphics[width=\linewidth]{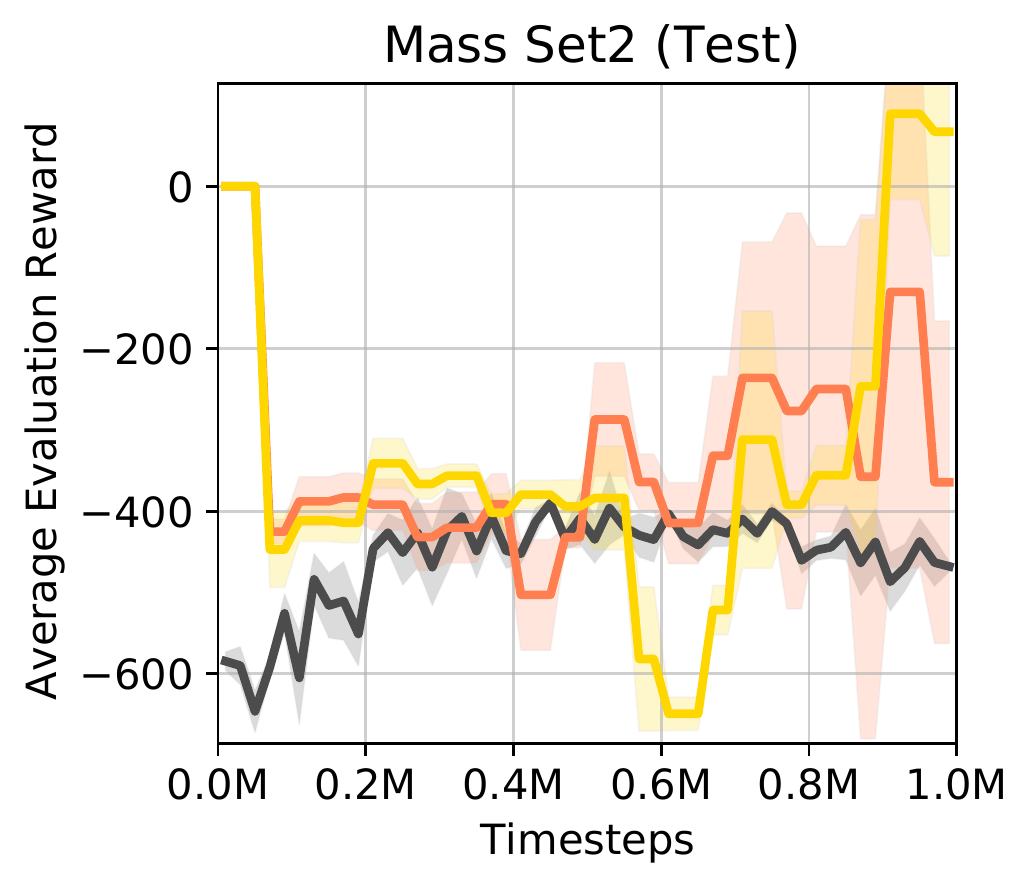}
\end{minipage}
\hfill
\begin{minipage}{0.16\textwidth}
    \centering
    \includegraphics[width=\linewidth]{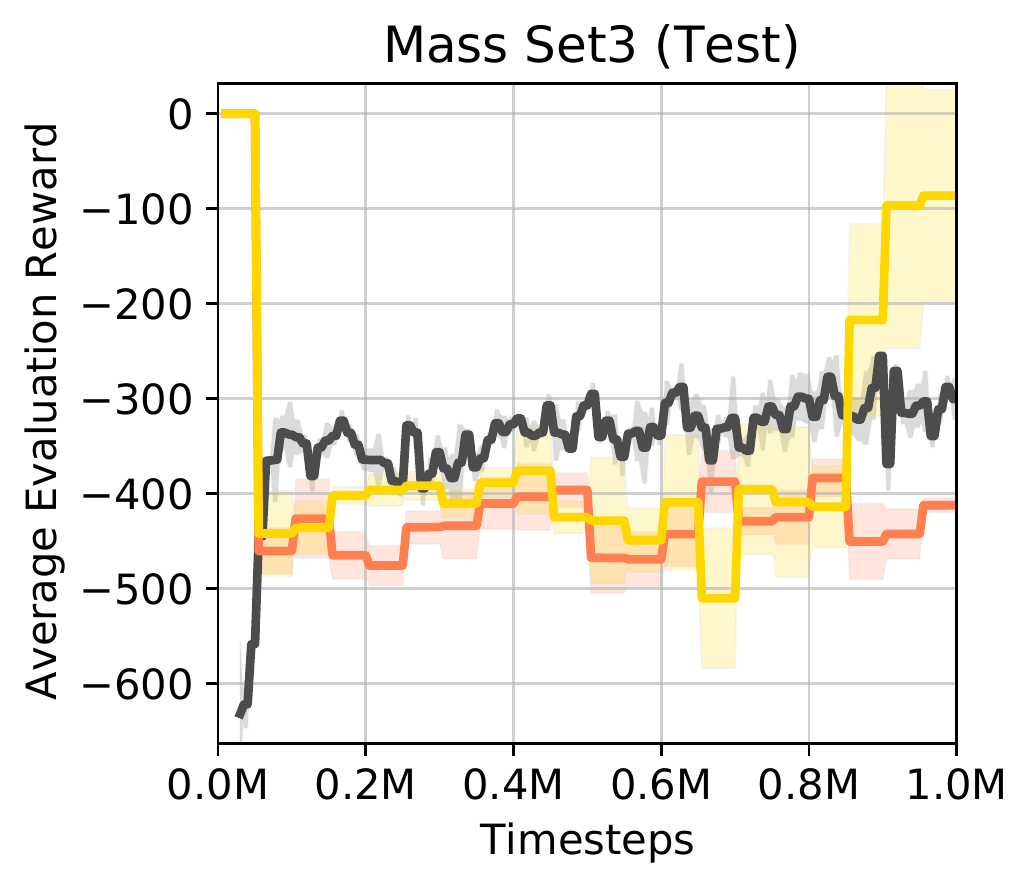}
\end{minipage}
\hfill

    \vspace*{-8pt}
    \caption{Learning curves of the average episode reward, averaged over 3 runs. The average episode reward and standard error are reported on training tasks and test tasks respectively.
}
    \label{fig:effect_distinct}
    \vspace*{-0.2in}
\end{figure*}

\subsection{Effect of Self-Imitation Learning}
\label{app:effect_sil}
We run experiments combining baseline methods with self-imitation learning (SIL) \citep{oh2018self}. SIL brings improvement to baselines but still ours shows significant advantages. In Tab.~\ref{tab:app_effec_sil}, MCAT w/o SIL compares favorably with the baseline methods. MCAT further improves the performance of MCAT w/o SIL, and MCAT outperform the variants of baseline methods with SIL.

\begin{table}[!ht]
\small
\setlength{\tabcolsep}{5pt}
\centering
\begin{tabular}{c|cccccccccc}
\toprule
            Setting 
            &\multicolumn{2}{c}{\begin{tabular}[c]{@{}c@{}}Hopper\\ Size\end{tabular}}            
            &\multicolumn{2}{c}{\begin{tabular}[c]{@{}c@{}}HalfCheetah\\ Armature\end{tabular}}
            &\multicolumn{2}{c}{\begin{tabular}[c]{@{}c@{}}HalfCheetah\\ Mass\end{tabular}}
            &\multicolumn{2}{c}{\begin{tabular}[c]{@{}c@{}}Ant\\ Damping\end{tabular}}
            &\multicolumn{2}{c}{\begin{tabular}[c]{@{}c@{}}Ant\\ Cripple\end{tabular}}\\
            \midrule
            
 Task & Train
 & Test
 & Train
 & Test 
 & Train
 & Test
 & Train
 & Test
 & Train
 & Test
\\ \midrule

\makecell{MQL\\\citep{fakoor2019meta}}
&\begin{tabular}[c]{@{}c@{}}1586.1\\[-3pt]\tiny{($\pm$ 321.4)}\end{tabular}
&\begin{tabular}[c]{@{}c@{}}1607.5\\[-3pt]\tiny{($\pm$ 327.5)}\end{tabular}
& \begin{tabular}[c]{@{}c@{}}-31.4\\[-3pt]\tiny{($\pm$ 243.5)}\end{tabular}
& \begin{tabular}[c]{@{}c@{}}-77.9\\[-3pt]\tiny{($\pm$ 214.3)}\end{tabular}
& \begin{tabular}[c]{@{}c@{}}-243.1\\[-3pt]\tiny{($\pm$ 69.8)}\end{tabular}
& \begin{tabular}[c]{@{}c@{}}-413.9\\[-3pt]\tiny{($\pm$ 11.1)}\end{tabular}
&\begin{tabular}[c]{@{}c@{}}93.8\\[-3pt]\tiny{($\pm$ 24.5)}\end{tabular}
&\begin{tabular}[c]{@{}c@{}}103.1\\[-3pt]\tiny{($\pm$ 35.7)}\end{tabular}
&\begin{tabular}[c]{@{}c@{}}17.4\\[-3pt]\tiny{($\pm$ 4.3)}\end{tabular}
&\begin{tabular}[c]{@{}c@{}}38.2\\[-3pt]\tiny{($\pm$ 4.0)}\end{tabular}
\\ 
\makecell{Distral\\\citep{teh2017distral}}
&\begin{tabular}[c]{@{}c@{}}1364.0\\[-3pt]\tiny{($\pm$ 216.3)}\end{tabular}
&\begin{tabular}[c]{@{}c@{}}1319.8\\[-3pt]\tiny{($\pm$ 162.2)}\end{tabular}
&\begin{tabular}[c]{@{}c@{}}774.7\\[-3pt]\tiny{($\pm$ 405.9})\end{tabular}
&\begin{tabular}[c]{@{}c@{}}566.9\\[-3pt]\tiny{($\pm$ 246.7)}\end{tabular}
&\begin{tabular}[c]{@{}c@{}}-54.3\\[-3pt]\tiny{($\pm$ 14.8)}\end{tabular}
&\begin{tabular}[c]{@{}c@{}}-29.5\\[-3pt]\tiny{($\pm$ 3.0)}\end{tabular}   
&\begin{tabular}[c]{@{}c@{}}123.0\\[-3pt]\tiny{($\pm$ 20.0)}\end{tabular}
&\begin{tabular}[c]{@{}c@{}}90.5\\[-3pt]\tiny{($\pm$ 28.4)}\end{tabular}
&\begin{tabular}[c]{@{}c@{}}-2.5\\[-3pt]\tiny{($\pm$ 1.7)}\end{tabular}
&\begin{tabular}[c]{@{}c@{}}-0.1\\[-3pt]\tiny{($\pm$ 0.7)}\end{tabular}
\\
\makecell{HiP-BMDP\\\citep{zhang2020robust}}
&\begin{tabular}[c]{@{}c@{}}1590.3\\[-3pt]\tiny{($\pm$ 238.7)}\end{tabular}
&\begin{tabular}[c]{@{}c@{}}1368.3\\[-3pt]\tiny{($\pm$ 150.7)}\end{tabular}
&\begin{tabular}[c]{@{}c@{}}-212.4\\[-3pt]\tiny{($\pm$ 52.2)}\end{tabular}
&\begin{tabular}[c]{@{}c@{}}-102.4\\[-3pt]\tiny{($\pm$ 24.9)}\end{tabular}
&\begin{tabular}[c]{@{}c@{}}-81.3\\[-3pt]\tiny{($\pm$ 8.31)}\end{tabular}
&\begin{tabular}[c]{@{}c@{}}-101.8\\[-3pt]\tiny{($\pm$ 29.6)}\end{tabular}
&\begin{tabular}[c]{@{}c@{}}15.0\\[-3pt]\tiny{($\pm$ 5.7)}\end{tabular}
&\begin{tabular}[c]{@{}c@{}}33.1\\[-3pt]\tiny{($\pm$ 6.0)}\end{tabular}
&\begin{tabular}[c]{@{}c@{}}12.7\\[-3pt]\tiny{($\pm$ 5.3)}\end{tabular}
&\begin{tabular}[c]{@{}c@{}}7.3\\[-3pt]\tiny{($\pm$ 2.6)}\end{tabular}
\\
MCAT w/o SIL
&\begin{tabular}[c]{@{}c@{}}1261.6\\[-3pt]\tiny{($\pm$ 55.2)}\end{tabular}
&\begin{tabular}[c]{@{}c@{}}1165.1\\[-3pt]\tiny{($\pm$ 8.6)}\end{tabular}
&\begin{tabular}[c]{@{}c@{}}1548.8\\[-3pt]\tiny{($\pm$ 418.4)}\end{tabular}
&\begin{tabular}[c]{@{}c@{}}883.8\\[-3pt]\tiny{($\pm$ 267.2)}\end{tabular}
&\begin{tabular}[c]{@{}c@{}}610.6\\[-3pt]\tiny{($\pm$ 482.3)}\end{tabular}
&\begin{tabular}[c]{@{}c@{}}119.0\\[-3pt]\tiny{($\pm$ 210.0)}\end{tabular}
&\begin{tabular}[c]{@{}c@{}}123.3\\[-3pt]\tiny{($\pm$ 25.8})\end{tabular}
&\begin{tabular}[c]{@{}c@{}}123.8\\[-3pt]\tiny{($\pm$ 26.9)}\end{tabular}
&\begin{tabular}[c]{@{}c@{}}97.3\\[-3pt]\tiny{($\pm$ 3.6)}\end{tabular}
&\begin{tabular}[c]{@{}c@{}}163.1\\[-3pt]\tiny{($\pm$ 26.1)}\end{tabular}
\\ \midrule
MQL+SIL
&\begin{tabular}[c]{@{}c@{}}1395.5\\[-3pt]\tiny{($\pm$ 60.8)}\end{tabular}
&\begin{tabular}[c]{@{}c@{}}1398.9\\[-3pt]\tiny{($\pm$ 85.9)}\end{tabular}
& \begin{tabular}[c]{@{}c@{}}1399.7\\[-3pt]\tiny{($\pm$ 350.2)}\end{tabular}
& \begin{tabular}[c]{@{}c@{}}743.5\\[-3pt]\tiny{($\pm$ 246.1)}\end{tabular}
& \begin{tabular}[c]{@{}c@{}}617.8\\[-3pt]\tiny{($\pm$ 133.1)}\end{tabular}
&\begin{tabular}[c]{@{}c@{}}-63.3\\[-3pt]\tiny{($\pm$ 158.3)}\end{tabular}
&\begin{tabular}[c]{@{}c@{}}153.0\\[-3pt]\tiny{($\pm$ 28.3)}\end{tabular}
&\begin{tabular}[c]{@{}c@{}}144.3\\[-3pt]\tiny{($\pm$ 28.1)}\end{tabular}
&\begin{tabular}[c]{@{}c@{}}13.9\\[-3pt]\tiny{($\pm$ 19.8)}\end{tabular}
&\begin{tabular}[c]{@{}c@{}}10.2\\[-3pt]\tiny{($\pm$ 2.3)}\end{tabular}
\\
Distral+SIL
&\begin{tabular}[c]{@{}c@{}}1090.2\\[-3pt]\tiny{($\pm$ 18.7)}\end{tabular}
&\begin{tabular}[c]{@{}c@{}}1090.9\\[-3pt]\tiny{($\pm$ 7.8)}\end{tabular}
&\begin{tabular}[c]{@{}c@{}}1014.1\\[-3pt]\tiny{($\pm$ 121.4)}\end{tabular}
&\begin{tabular}[c]{@{}c@{}}970.3\\[-3pt]\tiny{($\pm$ 164.2)}\end{tabular}
&\begin{tabular}[c]{@{}c@{}}809.7\\[-3pt]\tiny{($\pm$ 294.2)}\end{tabular}
&\begin{tabular}[c]{@{}c@{}}746.7\\[-3pt]\tiny{($\pm$ 120.5)}\end{tabular}
&\begin{tabular}[c]{@{}c@{}}174.3\\[-3pt]\tiny{($\pm$ 66.1)}\end{tabular}
&\begin{tabular}[c]{@{}c@{}}122.2\\[-3pt]\tiny{($\pm$ 44.5)}\end{tabular}
&\begin{tabular}[c]{@{}c@{}}107.7\\[-3pt]\tiny{($\pm$ 57.7)}\end{tabular}
&\begin{tabular}[c]{@{}c@{}}9.1\\[-3pt]\tiny{($\pm$ 5.0)}\end{tabular}
\\
HiP-BMDP+SIL 
&\begin{tabular}[c]{@{}c@{}}1573.3\\[-3pt]\tiny{($\pm$ 32.4)}\end{tabular}
&\begin{tabular}[c]{@{}c@{}}1589.5\\[-3pt]\tiny{($\pm$ 110.3)}\end{tabular}
&\begin{tabular}[c]{@{}c@{}}954.8\\[-3pt]\tiny{($\pm$ 192.3)}\end{tabular}
&\begin{tabular}[c]{@{}c@{}}713.3\\[-3pt]\tiny{($\pm$ 85.4)}\end{tabular}
&\begin{tabular}[c]{@{}c@{}}953.5\\[-3pt]\tiny{($\pm$ 61.2)}\end{tabular}
&\begin{tabular}[c]{@{}c@{}}506.6\\[-3pt]\tiny{($\pm$ 99.0)}\end{tabular}
&\begin{tabular}[c]{@{}c@{}}653.9\\[-3pt]\tiny{($\pm$ 262.6)}\end{tabular}
&\begin{tabular}[c]{@{}c@{}}523.6\\[-3pt]\tiny{($\pm$ 300.8)}\end{tabular}
&\begin{tabular}[c]{@{}c@{}}\textbf{170.9}\\[-3pt]\tiny{($\pm$ 68.7)}\end{tabular}
&\begin{tabular}[c]{@{}c@{}}215.4\\[-3pt]\tiny{($\pm$ 130.3)}\end{tabular}
\\
MCAT (Ours)
&\begin{tabular}[c]{@{}c@{}}\textbf{2278.8}\\[-3pt]\tiny{($\pm$ 426.2)}\end{tabular}
&\begin{tabular}[c]{@{}c@{}}\textbf{1914.8}\\[-3pt]\tiny{($\pm$ 373.2)}\end{tabular}
&\begin{tabular}[c]{@{}c@{}}\textbf{2267.2}\\[-3pt]\tiny{($\pm$ 579.2)}\end{tabular}
&\begin{tabular}[c]{@{}c@{}}\textbf{2071.5}\\[-3pt]\tiny{($\pm$ 447.4)}\end{tabular}  
&\begin{tabular}[c]{@{}c@{}}\textbf{2226.3}\\[-3pt]\tiny{($\pm$ 762.6)}\end{tabular}
&\begin{tabular}[c]{@{}c@{}}\textbf{1771.1}\\[-3pt]\tiny{($\pm$ 617.7)}\end{tabular}
&\begin{tabular}[c]{@{}c@{}}\textbf{1322.7}\\[-3pt]\tiny{($\pm$ 57.4)}\end{tabular}
&\begin{tabular}[c]{@{}c@{}}\textbf{1014.0}\\[-3pt]\tiny{($\pm$ 69.9)}\end{tabular}
&\begin{tabular}[c]{@{}c@{}}110.4\\[-3pt]\tiny{($\pm$ 30.5)}\end{tabular}
&\begin{tabular}[c]{@{}c@{}}\textbf{281.6}\\[-3pt]\tiny{($\pm$ 65.6)}\end{tabular}
\\ \bottomrule
\end{tabular}
\caption{Mean ($\pm$ standard error) of episode rewards on the training and test tasks, at 2M timesteps.}
\label{tab:app_effec_sil}
\end{table}

On one task, SIL boosts the performance by exploiting the successful past experiences.
But on multiple tasks, enhancing performance on one task with luckily collected good experiences may not benefit the exploration on other tasks.
If other tasks have never seen the good performance before, SIL might even prevent the exploration on these tasks because the shared policy is trained to overfit highly-rewarding transitions on the one task with good past trajectories.
We observe that after combining with SIL, the baselines show even more severe performance imbalance among multiple training tasks.
Therefore, we believe the idea of policy transfer is complementary to SIL, in that it makes each task benefit from good policies on any other tasks.

\subsection{Effect of Contrastive Loss}
\label{app:effect_cont}
To show the contrastive loss indeed helps policy transfer, we compare our method with and without the contrastive loss $\mathcal{L}_{cont}$ ( Equation~\ref{eq:cont}).
In Fig.~\ref{fig:contrastive}, one can observe that $\smash{\mathcal{L}_{cont}}$ helps cluster the embeddings of samples from the same task and separate the embeddings from different tasks.
We note that the tasks $\mathcal{T}^{(1)}, \mathcal{T}^{(2)}, \mathcal{T}^{(3)}, \mathcal{T}^{(4)}, \mathcal{T}^{(5)}$ have different values of the physics parameter armature $0.2, 0.3, 0.4, 0.5, 0.6$.
As mentioned in Sec.~\ref{sec:context}, the learned context embeddings maintain the similarity between tasks.
In Fig.~\ref{fig:contrastive}, the context embeddings of two tasks are closer if their values of armature is closer.

\begin{figure*}[!ht]
\vspace*{-0.1in}
\centering
\begin{subfigure}{.3\textwidth}
  \centering
  \includegraphics[width=0.8\linewidth]{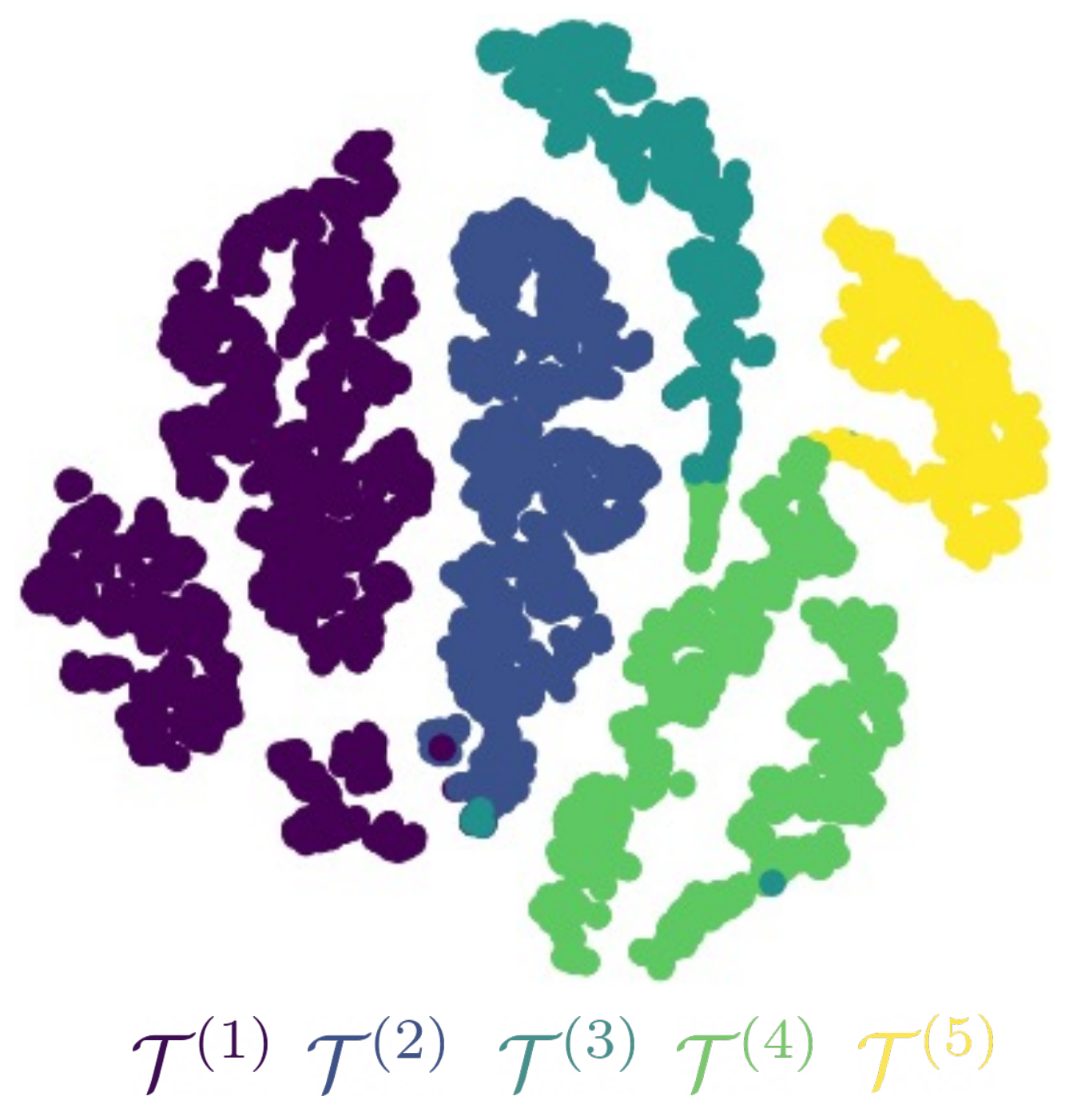}
  \vspace*{-8pt}
  \captionsetup{font=tiny, labelfont=tiny}
  \caption{Context embeddings of random samples in MCAT}
  \label{fig:with_cont}
\end{subfigure}%
\begin{subfigure}{.3\textwidth}
  \centering
  \includegraphics[width=0.8\linewidth]{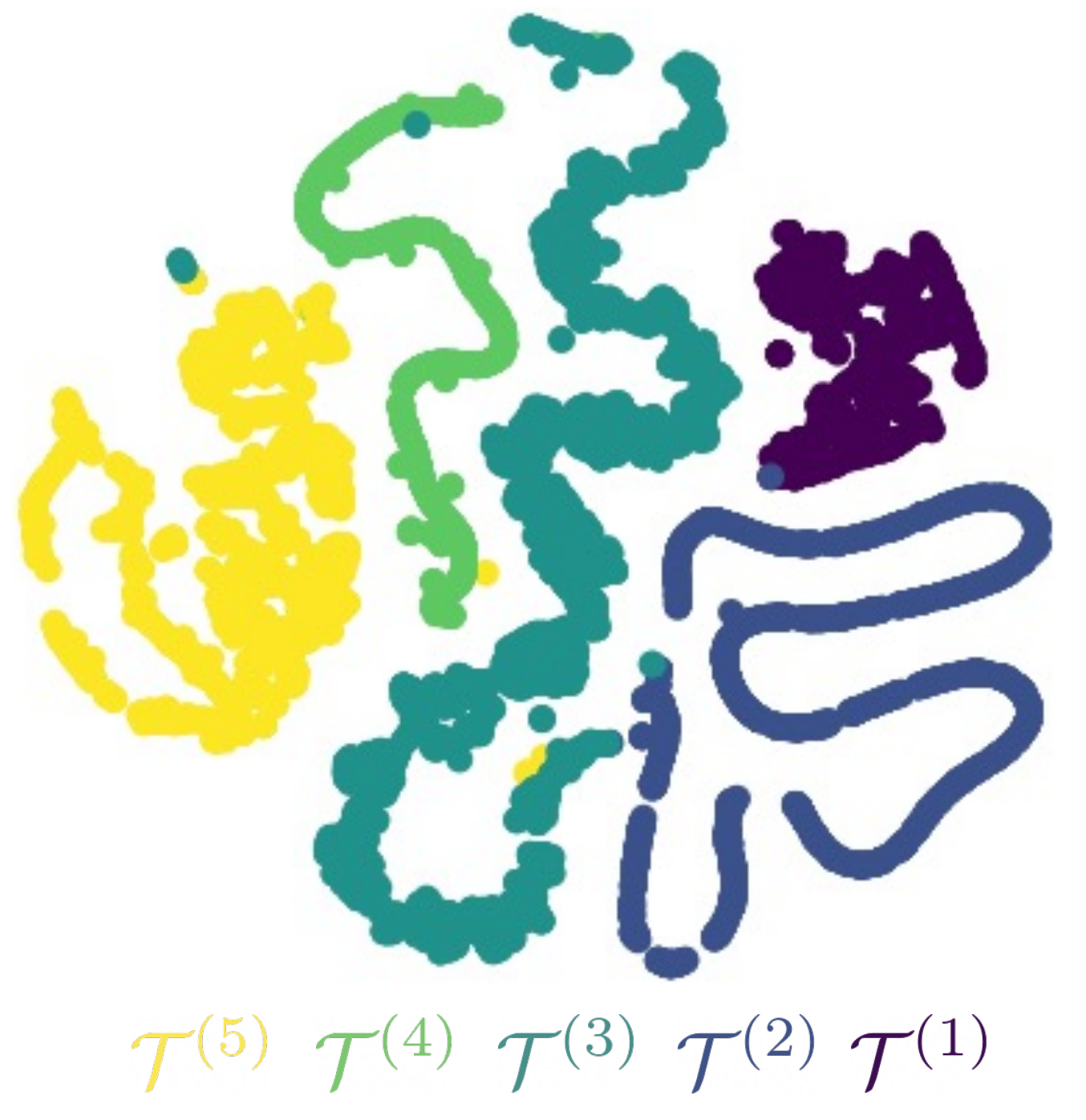}
  \vspace*{-8pt}
  \captionsetup{font=tiny, labelfont=tiny}
  \caption{Context embeddings in MCAT without contrastive loss}
  \label{fig:without_cont}
\end{subfigure}%
\vspace*{-8pt}
\caption{t-SNE visualization\citep{van2008visualizing} of the context embeddings learned via our method with and without contrastive loss. 
Different colors correspond to different training tasks.}
\label{fig:contrastive}
\vspace*{-0.1in}
\end{figure*}

MCAT shows superior performance to the variant without the contrastive loss.
Here we show the learning curves on training and test tasks separately(Fig.~\ref{fig:effect_contrastive}).
\begin{figure*}[!ht]
\centering
\begin{minipage}{0.2\textwidth}
    \includegraphics[width=\linewidth]{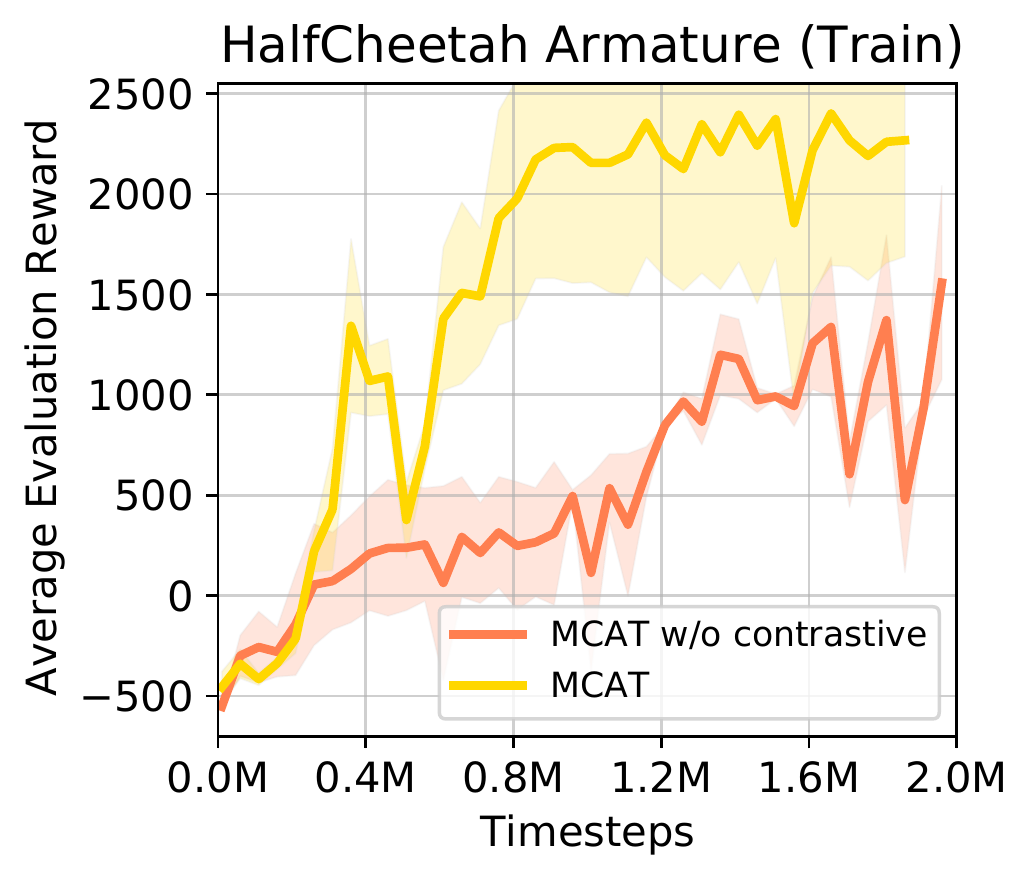}
\end{minipage}%
\hspace*{2pt}
\begin{minipage}{0.2\textwidth}
    \includegraphics[width=\linewidth]{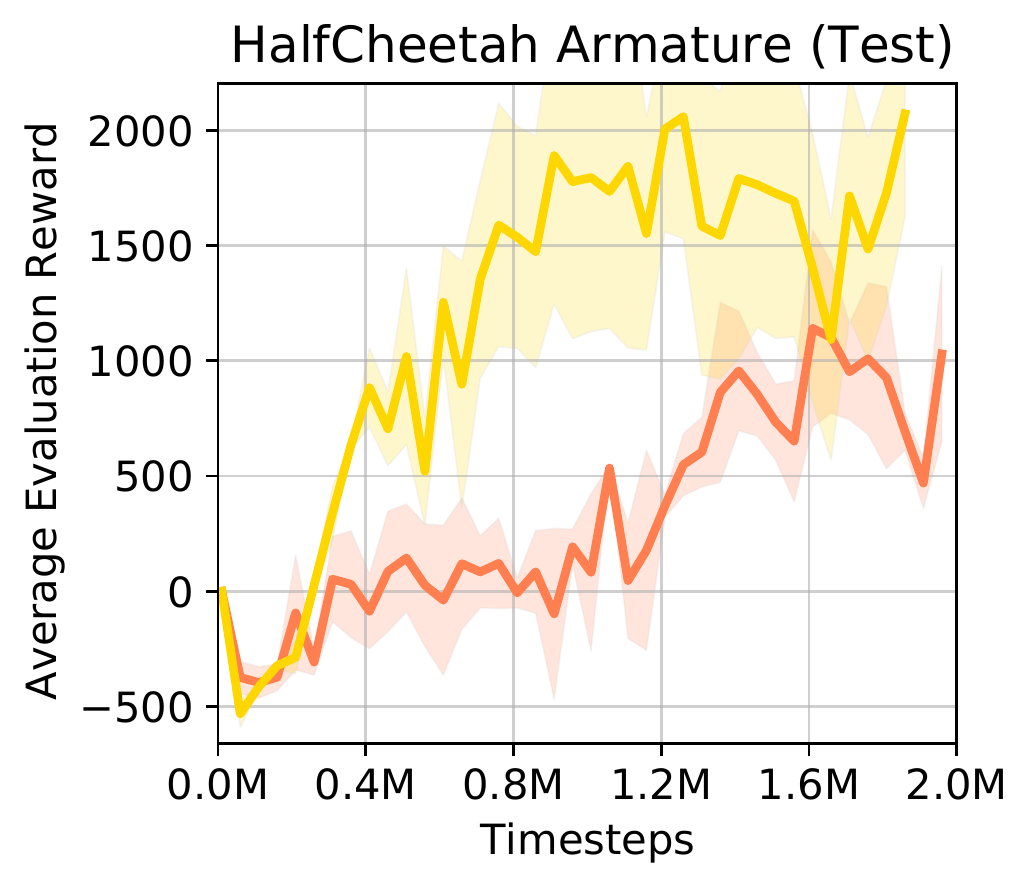}
\end{minipage}%
\hspace*{2pt}
\begin{minipage}{0.2\textwidth}
    \centering
    \includegraphics[width=\linewidth]{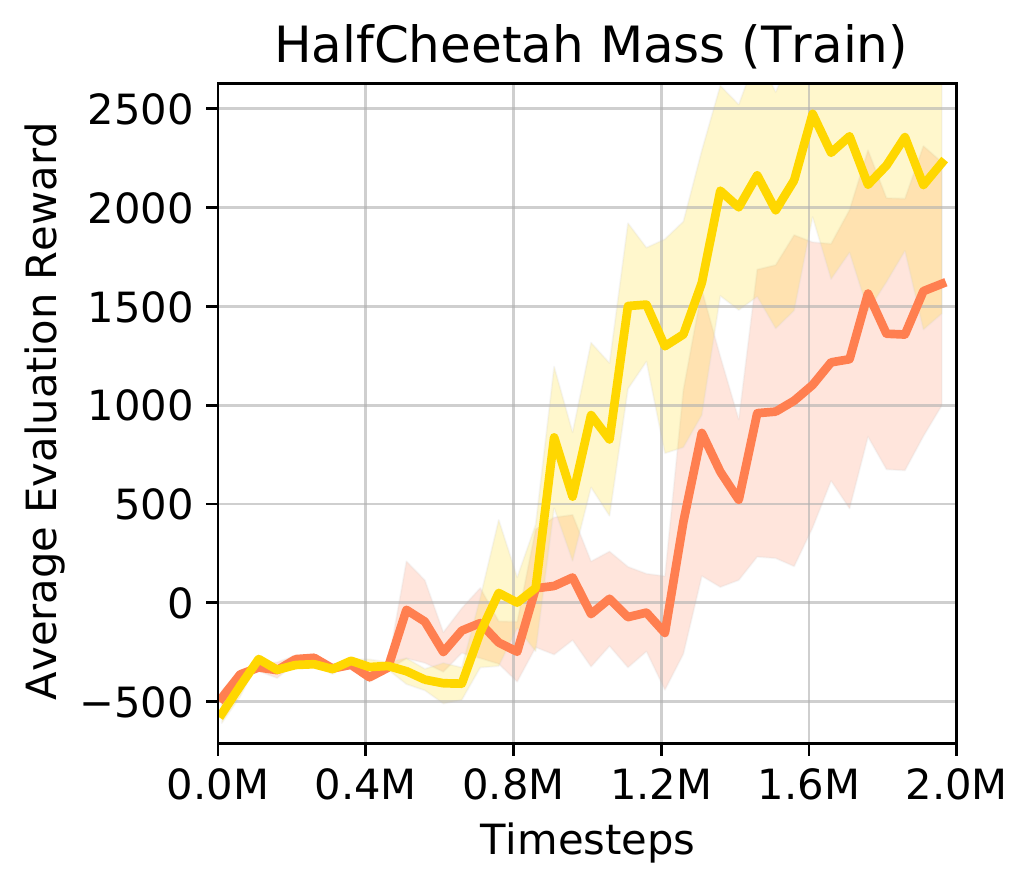}
\end{minipage}
\hspace*{2pt}
\begin{minipage}{0.2\textwidth}
    \centering
    \includegraphics[width=\linewidth]{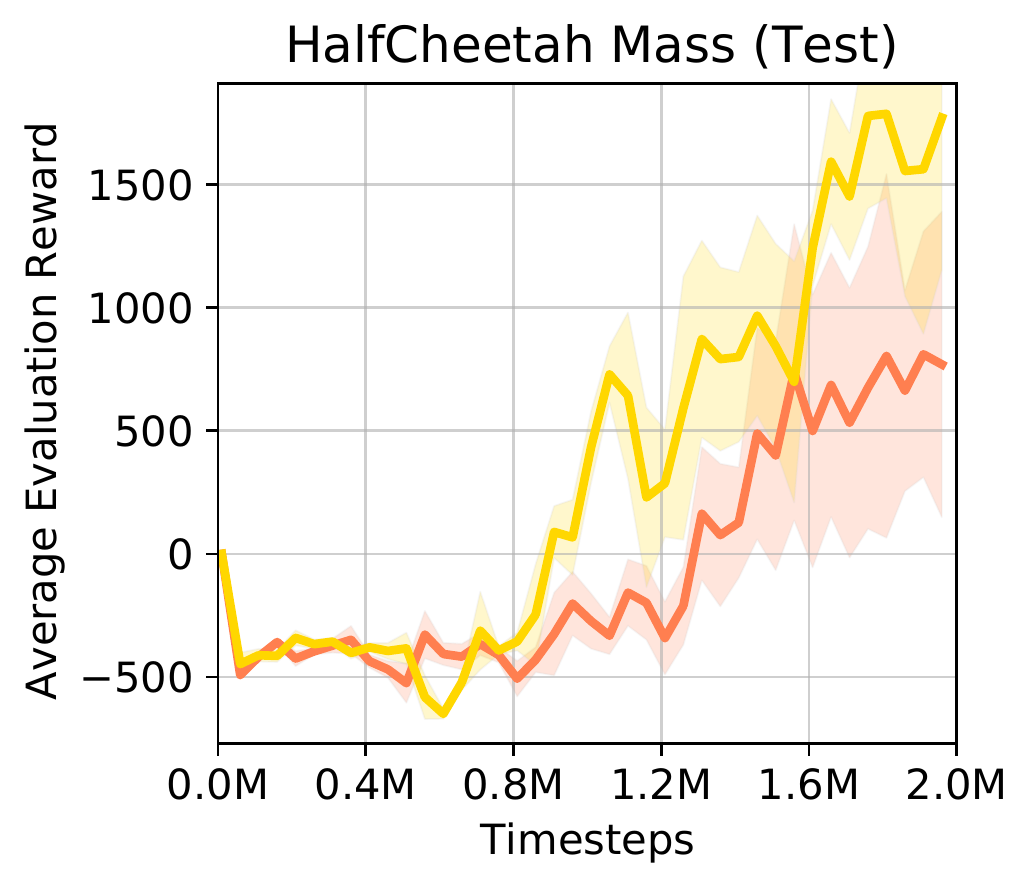}
\end{minipage}
\hspace*{2pt}

    \vspace*{-8pt}
    \caption{Learning curves of the average episode reward, averaged over 3 runs. The average episode reward and standard error are reported on training tasks and test tasks respectively.
}
    \label{fig:effect_contrastive}
    \vspace*{-0.2in}
\end{figure*}

\subsection{Design Choice of Action Translator}
\label{app:their_at}
We add this experimental comparison with the action translator by \citep{zhang2020learning}.
To learn a shared policy solving multiple tasks, we combine the context-based TD3 algorithm, self-imitation learning, and policy transfer with their action translator.
Using their action translator underperforms ours. 
The main reason is that, with changing datasets and policies, their action translator may be harder to tune because there are more moving components (i.e. another action translator, a discriminator) and more loss terms to be balanced (i.e. domain cycle-consistency loss, adversarial loss).

\begin{table}[h!]
\centering
\begin{tabular}{c|cc|cc}
\toprule
Setting & \multicolumn{2}{c|}{\begin{tabular}[c]{@{}c@{}}HalfCheetah Armature\end{tabular}}  & \multicolumn{2}{c}{\begin{tabular}[c]{@{}c@{}}HalfCheetah Mass\end{tabular}}  \\ \midrule
Tasks & Training & Test & Training & Test                                  \\ \midrule
MCAT            & \begin{tabular}[c]{@{}c@{}}\textbf{2267.2} \tiny{($\pm$ 579.2)}\end{tabular}                                   & \begin{tabular}[c]{@{}c@{}}\textbf{2071.5} \tiny{($\pm$ 447.4)}\end{tabular}                                 &
\begin{tabular}[c]{@{}c@{}}\textbf{2226.3}
\tiny{($\pm$ 762.6)}\end{tabular}
                      &
\begin{tabular}[c]{@{}c@{}}\textbf{1771.1}\tiny{($\pm$ 617.7)}\end{tabular}  \\ \midrule
\begin{tabular}[c]{@{}c@{}}MCAT with \citep{zhang2020learning}\\ action translator\end{tabular} & \begin{tabular}[c]{@{}c@{}}2255.2 \tiny{($\pm$ 644.4)}\end{tabular}                                   & \begin{tabular}[c]{@{}c@{}}1664.8 \tiny{($\pm$ 660.8)}\end{tabular}   &\begin{tabular}[c]{@{}c@{}}1185.8 \tiny{($\pm$ 798.0)}\end{tabular}  &\begin{tabular}[c]{@{}c@{}}684.7\tiny{($\pm$ 759.0)}\end{tabular}                                       \\ \bottomrule
\end{tabular}
\caption{Mean ($\pm$ standard error) of episode rewards on training and test tasks at 2M timesteps.}
\end{table}

\clearpage
\section{Extension of Policy Transfer}
\label{app:extension}

As clarified in Sec.~\ref{sec:problem}, in this work, we mainly focus on tasks with the same state space, action space, reward function but varying dynamics. However, we note that our proposed method of learning action translator may be extended to tackle the challenge of policy transfer in more general cases, such as (1) Tasks differing in reward function, (2) Tasks differing in state space and action space.
In this section, we establish the theory and method in details to extend our policy transfer approach, as a supplement to Sec.~\ref{sec:discussion}.

\subsection{Theoretical Analysis}
Intuitively, on two general tasks, we aim to discover correspondent state-action pairs achieving the same reward and transiting to correspondent next states.
With the state and action correspondence, the behavior of good source policy can be ``replicated'' in the target task and the high value of the good source policy can be maintained by the transferred policy on the target task.
Inspired by this idea, we extend our theory in Sec.~\ref{sec:theory} and Appendix~\ref{app:proof}.

We first define a binary relation for states to describe the equivalent states on two MDPs (Definition 2) and define an invertible function to capture the state equivalence relation (Definition 3). Building upon the notion of state equivalence, we derive the upper bound of value difference between policies on two MDPs (Theorem 2). Finally, we reach a proposition for the upper bound of value difference (Proposition 2) to explain that our objective function in learning action translator can be extended to bound the value difference between the source and transferred policy.

\begin{restatable}{definition}{state_action_equivalence}
\label{def:state_action_binary_relation}
Given two MDPs $\mathcal{T}^{(i)}=\{\mathcal{S}^{(i)}, \mathcal{A}^{(i)}, p^{(i)}, r^{(i)}, \gamma, \rho^{(i)}_0\}$ and  $\mathcal{T}^{(j)}=\{\mathcal{S}^{(j)}, \mathcal{A}^{(j)}, p^{(j)}, r^{(j)}, \gamma, \rho^{(j)}_0\}$,
we define a binary relation $B \in \mathcal{S}^{(i)} \times \mathcal{S}^{(j)}$ called \textbf{state equivalence relation}.
Let $s'^{(i)}$ denote the next state following state $s^{(i)}$, and $s'^{(j)}$ denote the next state following state $s^{(j)}$.
For states $s^{(i)}\in \mathcal{S}^{(i)}$, $s^{(j)}\in \mathcal{S}^{(j)}$, we have $(s^{(i)}, s^{(j)}) \in B$ (i.e. $s^{(i)} B s^{(j)}$) if for any $a^{(i)}\in \mathcal{A}^{(i)}$ there exists $a^{(j)}\in \mathcal{A}^{(j)}$ satisfying the following conditions:

\begin{equation*}
r^{(i)}(s^{(i)}, a^{(i)}) = r^{(j)}(s^{(j)}, a^{(j)})
\end{equation*}

\begin{equation*}
\forall s'^{(i)} \in S^{(i)}, \exists s'^{(j)} \in \mathcal{S}^{(j)} \text{ s.t. } p^{(i)}(s'^{(i)}|s^{(i)}, a^{(i)}) = p^{(j)}(s'^{(j)}|s^{(i)}, a^{(j)}) \text{ and } s'^{(i)}Bs'^{(j)}
\end{equation*}

\end{restatable}

We call the state $\smash{s^{(i)}}$ and $\smash{s^{(j)}}$ are correspondent/equivalent when $\smash{(s^{(i)}, s^{(j)}) \in B}$. Also, in this case, the action $\smash{a^{(i)}}$ for state $\smash{s^{(i)}}$ on the MDP $\smash{\mathcal{T}^{(i)}}$ is equivalent to the action $\smash{a^{(j)}}$ for state $\smash{s^{(j)}}$ on the MDP $\smash{\mathcal{T}^{(j)}}$.

This definition is related to stochastic bisimulation relation in \citep{ferns2004metrics, zhang2020invariant, zhang2020robust}. Unlike these prior works about state bisimulation, we allow the equivalent actions $\smash{a^{(j)} \neq a^{(i)}}$.
So action $a$ on the task $\smash{\mathcal{T}^{(i)}}$ might not be equivalent to $a$ on the task $\smash{\mathcal{T}^{(j)}}$, and hence we need to involve action translator in learning of both the state correspondence and action correspondence.

Drawing upon Definition~\ref{def:state_action_binary_relation}, we define a one-to-one mapping to identify the equivalent state across two spaces $\mathcal{S}^{(i)}$ and $\mathcal{S}^{(j)}$.

\begin{restatable}{definition}{state_translator}
\label{def:invertible_function}
Given two MDPs $\mathcal{T}^{(i)}=\{\mathcal{S}^{(i)}, \mathcal{A}^{(i)}, p^{(i)}, r^{(i)}, \gamma, \rho^{(i)}_0\}$ and  $\mathcal{T}^{(j)}=\{\mathcal{S}^{(j)}, \mathcal{A}^{(j)}, p^{(j)}, r^{(j)}, \gamma, \rho^{(j)}_0\}$ with state equivalence relation $B$,
we consider subsets $\mathcal{S}_B^{(i)} \subset \mathcal{S}^{(i)}$ and $\mathcal{S}_B^{(j)} \subset \mathcal{S}^{(j)}$ satisfying:
$\forall s^{(i)}\in \mathcal{S}^{(i)}_B, \exists s^{(j)}\in \mathcal{S}^{(j)}_B \text{ s.t. } (s^{(i)}, s^{(j)}) \in B$.
We define a invertible function $G: \mathcal{S}_B^{(i)} \rightarrow \mathcal{S}_B^{(j)}$ called \textbf{state translator function}, satisfying: $(s^{(i)}, G(s^{(i)})) \in B$.
\end{restatable}

Based on Defintion~\ref{def:state_action_binary_relation} and \ref{def:invertible_function}, given two correspondent states $s^{(i)} \in S_B^{(i)}$ and $s^{(j)} \in S_B^{(j)}$, we can derive the upper bound for the value difference between $V^{\pi^{(i)}}(s^{(i)}, \mathcal{T}^{(i)})$ and $V^{\pi^{(j)}}(s^{(j)}, \mathcal{T}^{(j)})$.

\begin{restatable}{thm}{state_action_boundpolicytransfer}
\label{thm:state_action_bound_policy}
$\mathcal{T}^{(i)}=\{\mathcal{S}^{(i)}, \mathcal{A}^{(j)}, p^{(i)}, r^{(i)}, \gamma, \rho^{(i)}_0\}$ and  $\mathcal{T}^{(j)}=\{\mathcal{S}^{(j)}, \mathcal{A}^{(j)}, p^{(j)}, r^{(j)}, \gamma, \rho^{(j)}_0\}$ are two MDPs sampled from the distribution of tasks $p(\mathcal{T})$.
$\pi^{(i)}$ is a deterministic policy on $\mathcal{T}^{(i)}$ and $\pi^{(j)}$ is a deterministic policy on $\mathcal{T}^{(j)}$. 
We assume there exist state equivalence relation $B\in \mathcal{S}^{(i)} \times \mathcal{S}^{(j)} $ and a state translator function $G$ defining a one-to-one mapping from $\mathcal{S}^{(i)}_B$ to $\mathcal{S}^{(j)}_B$.
Let $M=\sup_{s^{(i)}\in \mathcal{S}^{(i)}} |V^{\pi^{(i)}}(s^{(i)}, \mathcal{T}^{(i)})|$ and $d = \sup_{s^{(i)}\in\mathcal{S}_B^{(i)}} \left[ |r^{(i)}(s^{(i)}, \pi^{(i)}(s)) - r^{(j)}(s^{(j)}, \pi^{(j)}(s))| +2\gamma M D_{TV}(p^{(i)}(\cdot|s^{(i)},\pi^{(i)}(s^{(i)})), p^{(j)}(G(\cdot)|s^{(j)},\pi^{(j)}(s^{(j)})))\right]$.
Then $\forall s^{(i)} \in \mathcal{S}_B^{(i)}, s^{(j)} = G(s^{(i)})$, we have

\begin{equation*}
\left|V^{\pi^{(i)}}(s^{(i)}, \mathcal{T}^{(i)}) - V^{\pi^{(j)}}(s^{(j)}, \mathcal{T}^{(j)})\right| \leq \frac{d}{1-\gamma}
\end{equation*}

\end{restatable}

\begin{proof}
Let $a^{(i)} = \pi^{(i)}(s^{(i)})$ and $a^{(j)} = \pi^{(j)}(s^{(j)})$. We rewrite the value difference.

\begin{eqnarray*}
&&V^{\pi^{(i)}}(s^{(i)}, \mathcal{T}^{(i)}) - V^{\pi^{(j)}}(s^{(j)}, \mathcal{T}^{(j)}) \\ \nonumber 
&=& r^{(i)}(s^{(i)}, a^{(i)}) + \gamma \sum_{s'^{(i)} \in \mathcal{S}^{(i)}}p^{(i)}(s'^{(i)}|s^{(i)}, a^{(i)}) V^{\pi^{(i)}}(s'^{(i)}, \mathcal{T}^{(i)})  
- r^{(j)}(s^{(j)}, a^{(j)}) - \gamma \sum_{s'^{(j)}\in \mathcal{S}^{(j)}}p^{(j)}(s'^{(j)}|s^{(j)}, a^{(j)}) V^{\pi^{(j)}}(s'^{(j)}, \mathcal{T}^{(j)}) \\ \nonumber
&= & (r^{(i)}(s^{(i)}, a^{(i)})- r^{(j)}(s^{(j)}, a^{(j)})) \\ \nonumber
&+& \gamma (\sum_{s'^{(i)}\in\mathcal{S}^{(i)}}p^{(i)}(s'^{(i)}|s^{(i)}, a^{(i)}) V^{\pi^{(i)}}(s'^{(i)},\mathcal{T}^{(i)}) - \sum_{s'^{(j)}\in\mathcal{S}^{(j)}}p^{(j)}(s'^{(j)}|s^{(j)}, a^{(j)}) V^{\pi^{(j)}}(s'^{(j)}, \mathcal{T}^{(j)})) \\
\end{eqnarray*}

According to Definition 2, since $s^{(i)} \in \mathcal{S}^{(i)}_B$, we have $s'^{(i)} \in \mathcal{S}^{(i)}_B $. Similarly, $s'^{(j)}\in \mathcal{S}^{(j)}_B$.

Then we derive the second term in the right side of the equation above:

\begin{eqnarray*}
\nonumber
&&\sum_{s'^{(i)}\in\mathcal{S}^{(i)}}p^{(i)}(s'^{(i)}|s^{(i)}, a^{(i)}) V^{\pi^{(i)}}(s'^{(i)},\mathcal{T}^{(i)}) - \sum_{s'^{(j)}\in\mathcal{S}^{(j)}}p^{(j)}(s'^{(j)}|s^{(j)}, a^{(j)}) V^{\pi^{(j)}}(s'^{(j)}, \mathcal{T}^{(j)}) \\ \nonumber
&&\text{\hspace{1in} {\color{blue} *replace $\mathcal{S}^{(i)}$ by $\mathcal{S}_B^{(i)}$ because $s'^{(i)}\in \mathcal{S}^{(i)}_B$, replace $\mathcal{S}^{(j)}$ by $\mathcal{S}_B^{(j)}$ because $s'^{(j)}\in \mathcal{S}^{(j)}_B$}} \\ \nonumber
&&\text{\hspace{1in} {\color{blue} **minus and plus $\sum_{s'^{(i)} \in\mathcal{S}_B^{(i)}} p^{(j)}(G(s'^{(i)})|s^{(j)}, a^{(j)}))V^{\pi^{(i)}}(s'^{(i)}, \mathcal{T}^{(i)})$}} \\ \nonumber
&=& \sum_{s'^{(i)} \in\mathcal{S}_B^{(i)}}p^{(i)}(s'^{(i)}|s^{(i)}, a^{(i)}) V^{\pi^{(i)}}(s'^{(i)}, \mathcal{T}^{(i)}) - \sum_{s'^{(i)} \in\mathcal{S}_B^{(i)}} p^{(j)}(G(s'^{(i)})|s^{(j)}, a^{(j)}))V^{\pi^{(i)}}(s'^{(i)}, \mathcal{T}^{(i)}) \\ \nonumber
&+& \sum_{s'^{(i)} \in\mathcal{S}_B^{(i)}} p^{(j)}(G(s'^{(i)})|s^{(j)}, a^{(j)}))V^{\pi^{(i)}}(s'^{(i)}, \mathcal{T}^{(i)}) - \sum_{s'^{(j)}\in\mathcal{S}_B^{(j)}}p^{(j)}(s'^{(j)}|s^{(j)}, a^{(j)}) V^{\pi^{(j)}}(s'^{(j)}, \mathcal{T}^{(j)}) \\ \nonumber
&&\text{\hspace{1in} {\color{blue} *combine the first two terms, rewrite the third term because $G$ is invertible function }} \\ \nonumber
&=& \sum_{s'^{(i)} \in \mathcal{S}_B^{(i)}}\left[p^{(i)}(s'^{(i)}|s^{(i)}, a^{(i)})-p^{(j)}(G(s'^{(i)})|s^{(j)}, a^{(j)})\right] V^{\pi^{(i)}}(s'^{(i)}, \mathcal{T}^{(i)})  \\ \nonumber
&+& \sum_{s'^{(j)} \in\mathcal{S}_B^{(j)}} p^{(j)}(s'^{(j)}|s^{(j)}, a^{(j)}))V^{\pi^{(i)}}(G^{-1}(s'^{(j)}), \mathcal{T}^{(i)}) - \sum_{s'^{(j)}\in\mathcal{S}_B^{(j)}}p^{(j)}(s'^{(j)}|s^{(j)}, a^{(j)}) V^{\pi^{(j)}}(s'^{(j)}, \mathcal{T}^{(j)}) \\ \nonumber
&&\text{\hspace{4in} {\color{blue} *combine the last two terms}} \\ \nonumber
&=& \sum_{s'^{(i)} \in \mathcal{S}_B^{(i)}}\left[p^{(i)}(s'^{(i)}|s^{(i)}, a^{(i)})-p^{(j)}(G(s'^{(i)})|s^{(j)}, a^{(j)})\right] V^{\pi^{(i)}}(s'^{(i)}, \mathcal{T}^{(i)}) \\
&+& \sum_{s'^{(j)} \in\mathcal{S}_B^{(j)}} p^{(j)}(s'^{(j)}|s^{(j)}, a^{(j)})) \left[V^{\pi^{(i)}}(G^{-1}(s'^{(j)}), \mathcal{T}^{(i)}) - V^{\pi^{(j)}}(s'^{(j)}, \mathcal{T}^{(j)})\right]
\end{eqnarray*}

Therefore, we can bound the absolute value of the value difference according to the two equation arrays above:

\begin{eqnarray*}
\left|V^{\pi^{(i)}}(s^{(i)}, \mathcal{T}^{(i)}) - V^{\pi^{(j)}}(s^{(j)}, \mathcal{T}^{(j)})\right|&\leq & \left|r^{(i)}(s^{(i)}, a^{(i)})- r^{(j)}(s^{(j)}, a^{(j)})\right| \\ \nonumber
&+& \gamma \sum_{s'^{(i)} \in \mathcal{S}_B^{(i)}}\left[p^{(i)}(s'^{(i)}|s^{(i)}, a^{(i)})-p^{(j)}(G(s'^{(i)})|s^{(j)}, a^{(j)})\right] V^{\pi^{(i)}}(s'^{(i)}, \mathcal{T}^{(i)}) \\ \nonumber
&+& \gamma \sum_{s'^{(j)} \in\mathcal{S}_B^{(j)}} p^{(j)}(s'^{(j)}|s^{(j)}, a^{(j)})) \left[V^{\pi^{(i)}}(G^{-1}(s'^{(j)}), \mathcal{T}^{(i)}) - V^{\pi^{(j)}}(s'^{(j)}, \mathcal{T}^{(j)})\right] \\ \nonumber
&\leq & \left|r^{(i)}(s^{(i)}, a^{(i)})- r^{(j)}(s^{(j)}, a^{(j)})\right| 
+ 2 \gamma M D_{TV}(p^{(i)}(\cdot|s^{(i)}, a^{(i)}), p^{(j)}(G(\cdot)|s^{(j)}, a^{(j)})) \\ \nonumber
&+&\gamma \sup_{s'^{(j)} \in \mathcal{S}_B^{(j)}} \left|V^{\pi^{(i)}}(G^{-1}(s'^{(j)}), \mathcal{T}^{(i)}) - V^{\pi^{(j)}}(s'^{(j)}, \mathcal{T}^{(j)})\right| \\ \nonumber
&\leq& d + \gamma \sup_{s'^{(i)} \in \mathcal{S}_B^{(i)}}\left|V^{\pi^{(i)}}(s'^{(i)}, \mathcal{T}^{(i)}) - V^{\pi^{(j)}}(G(s'^{(i)}), \mathcal{T}^{(j)})\right| 
\leq \frac{d}{1-\gamma}\\ \nonumber
\end{eqnarray*}
\end{proof}

Theorem 2 proves the value difference is upper bounded by a scalar $d$, depending on the reward difference $\smash{|r^{(i)}(s^{(i)}, \pi^{(i)}(s^{(j)})) - r^{(j)}(s^{(j)}, \pi^{(j)}(s^{(j)}))|}$ and $\smash{D_{TV}(p^{(i)}(\cdot|s^{(i)}, a^{(i)}), p^{(j)}(G(\cdot)|s^{(j)}, a^{(j)}))}$, i.e. the total-variation distance between probability distribution of next state on $\mathcal{T}^{(i)}$ and probability distribution of correspondent next state on $\mathcal{T}^{(j)}$.
Indeed, if the state equivalence relation is only true for identical states (i.e. $G$ is an identity mapping, $s^{(i)}Bs^{(j)}$ if and only if $s^{(i)}=s^{(j)}$), then Theorem 2 degenerates into Theorem 1.
We note the proof of Theorem 2 is similar to proof of Theorem 1 in Appendix A.

For a special case, where the reward only depends on the current state and next state, we can formulate a simpler definition of scalar $d$. The following Proposition 2 is analogous to Proposition 1 in the assumption about reward function.

\begin{restatable}{cor}{state_action_rew_boundpolicytransfer}
\label{cor:state_action_bound_policy}
$\mathcal{T}^{(i)}=\{\mathcal{S}^{(i)}, \mathcal{A}{(i)}, p^{(i)}, r^{(i)}, \gamma, \rho^{(i)}_0\}$ and  $\mathcal{T}^{(j)}=\{\mathcal{S}^{(j)}, \mathcal{A}^{(j)}, p^{(j)}, r^{(j)}, \gamma, \rho^{(j)}_0\}$ are two MDPs sampled from the distribution of tasks $p(\mathcal{T})$.
$\pi^{(i)}$ is a deterministic policy on $\mathcal{T}^{(i)}$ and $\pi^{(j)}$ is a deterministic policy on $\mathcal{T}^{(j)}$. 
We assume there exist state equivalence relation $B\in \mathcal{S}^{(i)} \times \mathcal{S}^{(j)} $ and a state translator function $G$ defining a one-to-one mapping from $\mathcal{S}^{(i)}_B$ to $\mathcal{S}^{(j)}_B$.
Suppose that the reward function $r^{(i)}(s^{(i)},a^{(i)},s'^{(i)})=r^{(i)}(s^{(i)},s'^{(i)})$ and $r^{(j)}(s^{(j)},a^{(j)},s'^{(j)})=r^{(j)}(s^{(j)},s'^{(j)})$. 
If $s^{(j)}=G(s^{(i)})$ and $s'^{(j)}=G(s'^{(i)})$, $r^{(i)}(s^{(i)},s'^{(i)})=r^{(j)}(s^{(j)},s'^{(j)})$.
Let $M=\sup_{s^{(i)}\in \mathcal{S}^{(i)}} |r^{(i)}(s^{(i)}, s'^{(i)})+\gamma V^{\pi^{(i)}}(s'^{(i)}, \mathcal{T}^{(i)})|$ and $d=\sup_{s^{(i)}\in\mathcal{S}_B^{(i)}} 2 M D_{TV}(p^{(i)}(\cdot|s^{(i)}, \pi^{(i)}(s^{(i)})), p^{(j)}(G(\cdot)|s^{(j)}, \pi^{(j)}(s^{(j)})))$.

Then $\forall s^{(i)} \in \mathcal{S}_B^{(i)}, s^{(j)} = G(s^{(i)})$, we have

$$\left|V^{\pi^{(i)}}(s^{(i)}, \mathcal{T}^{(i)}) - V^{\pi^{(j)}}(s^{(j)}, \mathcal{T}^{(j)})\right| \leq \frac{d}{1-\gamma}$$
\end{restatable}

\begin{proof}
Let $a^{(i)} = \pi^{(i)}(s^{(i)})$ and $a^{(j)} = \pi^{(j)}(s^{(j)})$. $s'^{(i)}$ denotes the next state following state $s^{(i)}$.

Because the reward solely depends on the current and next state, we rewrite the value function:

\begin{eqnarray*}
V^{\pi^{(i)}}(s^{(i)}, \mathcal{T}^{(i)}) &=& r^{(i)}(s^{(i)}, a^{(i)}) + \gamma \sum_{s'^{(i)} \in \mathcal{S}^{(i)}}p^{(i)}(s'^{(i)}|s^{(i)}, a^{(i)}) V^{\pi^{(i)}}(s'^{(i)},\mathcal{T}^{(i)}) \\ \nonumber
&=& \sum_{s'^{(i)} \in \mathcal{S}^{(i)}}p^{(i)}(s'^{(i)}|s^{(i)}, a^{(i)}) r^{(i)}(s^{(i)},s'^{(i)}) + \gamma \sum_{s'^{(i)} \in \mathcal{S}^{(i)}}p^{(i)}(s'^{(i)}|s^{(i)}, a^{(i)}) V^{\pi^{(i)}}(s'^{(i)},\mathcal{T}^{(i)}) \\ \nonumber
&=& \sum_{s'^{(i)} \in \mathcal{S}^{(i)}}p^{(i)}(s'^{(i)}|s^{(i)}, a^{(i)}) \left[r^{(i)}(s^{(i)}, s'^{(i)}) + \gamma V^{\pi^{(i)}}(s'^{(i)}, \mathcal{T}^{(i)})\right] 
\end{eqnarray*}

Then we derive the value difference:

\begin{eqnarray*}
&&V^{\pi^{(i)}}(s^{(i)}, \mathcal{T}^{(i)}) - V^{\pi^{(j)}}(s^{(i)}, \mathcal{T}^{(j)}) \\ \nonumber
&=&  \sum_{s'^{(i)}}p^{(i)}(s'^{(i)}|s^{(i)}, a^{(i)}) \left[r^{(i)}(s^{(i)}, s'^{(i)}) + \gamma V^{\pi^{(i)}}(s'^{(i)}, \mathcal{T}^{(i)})\right] 
- \sum_{s'^{(j)}}p^{(j)}(s'^{(j)}|s^{(j)}, a^{(j)})\left[r^{(j)}(s^{(j)}, s'^{(j)})+\gamma V^{\pi^{(j)}}(s'^{(j)}, \mathcal{T}^{(j)})\right] \\ \nonumber
&&\text{\hspace{2.5in} {\color{blue} *minus and plus $\sum_{s'^{(i)}}p^{(j)}(G(s'^{(i)})|s^{(j)}, a^{(j)})\left[r^{(i)}(s^{(i)}, s'^{(i)})+\gamma V^{\pi^{(i)}}(s'^{(i)}, \mathcal{T}^{(i)})\right]$}} \\ \nonumber
&=& \sum_{s'^{(i)}}p^{(i)}(s'^{(i)}|s^{(i)}, a^{(i)}) \left[r^{(i)}(s^{(i)}, s'^{(i)}) + \gamma V^{\pi^{(i)}}(s'^{(i)}, \mathcal{T}^{(i)})\right] - \sum_{s'^{(i)}}p^{(j)}(G(s'^{(i)})|s^{(j)}, a^{(j)})\left[r^{(i)}(s^{(i)}, s'^{(i)})+\gamma V^{\pi^{(i)}}(s'^{(i)}, \mathcal{T}^{(i)})\right] \\ \nonumber
& + & \sum_{s'^{(i)}}p^{(j)}(G(s'^{(i)})|s^{(j)}, a^{(j)})\left[r^{(i)}(s^{(i)}, s'^{(i)})+\gamma V^{\pi^{(i)}}(s'^{(i)}, \mathcal{T}^{(i)})\right] - \sum_{s'^{(j)}}p^{(j)}(s'^{(j)}|s^{(j)}, a^{(j)})\left[r^{(j)}(s^{(j)}, s'^{(j)})+\gamma V^{\pi^{(j)}}(s'^{(j)}, \mathcal{T}^{(j)})\right] \\ \nonumber
&&\text{\hspace{2.8in} {\color{blue} *combine first two terms, rewrite the third term given invertible function $G$}} \\ \nonumber
&=& \sum_{s'^{(i)}}\left[p^{(i)}(s'^{(i)}|s^{(i)}, a^{(i)})-p^{(j)}(G(s'^{(i)})|s^{(j)}, a^{(j)})\right] \left[r^{(i)}(s^{(i)}, s'^{(i)})+\gamma V^{\pi^{(i)}}(s'^{(i)}, \mathcal{T}^{(i)})\right] \\ \nonumber
&+& \sum_{s'^{(j)}}p^{(j)}(s'^{(j)}|s^{(j)}, a^{(j)})\left[r^{(i)}(G^{-1}(s^{(j)}), G^{-1}(s'^{(j)}))+\gamma V^{\pi^{(i)}}(G^{-1}(s'^{(j)}), \mathcal{T}^{(i)})\right] \\ \nonumber
&-& \sum_{s'^{(j)}}p^{(j)}(s'^{(j)}|s^{(j)}, a^{(j)})\left[r^{(j)}(s^{(j)}, s'^{(j)})+\gamma V^{\pi^{(j)}}(s'^{(j)}, \mathcal{T}^{(j)})\right]
\text{\hspace{0in} {\color{blue} *combine last two terms, note the assumption of reward function}} \\ \nonumber
&=& \sum_{s'^{(i)}}\left[p^{(i)}(s'^{(i)}|s^{(i)}, a^{(i)})-p^{(j)}(G(s'^{(i)})|s^{(j)}, a^{(j)})\right] \left[r^{(i)}(s^{(i)}, s'^{(i)})+\gamma V^{\pi^{(i)}}(s'^{(i)}, \mathcal{T}^{(i)})\right] \\ \nonumber
&+& \gamma\sum_{s'^{(j)}}p^{(j)}(s'^{(j)}|s^{(j)}, a^{(j)})\left[V^{\pi^{(i)}}(G^{-1}(s'^{(j)}), \mathcal{T}^{(i)}) -V^{\pi^{(j)}}(s'^{(j)}, \mathcal{T}^{(j)}) \right]  \\ \nonumber
\end{eqnarray*}

Therefore, the absolute value of value difference can be upper bounded. The proof is similar to the proof of Theorem 2.

\begin{eqnarray*}
\left|V^{\pi^{(i)}}(s^{(i)}, \mathcal{T}^{(i)}) - V^{\pi^{(j)}}(s^{(j)}, \mathcal{T}^{(j)})\right|&\leq & 2 M D_{TV}(p^{(i)}(\cdot|s^{(i)}, a^{(i)}), p^{(j)}(G(\cdot)|s^{(j)}, a^{(j)})) \\ \nonumber
&+&\gamma \sup_{s'^{(j)} \in \mathcal{S}_B^{(j)}} \left|V^{\pi^{(i)}}(G^{-1}(s'^{(j)}), \mathcal{T}^{(i)}) - V^{\pi^{(j)}}(s'^{(j)}, \mathcal{T}^{(j)})\right| \\ \nonumber
&\leq& d + \gamma \sup_{s'^{(i)} \in \mathcal{S}_B^{(i)}}\left|V^{\pi^{(i)}}(s'^{(i)}, \mathcal{T}^{(i)}) - V^{\pi^{(j)}}(G(s'^{(i)}), \mathcal{T}^{(j)})\right| \\ \nonumber
&\leq& \frac{d}{1-\gamma}\\ \nonumber
\end{eqnarray*}

\end{proof}

Obviously, if the state equivalence relation is only true for identical states (i.e. $G$ is an identity mapping, $s^{(i)}Bs^{(j)}$ if and only if $s^{(i)}=s^{(j)}$), then Proposition 2 degenerates into Proposition 1.
If we optimize the action translator $H$ to minimize $d$ for policy $\pi^{(j)}$ and $\pi^{(i)}(s^{(i)})=H(s^{(j)}, \pi^{(j)}(s^{(j)}))$, the policy value for correspondent states $s^{(i)}$ and $s^{(j)}$ can be close.
Minimizing $d$ means finding actions leading to next states remaining correspondent.

\subsection{Method}
According to Proposition 2, we not only learn an action translator $H$, but also state translators $G$ mapping target states $s^{(i)}$ to the equivalent states on source task $\mathcal{T}^{(j)}$ and $G^{-1}$ identifying correspondent state on target task $\mathcal{T}^{(i)}$.
We additionally learn a discriminator network $D$ to assist learning of state translator.

Given transition data $ s^{(j)}$ on source task and $s^{(i)}$ on target task, the adversarial objective is:

$$\min_G\max_D \mathcal{L}_{adv}(G,D)=\log D(s^{(j)}) + \log (1-D(G(s^{(i)})))$$.

$G$ aims to map target state $s^{(i)}$ to the distribution of states on source task, while $D$ tries to distinguish translated state $G(s^{(i)})$ and real states in the source task.
To build state equivalence, the translated state should be translated back to the source state. We further leverage cycle consistency loss to learn the one-to-one mapping on states across tasks:

$$\mathcal{L}_{back}=|G^{-1}(G(s^{(i)}))-s^{(i)}| + |G(G^{-1}(s^{(j)}))-s^{(j)}|$$ 

Drawn upon Proposition 2, we extend our transfer loss $\mathcal{L}_{trans}$ to $\mathcal{L}_{trans,s,a}$. Formally,

$$\mathcal{L}_{trans,s,a} = -\log F(\tilde{s}_{t+1}^{(i)}|\tilde{s}^{(i)}_t, \tilde{a}^{(i)}_t)$$

where $\tilde{s}_{t+1}^{(i)} = G^{-1}(s^{(j)}_{t+1})$, $\tilde{s}_{t}^{(i)} = G^{-1}(s^{(j)}_{t})$, and $\tilde{a}^{(i)}=H(s^{(j)}_t, a^{(j)}_t)$. $\mathcal{L}_{trans,s,a}$ is applied to optimize the state translator $G^{-1}$ and action translator $H$.

In this way, given the state $\smash{s_t^{(j)}}$ on source task, we first get the correspondent state $\smash{\tilde{s}_{t}^{(i)}}$ on target task. Then the translated action $\smash{\tilde{a}^{(i)}}$ make transition to next state $\smash{\tilde{s}_{t+1}^{(i)}}$ on target task still correspondent to next state $\smash{s^{(j)}_{t+1}}$ on source task.
The objective function $\smash{\mathcal{L}_{trans,s,a}}$ drives the next state distribution on the target task $\smash{p^{(i)}(\cdot|\tilde{s}_{t}^{(i)}, \tilde{a}^{(i)})}$ to be close to the distribution of correspondent next state on the source task $\smash{p^{(j)}(G(\cdot)|s_t^{(j)}, a_t^{(j)})}$. This is implicitly minimizing $d$ in Proposition 2.

In practice, we may need the action translator network $H$ or the state translator network $G$ and $\smash{G^{-1}}$ reasonably initialized, in order to prevent the joint training collapsing to a trivial solution.
The implementation details of learning the context model, forward dynamics model and action translator are the same as we explained in Appendix~\ref{app:fixed_details}.
During training of the state translator, the weight of $\mathcal{L}_{adv}, \mathcal{L}_{back}, \mathcal{L}_{trans\_s\_a}$ is 10, 30, 100 respectively, the same as default hyper-parameters in \citep{zhang2020learning}.
The similar technique of learning state translator and action translator has been mentioned in \citep{zhang2020learning}. Yet, our theorems shed light on its underlying mechanism and our objective function for learning the action translator is simpler.

\subsection{Experiments on Tasks Differing in Reward Function}
When the tasks share \textbf{the same state space and action space but the reward function varies}, we combine our action translator with a state translator for policy transfer. 

To investigate this scheme for policy transfer, we conduct experiments on MetaWorld task moving the robot arm to a goal location.
We set the source and target task with different goal locations and hence with different reward functions.
Tab.~\ref{tab:app_extension} lists the goal locations on the source and target tasks.
Specifically, on the same state $s$ of the agent's current location $(x, y, z)$, the reward varies across tasks, because it is inversely proportional to the distance from the current location to goal.
The initial location of robot arm is randomly sampled between $[-0.1, 0.6, 0.02]$ and $[0.1, 0.7, 0.02]$.
The state is current location of the robot arm.
The action is the moving vector of the robot arm.

We compare our method and \citep{zhang2020learning} learning both state translator and action translator.
We initialize the state translator networks by assigning $G(s=(x,y,z))=G^{-1}(s=(x,y,z))=(-x, y, z)$.
As observed in Tab.~\ref{tab:app_extension}, ours compares favorably with \citep{zhang2020learning} and achieves satisfactory cumulative episode reward on the target task. 
We conclude that, for source and target tasks with different reward functions depending on the state and next state, learning state translator and action translator jointly is promising for policy transfer.

\begin{table}[!ht]
\centering

\setlength{\tabcolsep}{2pt}
\begin{tabular}{c|c|c|ccc}
\toprule
\makecell{Source Task} &
\makecell{Target Task} &
\footnotesize{\makecell{Source policy\\on source task} } &
\footnotesize{\makecell{Source policy\\on target task} } & \footnotesize{\makecell{Transferred policy \\ \citep{zhang2020learning} \\ on target task}}&
\footnotesize{\makecell{Transferred policy \\(Ours) \\ on target task}} \\ 
\midrule
$[-0.1,0.8,0.2]$ & $[0.1,0.8,0.2]$ & 4855.7 & 947.5 & 1798.2\tiny{($\pm$ 592.4)} & \textbf{3124.3}\tiny{($\pm$ 1042.0)}  \\
$[-0.1,0.8,0.2]$ & $[0.05,0.8,0.2]$ & 4855.7 & 1470.2 & 1764.0\tiny{($\pm$ 316.3)}  &  \textbf{1937.1}\tiny{($\pm$ 424.5)}  \\ 
$[-0.1,0.8,0.2]$ & $[0.1,0.8,0.05]$ & 4855.7 & 1040.8 & \textbf{2393.7}\tiny{($\pm$ 869.8)}   &  2315.7\tiny{($\pm$ 1061.5)} \\ 
\midrule
2-leg & 3-leg & 5121.4 & NA & 1957.8\tiny{($\pm$298.4)}	& \textbf{2018.2}\tiny{($\pm$50.8)} \\
\bottomrule
\end{tabular}

\caption{Mean ($\pm$ standard error) of episode rewards over 3 runs, comparing source and transferred policy on target task. This is expanding Tab.~\ref{tab:app_extension} in the main text.}
\label{tab:app_extension}
\end{table}

\subsection{Experiments on Tasks Differing in State and Action Space}
For tasks with \textbf{different state space and action space},
we investigate the proposed idea on MuJoco environment HalfCheetah.
The HalfCheetah agent by default has 2 legs in the source task and we modify the agent to have 3 legs in the target task.
Because the agents have different numbers of joints in the source and target task, the dimensions of state space and action space also differ, as explained in \citep{zhang2020learning}.
Again, we compare our method and \citep{zhang2020learning} learning both state translator and action translator.
We assign a good initialization for the action translator in both methods as \citep{zhang2020learning} introduced. We remark that ours with a simpler objective function and fewer components than the baseline method can transfer the source policy to perform well on the target task.

\end{document}